\documentclass[journal, 11pt, a4paper,onecolumn]{IEEEtran}

\usepackage[cmex10]{amsmath}
\usepackage{amsthm, amsfonts, amssymb, bm}
\usepackage{array}
\usepackage{color}
\usepackage{epsfig}
\usepackage[linesnumbered,ruled,vlined]{algorithm2e}
\usepackage {algpseudocode}
\usepackage{algorithmicx}
\usepackage{algcompatible}
\usepackage[active]{srcltx}
\usepackage{subfig}
\usepackage{hyperref}
\usepackage{rotating}
\usepackage{multirow}
\usepackage{hhline}
\usepackage[table]{xcolor}

\usepackage[autostyle]{csquotes}

\definecolor{gray}{rgb}{0.92,0.92,0.92}
\definecolor{dgray}{rgb}{0.75,0.75,0.75}
\definecolor{darkgray}{rgb}{0.5,0.5,0.5}
\definecolor{col1}{RGB}{213,229,255}
\definecolor{col2}{RGB}{128,179,255}
\definecolor{col3}{RGB}{42,127,255}
\definecolor{col4}{RGB}{0,85,212}
\definecolor{col5}{RGB}{0,51,128}
\definecolor{col6}{RGB}{0,34,85}

\newtheorem{theorem}{Theorem}[section]
\newtheorem{lemma}[theorem]{Lemma}
\newtheorem{corollary}[theorem]{Corollary}
\newtheorem{example}[theorem]{Example}
\newtheorem{remark}[theorem]{Remark}

\newtheorem{definition}[theorem]{Definition}

\def\norm#1{\|#1\|}

\def\R{{\mathbb R}}

\def\be{\begin{eqnarray*}}
\def\ee{\end{eqnarray*}}
\def\beq{\begin{equation}}
\def\eeq{\end{equation}}

\def\2q{\quad\quad}

\def\E{{\bf E}}

\def\R{{\mathbb R}}

\def\:{{\,:\,}}

\def\norm#1{{\left\|\,#1\,\right\|}}

\def\squote#1{{\textquoteleft\,#1\,\textquoteright}}

\def\abs#1{{\left|\,#1\,\right|}}

\def\argmin{\mathop{\rm arg\,min}}

\def\inprod#1#2{\left\langle #1,\,#2\right\rangle}

\def\itb{\begin{itemize}}
\def\ite{\end{itemize}}
\def\bit{\begin{itemize}}
\def\eit{\end{itemize}}
\def\dom{\hbox{dom}}

\begin{document}

\title{An extended asymmetric sigmoid with Perceptron(SIGTRON) for imbalanced linear classification}

\author{Hyenkyun Woo\thanks{Logitron X, \; hyenkyun@gmail.com} }

\author{Hyenkyun~Woo\thanks{H. Woo is with Logitron X, Republic of Korea, e-mail:hyenkyun@gmail.com.}}

\markboth{Preprint ,~Vol.~x, No.~x, \today}{H. Woo: An extended asymmetric sigmoid with Perceptron(SIGTRON) for imbalanced linear classification}

\maketitle

\begin{abstract}
This article presents a new polynomial parameterized sigmoid called SIGTRON, which is an extended asymmetric sigmoid with Perceptron, and its companion convex model called SIGTRON-imbalanced classification (SIC) model that employs a virtual SIGTRON-induced convex loss function. In contrast to the conventional $\pi$-weighted cost-sensitive learning model, the SIC model does not have an external $\pi$-weight on the loss function but has internal parameters in the virtual SIGTRON-induced loss function. As a consequence, when the given training dataset is close to the well-balanced condition considering the (scale-)class-imbalance ratio, we show that the proposed SIC model is more adaptive to variations of the dataset, such as the inconsistency of the (scale-)class-imbalance ratio between the training and test datasets. This adaptation is justified by a skewed hyperplane equation, created via linearization of the gradient satisfying $\epsilon$-optimal condition. 

Additionally, we present a quasi-Newton optimization(L-BFGS) framework for the virtual convex loss by developing an interval-based bisection line search. Empirically, we have observed that the proposed approach outperforms (or is comparable to) $\pi$-weighted convex focal loss and balanced classifier LIBLINEAR(logistic regression, SVM, and L2SVM) in terms of test classification accuracy with $51$ two-class and $67$ multi-class datasets. In binary classification problems, where the scale-class-imbalance ratio of the training dataset is not significant but the inconsistency exists, a group of SIC models with the best test accuracy for each dataset (TOP$1$) outperforms LIBSVM(C-SVC with RBF kernel), a well-known kernel-based classifier.
\end{abstract}


\section{Introduction}
Learning a hyperplane from the given training dataset ${\cal D} = \{ (x_l,y_l) \in \R^s \times \{ -1,+1 \} \;|\; l=1,2,\cdots,d \}$ is the most fundamental process while we characterize the inherent clustered structure of the test dataset. The main hindrance of the process is that the dataset is imbalanced~\cite{fernandez18,johnson19,oksuz21}, and there is an inconsistency between the training and test datasets~\cite{bach06}. To address the class-imbalance problem, one can apply under-sampling or over-sampling strategies while preserving the cluster structure of dataset ${\cal D}$~\cite{he09}. In addition to the class imbalance problem, there is another imbalance problem, known as scale imbalance, between the positive class $\{ x_i  \;|\; i \in {\cal N}_+ \}$ and the negative class $\{x_j  \;|\; j \in {\cal N}_- \}$ of ${\cal D}$~\cite{oksuz21}. Here, ${\cal N}_+ = \{ \ell \;|\; y_{\ell} = +1 \; \hbox{ and } \; (x_{\ell},y_{\ell}) \in {\cal D} \}$ and ${\cal N}_- = \{ \ell \;|\; y_{\ell} = -1\; \hbox{ and }\; (x_{\ell},y_{\ell}) \in {\cal D} \}$. Considering scale and class imbalance simultaneously, we generalize the class-imbalance ratio $r_c = \frac{\abs{{\cal N}_+}}{\abs{{\cal N}_-}}$ to the scale-class-imbalance ratio  
\begin{equation}\label{imbrsc}
r_{sc} = r_c\sqrt{\frac{\norm{x_p^c}^2+1}{\norm{x_n^c}^2+1}},
\end{equation} 
where $x_p^c = \frac{1}{\abs{{\cal N}_+}}\sum_{i \in {\cal N}_+} x_i$ is the centroid of the positive class of ${\cal D}$ and $x_n^c = \frac{1}{\abs{{\cal N}_-}}\sum_{j \in {\cal N}_-} x_j$ is the centroid of the negative class of ${\cal D}$. When $r_{sc}=1$ and $\norm{x_p^c-x_n^c} > a$ where $a$ is a positive constant, we say that ${\cal D}$ is {\it well-balanced} with respect to $r_{sc}$. See \cite{he09,oksuz21} for more details on imbalanced problems appearing in classification. It is worth mentioning that we can improve the scale imbalance through various normalization methods~\cite{garcia15,ioffe15}. In our experiments, we use the well-organized datasets in~\cite{delgado14}. They are normalized in each feature dimension with mean zero and variance one. This dimension-wise standardization is the first step of batch normalization~\cite{ioffe15} in deep learning. By way of the mean zero normalization, the scale-class-imbalanced ratio is always better than the class-imbalance ratio. The details are following.
\begin{theorem}\label{rsc-standard}
For $\{ x_i \;|\; i \in {\cal N}_+ \}$ and $\{ x_j \;|\; j \in {\cal N}_- \}$, let us consider mean zero normalization, $\sum_{i \in {\cal N}_+} x_i + \sum_{j \in {\cal N}_-} x_j=0$. Here $x_p^c \not= 0$ and $x_n^c \not= 0$. Then, if $r_{sc} =1$, we have $r_c=1$. On the other hand, if $r_{sc} \not=1$, we have  $\abs{r_{sc} - 1} < \abs{r_c-1}$.
\end{theorem}
\begin{proof}
From $\sum_{i \in {\cal N}_+} x_i = - \sum_{j \in {\cal N}_-} x_j$, we have $r_c x_p^c = - x_n^c$ and thus $r_c= \frac{\norm{x_n^c}}{ \norm{x_p^c}}$. By using \eqref{imbrsc} and $r_c= \frac{\norm{x_n^c}}{ \norm{x_p^c}}$, we have $\norm{x_n^c}^2(r_{sc}^2-1) = r_c^2-r_{sc}^2$ and thus the results are obtained.
\end{proof}
In addition to the mean zero condition, the variance one condition implies that $r_c \left( \frac{\sum_{i \in {\cal N}_+} x_i^2}{\abs{{\cal N}_+}} -1 \right) + \left(\frac{\sum_{j \in {\cal N}_-} x_j^2}{\abs{{\cal N}_-}} -1 \right) = 0$. By assuming that $\abs{{\cal N}_+}$ and $\abs{{\cal N}_-}$ are sizable, we have
$r_c (Var({\cal N}_+) + (x_p^c)^2 -1) = (1-Var({\cal N}_-) - (x_n^c)^2)$
where $Var({\cal N}_+) = \frac{1}{\abs{{\cal N}_+}} \sum_{i \in {\cal N}_+ } (x_i - x_p^c)^2$ and $Var({\cal N}_-) = \frac{1}{\abs{{\cal N}_-}} \sum_{j \in {\cal N}_- } (x_j - x_n^c)^2$. If the number of positive instances is sufficiently larger than the number of negative instances, then, from the dimension-wise standardization, we have $x_p^c \gtrsim 0$ and $Var({\cal N}_+) \approx 1$. On the contrary, $x_n^c$ is located away from $0$, and thus $Var({\cal N}_-)$ is relatively small. We have the opposing situation in the case $r_c \ll 1$. We notice that when $r_c$ is sufficiently far from $1$, the minority class is tiny and far from coordinate zero. It is hard to quantify the variance of the minority class. Thus, the corresponding decision boundary is somehow ambiguous. Even though a deep neural network is trained to the limit with batch normalization, it is hard to obtain reasonable classification results. This phenomenon is known as minority collapse~\cite{fang21}. This article is mainly interested in the case where $r_{sc}$ is not severe, but there exists an inconsistency of $r_{sc}$ between the training and test datasets~\cite{bach06}.

In cost-sensitive learning~\cite{bach06,cui19,fernandez18,he09,lin20}, we usually use the $\pi$-weighted loss function to learn a hyperplane decision boundary considering $r_c$ of the training dataset. Empirically, $\pi$-weight is set to be in proportion to the inverse of the size of each class. Additionally, it could be determined by the cross-validation~\cite{lin20}. However, it is still unclear how $\pi$-weight is related to $r_c$ of the training dataset. \cite{cui19} uses the inverse of the effective number of instances, known as the expected volume of instances, for the $\pi$-weight. For the dimension-wise standardized dataset~\cite{delgado14}, the inverse of the proposed $r_{sc}$~\eqref{imbrsc} is a possible option for $\pi$-weight. For imbalanced object detection, \cite{lin20} has proposed the famous non-convex focal loss function, which uses the power function of the probability related to the opposite class (and $\pi$-weight) on the logistic loss. They also suggested a $\pi$-weighted convex focal loss function based on standard logistic loss. Regarding test classification accuracy, the convex approach is comparable to the non-convex focal loss. Additionally, see~\cite{masnadi08,reid11} for designing large-margin loss functions and the corresponding $\pi$-weighted cost-sensitive loss functions based on Bregman-divergence.

This article shows the connection between the class-imbalance ratio (including the scale-class-imbalance ratio, the ratio of the effective number of instances, etc) and the loss function via a skewed hyperplane equation. Instead of the margin of each instance, the statistic of the margin distribution~ \cite{elsayed18,jiang18}, i.e., mean-margin, is considered while we describe the skewed hyperplane equation. One of the primary goals of this article is to suggest not a $\pi$-weighted loss function but a new class of adjustable convex loss functions by way of virtualization for novel cost-sensitive learning. For this purpose, we design SIGTRON(extended asymmetric sigmoid with Perceptron) and a novel cost-sensitive learning model, the SIGTRON-imbalanced classification (SIC) model. The proposed SIC model has internal polynomial parameters in the virtual SIGTRON-induced loss function instead of the external $\pi$-weight on the loss function. By the inherent internal structure of the parameters, when $r_{sc}$ (or $r_c$) of the training dataset is not severe, the SIC model is more adaptable to inconsistencies of (scale-)class-imbalance ratio between training and test datasets. We demonstrate the effectiveness of our model by conducting experiments on $51$ two-class datasets. For more information, refer to Figure \ref{fig:X} (a) in Section \ref{Sec5A} and Table \ref{2classimb} in Appendix \ref{appB} for $r_{sc}$ and $r_c$ of the binary class datasets.  

Before we go further, we present the definition of virtualization. {\it The virtual convex loss function $\ell$ is defined as a function satisfying $\nabla\ell = -p$ for the given probability function $p$.} For instance, the gradient of the logistic loss function is the negative canonical sigmoid (probability function) $\nabla \ell(x) = -\sigma(-x)$. Various variants of soft-max function and canonical sigmoid function, such as sparsemax~\cite{martins16}, sphericalmax~\cite{ollivier15}, Taylormax~\cite{brebisson16}, high-order sigmoid function~\cite{woo19a}, and other diverse activation functions~\cite{dubey22} are in the category of gradients of virtual loss functions. SIGTRON, which we will introduce in the coming Section \ref{secLogitronSIGTRON}, is also in this category. Although, in this article, we only consider S-shaped probability functions~\cite{murphy12,sigmoidwiki} for virtualization, they could be expandable to general functions. A typical example is the quasi-score function, of which the virtual loss function is the negative quasi-likelihood function defined by the mean and variance relation~\cite{mccullagh89,wedderburn74,woo19c}. In addition, virtual loss functions with monotonic gradient function include various ready-made adjustable convex loss functions, such as tunable loss function~\cite{liao18,sypherd19}, high-order hinge loss~\cite{dubey22,fan08,janocha17,lin02}, and Logitron~\cite{woo19a}. 

The other main goal of this article is to introduce a quasi-Newton optimization framework for cost-sensitive learning, including the proposed SIC model and $\pi$-weighted convex focal loss~\cite{lin20}. We name the presented optimization framework {\it quasi-Newton(L-BFGS) optimization for virtual convex loss}. In quasi-Newton(L-BFGS) optimization, the Hessian matrix is approximated by a rank-two symmetric and positive definite matrix, and its inverse matrix is algorithmically computed by simple two-loop iterations with $m$ recent elements. It generally uses sophisticated cubic-interpolation-based line search to keep positive definiteness. This line search heavily depends on the evaluation of loss function~\cite{nocedal06,schmidt05}. Instead of the well-known cubic-interpolation-based line search, we propose a relatively simple but accurate line search method, the interval-based bisection line search. With the relatively accurate strong Wolfe stopping criterion, the proposed method performs better than L-BFGS with the cubic-interpolation-based line search regarding test classification accuracy. Please refer to the details in Figure \ref{fig:losslessperformance}. Although we only consider virtual convex loss functions, which are smooth and bounded below, the proposed optimization framework could be extended to deep neural networks where non-convexity of loss functions is not severe~\cite{mutschler20}. It is worth mentioning that with the exact line search condition, the nonlinear conjugate gradient utilizes a larger subspace for Hessian matrix approximation~\cite{hager05,hager06,nocedal06}.

We justify the performance advantage of the proposed approach, the cost-sensitive SIC model and {\it quasi-Newton(L-BFGS) for virtual convex loss}, with $118$ various classification datasets~\cite{delgado14,woo19a}. For binary classification problems($51$ datasets) where $r_{sc}$ of training datasets is not severe, the test classification accuracy of TOP$1$(a group of SIC models having the best test accuracy for each dataset) is $83.96\%$, which is $0.74\%$ better than that of kernel-based LIBSVM(C-SVS with RBF kernel) and $0.16\%$ better than that of TOP$1$-FL of $\pi$-weighted convex focal loss. Within linear classifiers, the MaxA$(\alpha_+=\frac{7}{8},\alpha_-=\frac{8}{7})$ SIC model shows better performance than the $\pi$-weighted convex focal loss~\cite{lin20} and the balanced classifier LIBLINEAR(logistic regression, SVM, and L2SVM)~\cite{fan08,galli22} in terms of test classification accuracy with all $118$ datasets. Last but not least, the proposed SIC model with $(\alpha_+,\alpha_-)$-matrix parameters is a useful tool for understanding the structure of each dataset. For example, see Figure \ref{fig:example1} {\it spectf} dataset for $r_{sc}$-inconsistency between training and test datasets. The training dataset of {\it spectf} is well-balanced, but the test dataset of it is imbalanced~\cite{kurgan01}. For the multi-label structure, refer to Figure \ref{fig:pattern} (e) {\it energy-y1} dataset and (f) {\it energy-y2} dataset. They have the same input but opposite outputs, such as heating load vs. cooling load~\cite{tsanas12}.

\subsection{Notation}
We briefly review the extended exponential function~\cite{woo19b} and the extended logarithmic function~\cite{woo17}. For information on the Tweedie statistical distribution and beta-divergence based on extended elementary functions, refer to the following citations: \cite{amari16,jorgensen97,woo17,woo19b,woo19c}.

For notational convenience, let $\R_{\ge a} = \{ x \in \R \;|\; x \ge a \}$ and $\R_{> a} = \{ x \in \R \;|\; x > a\}$, where $a \in \R$. In the same way, $\R_{\le a}$ and $\R_{< a}$ are set. Then the extended logarithmic function $\ln_{\alpha,c}$~\cite{woo17}  and the extended exponential function $\exp_{\alpha,c}$~\cite{woo19b} are defined as follows: 
\begin{eqnarray}
\label{exlog}
\ln_{\alpha,c}(x) &=& \left\{\begin{array}{l} \ln\left(\frac{x}{c}\right), \quad\;\; \hbox{ if } \alpha=1 \\ c_{\alpha} - x_{\alpha}, \quad \hbox{ otherwise } 
\end{array}\right.\\
\label{exexp}
\exp_{\alpha,c}(x) &=&
\left\{\begin{array}{l} 
c\exp(x), \qquad\qquad\;\; \hbox{ if } \alpha=1\\
c(1 - \frac{x}{c_{\alpha}})^{1/(1-\alpha)} , \quad \hbox{ otherwise }
\end{array}\right.
\end{eqnarray}
where $c>0$, $\alpha \ge 0$, $x_{\alpha} = \frac{1}{\alpha-1} x^{1-\alpha}$ and $c_{\alpha} = \frac{1}{\alpha-1} c^{1-\alpha}$. Note that we also explain $c_{\alpha}$ as $(c)_{\alpha}$ for notational convenience. In the case where $c=1$, the extended functions $\exp_{\alpha,c}$ and $\ln_{\alpha,c}$ become the generalized exponential and logarithmic functions~\cite{amari16,ding10,sypherd19}, respectively. For the effective domains of $\ln_{\alpha,c}$ and $\exp_{\alpha,c}$, see \cite{rockafellar70,woo17,woo19b}. In this article, we only consider restricted domains of $\ln_{\alpha,c}$ and $\exp_{\alpha,c}$ in Table \ref{table7}. Within the restricted domains in Table \ref{table7}, irrespective of $\alpha_i$ and $c_i$, we have $\ln_{\alpha_2,c_2}(\exp_{\alpha_1,c_1}(x)) \in \R$ for all $x \in int(\dom(\exp_{\alpha_1,c_1}))$. This property defines the extended logistic loss, including high-order sigmoid function~\cite{woo19a}. Here, $int(E)$ means the largest open interval contained in an interval $E \subseteq \R$. Note that $\inprod{x}{y} = \sum_{l=1}^s x_ly_l$ for $x,y \in \R^s$, $\norm{x} = \sqrt{\inprod{x}{x}}$, and $\norm{x}_{\infty} = \max_l \abs{x_l}$. Additionally, $\abs{\cdot}$ means the absolute value or the size of a discrete set, depending on the context in which it is used. 

\begin{table}[h]
\centerline{\begin{tabular}{c||c|c|c}
 & $\alpha = 1$  & $0 \le \alpha < 1$  & $\alpha>1$  \\ \hline 
 $\dom (\ln_{\alpha,c})$ & $\R_{>0}$  & $\R_{\ge0}$  & $\R_{>0}$ \\ 
$\dom (\exp_{\alpha,c})$   & $\R$  & $\R_{\ge c_{\alpha}}$ & $\R_{< c_{\alpha}}$ \\ 
$\dom (\sigma_{\alpha,c})$   & $\R$  & $\R_{\le -c_{\alpha}}$ & $\R_{\ge -c_{\alpha}}$ \\ 
\end{tabular}}
\caption{Restricted domains of the extended logarithmic function $\ln_{\alpha,c}(x)$~\eqref{exlog}, the extended exponential function  $\exp_{\alpha,c}(x)$~\eqref{exexp}, and the extended asymmetric sigmoid function $\sigma_{\alpha,c}(x)$~\eqref{xSigF}. Here $c>0$. Note that, when $\alpha>1$, $\dom(\sigma_{\alpha,c})$ is a closed set $\R_{\ge -c_{\alpha}}$ in the sense $\sigma_{\alpha,c}(-c_{\alpha}) = \lim_{x \searrow -c_{\alpha}}\frac{c}{c + \exp_{\alpha,c}(-x)} = 0 \in \R$. }\label{table7}
\end{table}

\subsection{Cost-sensitive Learning framework and Overview\label{overviewandc}}
Let us start with the cost-sensitive learning model
\begin{equation}\label{costsensitive}
h^* = \argmin_{h \in {\cal H}}\; \sum_{i \in {\cal N}_+ } L_+(h(x_i)) +  \sum_{j \in {\cal N}_- } L_-(-h(x_j)) + \frac{\lambda}{2} Reg(w), 
\end{equation}
where ${\cal H} = \{ \inprod{w}{\cdot} +b \;|\; (w,b) \in \R^s \times \R \}$ and $Reg$ is an appropriate regularizer for $w$, such as $\norm{w}^2$. Note that $L_+$ and $L_-$ are virtualized large-margin convex loss functions that are differentiable and lower-bounded. This article only considers the case that $-\nabla L_{\pm} = p_{\pm} \in [0,1]$ are probability functions. For more information on cost-sensitive learning, please refer to~\cite{bach06,fernandez18,garcia15,he09,johnson19,lin20,oksuz21}.

In Section \ref{secLogitronSIGTRON}, we study the various properties of SIGTRON, such as smoothness, inflection point, probability-half point, and parameterized mirror symmetry of inflection point with respect to the probability-half point. SIGTRON is used to exemplify the probability functions $p_{\pm}$ in \eqref{costsensitive}. The details of the SIC model are discussed in Section \ref{seclossfunction}, where the virtual SIGTRON-induced loss function is introduced. In Section \ref{chardecision}, we derive the skewed hyperplane equation of the cost-sensitive learning model~\eqref{costsensitive}, based on the SIGTRON probability function. In Section \ref{SIGTRON-line-search}, we demonstrate the usefulness of {\it quasi-Newton optimization(L-BFGS) for virtual convex loss}, which includes the interval-based bisection line search. With this optimization method, we solve two different types of cost-sensitive learning models: the SIC model and the $\pi$-weighted convex focal loss. The performance evaluation of the proposed framework, i.e., the SIC model and {\it quasi-Newton optimization(L-BFGS) for virtual convex loss}, is done in Section \ref{secexperiment}. We compare the proposed framework with the imbalanced classifier $\pi$-weighted convex focal loss~\cite{lin20}, the balanced classifier LIBLINEAR(logistic regression, SVM, and L2SVM)~\cite{fan08,galli22}, and the nonlinear classifier LIBSVM(C-SVC with RBF kernel)~\cite{chang11}. The conclusion is given in Section \ref{seccon}.

\begin{figure*}[t]
\centering
\includegraphics[width=3.45in]{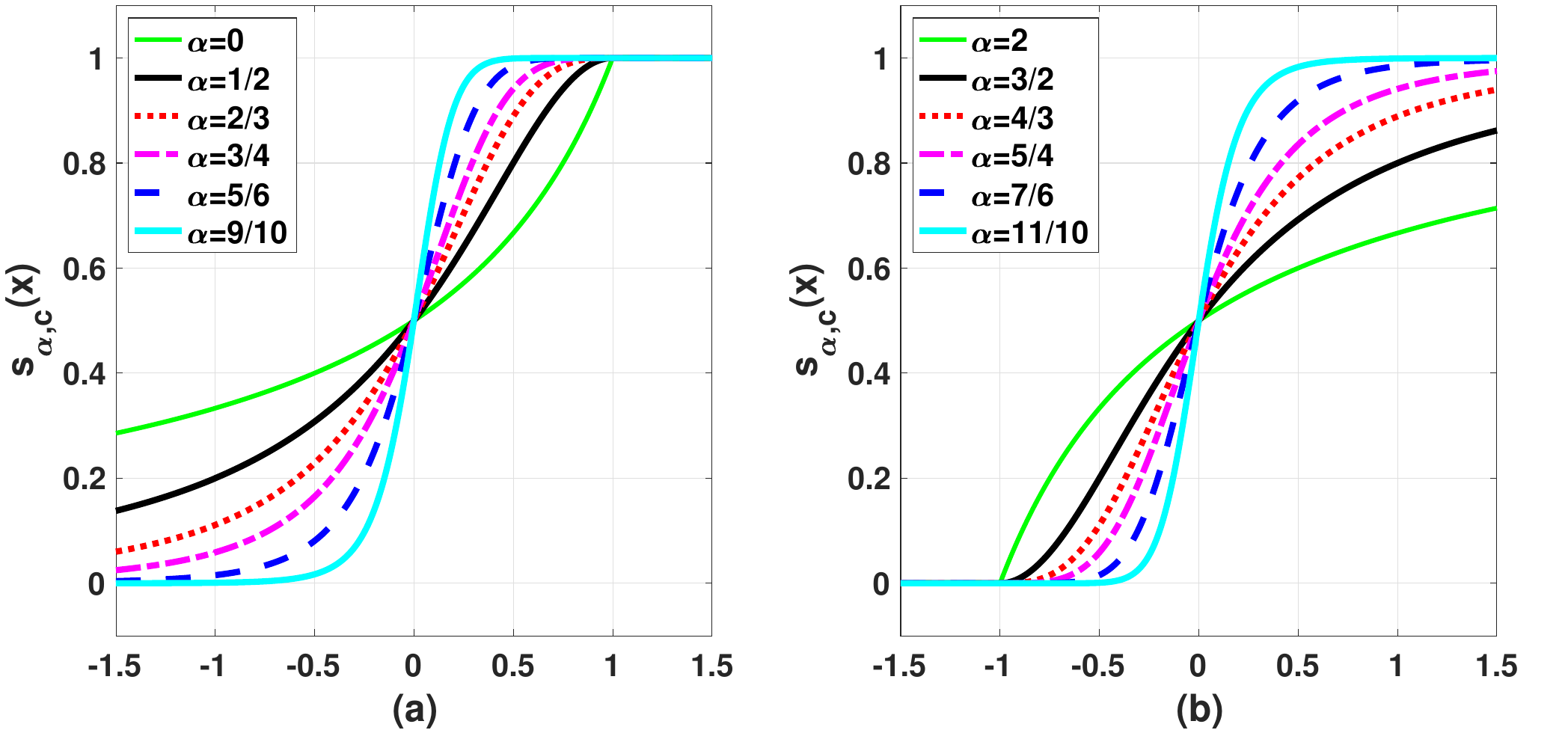}
\includegraphics[width=3.45in]{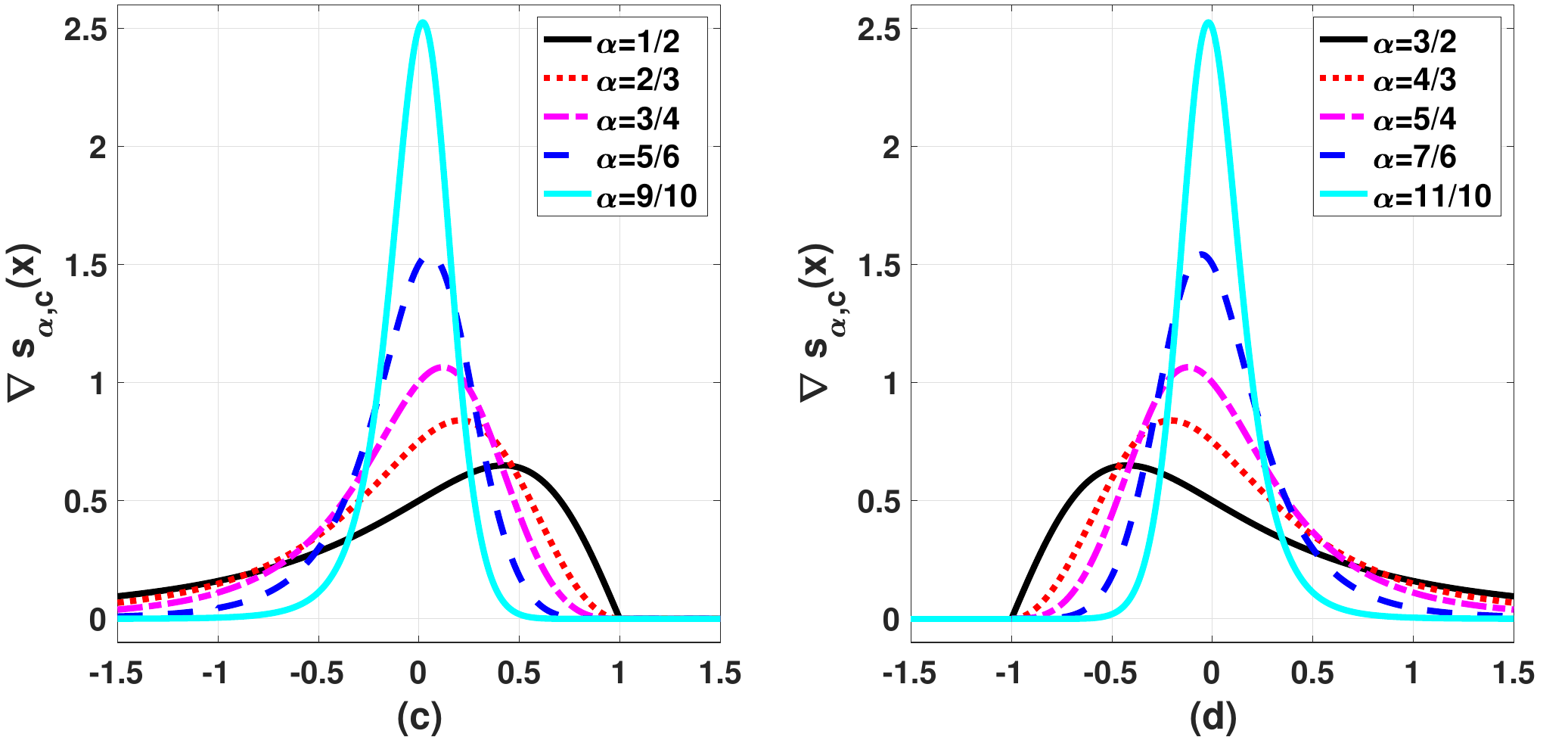}
\caption{(a) SIGTRON $s_{\alpha,c}(x)$ with $\alpha = \frac{k-1}{k}<1$ ($k=1,2,3,4,6,10$) and $c_{\alpha}=-1$. Note that $s_{\alpha,c}(x) = 1$, if $x \ge -c_{\alpha}$. (b) SIGTRON $s_{\alpha,c}(x)$ with $\alpha = \frac{k+1}{k}>1$ ($k=1,2,3,4,6,10$) and $c_{\alpha}=1$. Note that $s_{\alpha,c}(x) =  0$, if $x \le -c_{\alpha}$. (c) $\nabla s_{\alpha,c}(x)$ with $\alpha = \frac{k-1}{k}<1$ ($k=2,3,4,6,10$) and $c_{\alpha}=-1$. Note that $\nabla s_{\alpha,c}(x) = 0$, if $x \ge -c_{\alpha}$. The inflection point $x_{ip}$ is getting close to $-c_{\alpha} = 1$ as $\alpha \rightarrow 0$. (d) $\nabla s_{\alpha,c}(x)$ with $\alpha = \frac{k+1}{k}>1$ ($k=2,3,4,6,10$) and $c_{\alpha}=1$. Note that $\nabla s_{\alpha,c}(x) =  0$, if $x \le -c_{\alpha}$. The inflection point $x_{ip}$ is getting close to $-c_{\alpha} = -1$ as $\alpha \rightarrow 2$. }
\label{fig:img5}
\end{figure*}

\section{SIGTRON: extended asymmetric sigmoid with Perceptron~\label{secLogitronSIGTRON}}
In this Section, we define SIGTRON using the extended exponential function $\exp_{\alpha,c}$~\eqref{exexp}. We then study various properties of SIGTRON, such as its smoothness, inflection point, probability-half point, and parameterized mirror symmetry of the inflection point with respect to the probability-half point.
 
\begin{definition}[{\bf SIGTRON}]\label{Def:SIGTRON}
Let $ \alpha \ge 0$, $c >0$, and $x \in \R$. Then SIGTRON(extended asymmetric sigmoid with Perceptron) is defined as
\begin{equation}\label{sgtron}
s_{\alpha,c}(x) = \left\{
\begin{array}{l}
\sigma_{\alpha,c}(x) \qquad\hbox{ if } x \in \dom(\sigma_{\alpha,c})\\
\sigma_P(x) \qquad\;\; \hbox{ otherwise, }   
\end{array}
\right.
\end{equation}
where $\sigma_{\alpha,c}$ is the extended asymmetric sigmoid function
\begin{equation}\label{xSigF}
\sigma_{\alpha,c}(x) = \frac{c}{c+\exp_{\alpha,c}(-x)}.
\end{equation}
Here, $\exp_{\alpha,c}$ is the extended exponential function~\eqref{exexp} and 
$\sigma_P$ is the Perceptron function(or Heaviside function): $\sigma_P(x) = 1$, if $x \ge 0$ and $0$, otherwise. The restricted domains of $\exp_{\alpha,c}$ and $\sigma_{\alpha,c}$ are defined in Table \ref{table7}. Note that $s_{\alpha,c}(x) \in [0,1]$ is a non-decreasing continuous function defined on $\R$ with $\lim_{x \rightarrow -\infty} s_{\alpha,c}(x)= 0$ and $\lim_{x \rightarrow +\infty} s_{\alpha,c}(x)= 1$. Additionally, $s_{\alpha,c}(0) = 1/2$, irrespective of $\alpha$ and $c_{\alpha}$. Here $x_{ph}=0$ is denoted as the probability-half point. When $\alpha=1$, $s_{\alpha,c}(x) = \frac{1}{1+\exp(-x)}$ is the canonical sigmoid function, irrespective of $c$.
\end{definition}
Note that SIGTRON with $c=1$ becomes the canonical sigmoid function as $\abs{\alpha-1} \rightarrow 0$, since the extended exponential function with $c=1$ is the generalized exponential function. However, SIGTRON with $\abs{c_{\alpha}}=1$ becomes a smoothed Perceptron as $\abs{\alpha-1} \rightarrow 0$ and $\alpha \not=1$. Refer to Figure \ref{fig:img5} for additional information.

In the following Theorem \ref{th:derivative-bernoulli-like}, we characterize the smoothness of SIGTRON~\eqref{sgtron} depending on $\alpha$. The proof of Theorem \ref{th:derivative-bernoulli-like} is given in Appendix \ref{proofofnthderivative}. 
 \begin{theorem}\label{th:derivative-bernoulli-like}
For $n=1,2,3,\cdots$, when $\alpha \in \left(1-\frac{1}{n}, 1+ \frac{1}{n} \right)$, the $n$-th derivative of $s_{\alpha,c}$ is continuous on $\R$ and expressed as
\begin{equation}\label{nth-derivative_s}
\nabla^n s_{\alpha,c}(x) = \left\{\begin{array}{l} \sum_{k=1}^n F_{n,k}(x) \quad\hbox{ if } x \in \dom(\sigma_{\alpha,c}) \\\; 0 \qquad\quad\qquad\quad\; \hbox{ otherwise,} \end{array}\right.
\end{equation}
where 
\begin{equation}\label{nth-derivative_x}
F_{n,k}(x) =  A_{n,k}\left(\frac{1}{1-\alpha}\right) \frac{c \exp^{k-n(1-\alpha)}_{\alpha,c}(-x) }{ (c + \exp_{\alpha,c}(-x))^{k+1} },
\end{equation}
and
\begin{equation}\label{Ank}
A_{n,k}\left(\frac{1}{1-\alpha}\right) = (-1)^{n+k}k! \sum_{l=0}^n {\left[n \atop l\right]}{\left\{ l \atop k \right\}} (\alpha-1)^{n-l}.
\end{equation}
Here, $\left[ n \atop l \right]$ is the Stirling number of the first kind~~\cite{graham94} with the recurrence equation $\left[ n \atop l \right] = (n-1)\left[ n-1 \atop l \right] + \left[ n-1 \atop l-1 \right]$, where $n,l\ge1$. $\left\{ l \atop k \right\}$ is the Stirling number of the second kind with the recurrence equation $\left\{ l \atop k \right\} = k\left\{ l-1 \atop k \right\} + \left\{ l-1 \atop k-1 \right\}$, where $l,k\ge1$. 
\end{theorem}
For the computation of the Stirling number of the first kind and the second kind, we need additional notational conventions: $\left\{ 0 \atop 0 \right\} = \left[ 0 \atop 0 \right] = 1$ and $\left\{ a \atop 0 \right\} = \left[ a \atop 0  \right]=0$ for $a \ge 1$. We have $\left\{ a \atop 1 \right\} = 1$ and $\left[ a \atop 1 \right] = (a-1)!$ with $0!=1$, for $a\ge1$. Additionally, we note that $\left\{ a \atop b \right\} = \left[ a \atop b \right] = 0$ if $b>a \ge 0$. For more details, refer to \cite{graham94}. 

Theorem \ref{th:derivative-bernoulli-like} states that for any $\alpha \in (0,2)$, the gradient of $s_{\alpha,c}(x)$ is given by $c^{\alpha-1}(1-s_{\alpha,c}(x))^{\alpha}(s_{\alpha,c}(x))^{2-\alpha}$, where $x \in \R$. Check Figure \ref{fig:img5} (c) and (d) for a visual representation of $\nabla s_{\alpha,c}(x)$. The information regarding the inflection point of $s_{\alpha,c}$ is provided in Corollary \ref{inflectionpoint}. Additionally, we have observed that the function $\nabla s_{\alpha,c}(x)$ takes the form of the beta distribution $\beta_{D}(x;\alpha) = \frac{6}{\Gamma(3-\alpha)\Gamma(1+\alpha)}x^{2-\alpha}(1-x)^{\alpha}$, where $x\in[0,1]$. The cumulant distribution of the beta distribution, which has an adjustable parameter $\alpha$, can also be classified as an S-shaped sigmoid function.
\begin{corollary}\label{inflectionpoint}
For $\alpha \in (0,2)$, the inflection point $x_{ip}$ of SIGTRON $s_{\alpha,c}$ exists in the interval $int(\dom(\sigma_{\alpha,c}))$ and is expressed as 
\begin{equation*}
x_{ip} = -\ln_{\alpha,c}\left(\frac{c\alpha}{2-\alpha}\right). 
\end{equation*}
When $\alpha=1$, the inflection point is the probability-half point, that is, $x_{ip} = x_{hp}=0$. 
\end{corollary}
\begin{proof} From \eqref{nth-derivative_s} and Appendix \ref{proofofnthderivative}, we know that $s_{\alpha,c} \in C^{\infty}(int(\dom(\sigma_{\alpha,c})))$ and 
$
\nabla^2 s_{\alpha,c}(x) = \frac{-\alpha c \exp_{\alpha,c}^{2\alpha-1}(-x)}{(c+\exp_{\alpha,c}(-x))^2} +  \frac{2c \exp_{\alpha,c}^{2\alpha}(-x)}{(c+\exp_{\alpha,c}(-x))^3}. 
$
Let $\alpha \not=1$, then, since $\exp_{\alpha,c}(-x) \not= 0$ for $x \in int(\dom(\sigma_{\alpha,c}))$, the inflection point $x_{ip}$ is a point satisfying 
$
x_{ip} = - \ln_{\alpha,c}\left( \frac{c\alpha}{2-\alpha} \right).
$ If $\alpha=1$, then $s_{\alpha,c}$ is the canonical sigmoid function. Thus, $x_{ip}=x_{hp} = 0$.
\end{proof}

Figure \ref{fig:img5} shows $s_{\alpha,c}$ and its derivative for various choices of $\alpha$ satisfying $\abs{\alpha -1} = \frac{1}{k}$ ($k=1,2,3,4,6,10$) and $\abs{c_{\alpha}}=1$. Note that $\nabla s_{\alpha,c}$ is not defined at $\alpha=0$ and $\alpha=2$. When $\alpha>1$, the inflection point $x_{ip}$ is getting close to $-1$ as $\alpha \rightarrow 2$. On the other hand, when $\alpha<1$, the inflection point $x_{ip}$ is getting close to $1$ as $\alpha \rightarrow 0$. 

\begin{remark}\label{remark:mirror}
SIGTRON is a general framework for replacing the S-shaped sigmoid function in diverse machine learning problems requiring adjustability of probability(or inflection point) and fixed probability-half point. For instance, refer to the skewed hyperplane equation for classification~\eqref{gradF4} and Example \ref{example:prob}. As canonical sigmoid function $\sigma(x) = \frac{1}{1+\exp(-x)}$ has a symmetric property $\sigma(x) = 1 - \sigma(-x)$, SIGTRON $s_{\alpha,c}$ also has an extended symmetric property:
\begin{equation}\label{sigmirror}
s_{\alpha,c}(x) = 1 - s_{2-\alpha,c^{-1}}(-x)
\end{equation}
where $\alpha \in [0,2]$. Also, for  $\alpha \in (0,2)$, we have $\nabla s_{\alpha,c}(x) = \nabla s_{2-\alpha,c^{-1}}(-x)$, the parameterized mirror symmetry with respect to probability-half point $x_{hp}=0$. See Figure \ref{fig:img5} (c) and (d) for examples of parameterized mirror symmetry of $\nabla s_{\alpha,c}$. It is worth commenting that the gradient of Logitron $L_{\alpha,c}$~\cite{woo19a} is also a negative probability function, of which the probability-half point depends on $\alpha$. For $\alpha \in (0,2]$, we have 
\begin{equation}\label{probX}
\nabla L_{\alpha,c}(x) = - (s_{2-\alpha,c^{-1}}(-x))^{\alpha}
\end{equation}
where the exponent $\alpha$ is an acceleration parameter of SIGTRON $s_{2-\alpha,c^{-1}}(-x)$ and \eqref{sigmirror} is used. 
\end{remark}

\begin{example}\label{example:prob}
It is well-known that it is hard to give a probability for the results of max-margin SVM classifier~\cite{lin07,platt99}. In fact, \cite{platt99} uses the canonical sigmoid function $\sigma(\gamma x+ \xi)$ to fit a probability to the classified results of the SVM. Here $\gamma$ and $\xi$ should be estimated~\cite{chang11}. Instead of fitting with the canonical sigmoid function $\sigma(\gamma x+ \xi)$, we could use SIGTRON $s_{\alpha,c}$ as a probability estimator for the results of the SVM classifier or any other classifiers having decision boundary, such as hyperplane. For this purpose, there are three steps to follow. First, we must place the probability-half point $x_{hp}$ of $s_{\alpha,c}$ at the decision boundary. Second, we should adjust $c_{\alpha}$ to place the exact probability-one point of $s_{\alpha,c}$ at a specific point, such as the maximum margin point. Finally, we only need to estimate $\alpha$ for the decreasing slope of $s_{\alpha,c}$ based on the distribution of classified results. See \cite{guo17} for the probability estimation issues in deep neural networks.
\end{example}

\section{Virtual SIGTRON-induced loss function, SIC(SIGTRON-imbalanced classification) model, and skewed hyperplane equation\label{seclossfunction}}
This Section studies the SIC model with the virtual SIGTRON-induced loss functions and the skewed hyperplane equation of the SIC model. 

\begin{definition}\label{definition-SIGTRON-induced-loss-function}
Let $\alpha \in [0,2]$, $c >0$, and $x \in \R$, then the virtual SIGTRON-induced loss function $L_{\alpha,c}^S$ is defined by the following gradient equation 
\begin{equation}\label{SIGTRON-induced-loss}
\nabla L^S_{\alpha,c}(x) = s_{\alpha,c}(x)-1,
\end{equation}
where $s_{\alpha,c}(x) -1$ is a negative probability function. By the extended symmetric property of SIGTRON in \eqref{sigmirror}, we have $s_{\alpha,c}(x) -1= -s_{2-\alpha,c^{-1}}(-x)$.
\end{definition}
We notice that an expansion of the class of Logitron loss~\eqref{probX} via virtualization is easily achieved by $\nabla {L}_{\beta,\alpha,c}(x) = - (s_{2-\alpha,c^{-1}}(-x))^{\beta}$ where $\beta>0$ is a tuning parameter which controls the location of probability-half point $x_{hp}$. Thus, the virtualized Logitron loss contains both the virtual SIGTRON-induced loss~\eqref{SIGTRON-induced-loss} and the Logitron loss~\eqref{probX}. 

\begin{lemma}\label{SIGTRON-induced-loss-Th} 
Let $\alpha \in [0,1)\cup(1,2]$ and $c > 0$. Then the virtual SIGTRON-induced loss function $L^S_{\alpha,c}$ satisfying \eqref{SIGTRON-induced-loss} has the following integral formulations: \newline
(1) Case $\alpha \in (1,2]$:
\begin{equation}\label{SIGTRON-induced-loss2}
L_{\alpha,c}^S(x) = \left\{ 
\begin{array}{l}
-c_{\alpha}F\left(1+ \frac{x}{c_{\alpha}};\alpha-1\right) + c_{\alpha}\;\; \hbox{ if } x \ge -c_{\alpha}\\
-x \hskip 1.6in \hbox{ otherwise. }
\end{array}
\right.
\end{equation}
(2) Case $\alpha \in [0,1)$:
\begin{equation}\label{SIGTRON-induced-loss1}
L_{\alpha,c}^S(x) = \left\{ 
\begin{array}{l}
c_{\alpha}F\left(1+ \frac{x}{c_{\alpha}};1-\alpha\right) - c_{\alpha} - x \;\; \hbox{ if } x \le -c_{\alpha}\\
0 \hskip 1.9in \hbox{ otherwise. } 
\end{array}
\right.
\end{equation}
Here, $F(z;b) = \int_0^{z} \frac{1}{1+t^{1/b}}dt$ with $z \in \R_{\ge 0}$ and $b>0$.
\end{lemma}
\begin{proof}
(1) Case $\alpha \in (1,2]$:  From \eqref{SIGTRON-induced-loss}, we have
\begin{eqnarray*}
\nabla L_{\alpha,c}^S(x) = \left\{\begin{array}{l} 
-\frac{1}{ 1 + (1 + \frac{x}{c_{\alpha}})^{\frac{1}{\alpha-1}} }  \;\;\;\;\;\hbox{ if } x \ge -c_{\alpha} \\ 
-1 \qquad\qquad\qquad\;\; \hbox{ otherwise, }  
\end{array}\right.
\end{eqnarray*}
where $-c_{\alpha} <0$ and $1+\frac{x}{c_{\alpha}} \ge 0$. The integration of $\nabla L^S_{\alpha,c}$ becomes
\begin{eqnarray*}
L^S_{\alpha,c}(a_1) - L^S_{\alpha,c}(a_0) = \int_{a_0}^{a_1} \nabla L^S_{\alpha,c}(t)dt =  \left\{\begin{array}{l} 
-c_{\alpha}F(1+a_1/c_{\alpha};\alpha-1) + c_{\alpha} + a_0,  \;\hbox{ if } a_1 \ge -c_{\alpha} \\ 
-a_1 + a_0, \hskip 1.67in \hbox{ otherwise, }  
\end{array}\right.
\end{eqnarray*}
where we may choose $a_0 \ll -c_{\alpha}$. Then, we get the virtual SIGTRON-induced loss function \eqref{SIGTRON-induced-loss2}, after setting $a_1=x$ and removing constants. 

(2) Case $\alpha \in [0,1)$: We have
$$
\nabla L_{\alpha,c}^S(x) =  \left\{\begin{array}{l} 
\frac{1}{1 + (1+\frac{x}{c_{\alpha}})^{\frac{1}{1-\alpha}} } - 1 \quad\quad \hbox{ if } x \le -c_{\alpha}\\ 
0 \hskip 1.25in \hbox{ otherwise, }
\end{array}\right.
$$
where $-c_{\alpha} > 0$ and $1+\frac{x}{c_{\alpha}} \ge 0$. Thus, we get
\begin{eqnarray*}
L^S_{\alpha,c}(a_0) - L^S_{\alpha,c}(a_1) = \int_{a_1}^{a_0} \nabla L^S_{\alpha,c}(t)dt =  \left\{\begin{array}{l} 
-c_{\alpha}F(1+ \frac{a_1}{c_{\alpha}};1-\alpha) + c_{\alpha} + a_1 \quad \hbox{ if } a_1 \le -c_{\alpha} \\ 
0 \hskip 2.1in \hbox{ otherwise, }
\end{array}\right.
\end{eqnarray*}
where $a_0 > -c_{\alpha}$. Let $a_1=x$, then we get the virtual SIGTRON-induced loss function \eqref{SIGTRON-induced-loss1}. 
\end{proof}

\begin{figure}[t]
\centering
\includegraphics[width=4in]{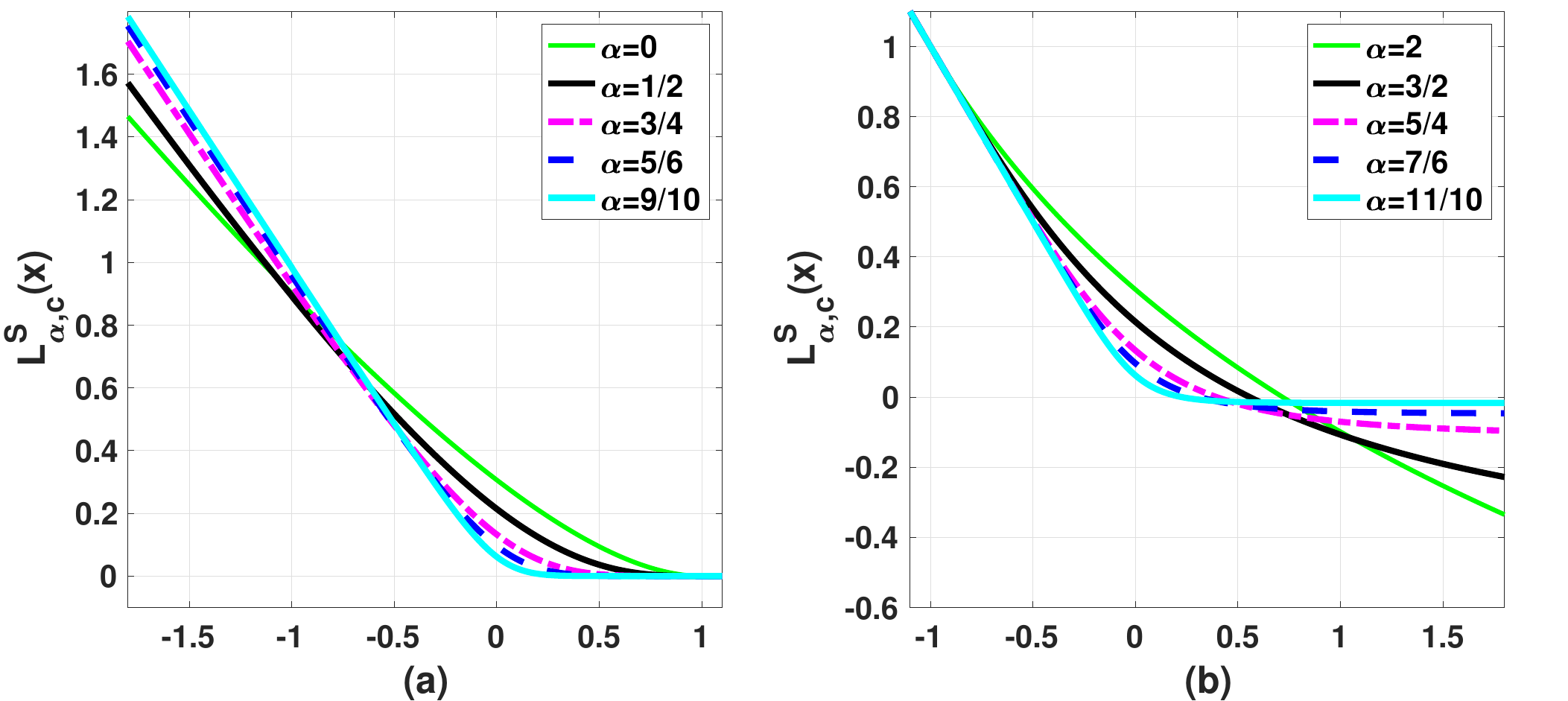}
\caption{Graphs of the virtual SIGTRON-induced loss function $L^S_{\alpha,c}$ for (a) $\alpha= \frac{k-1}{k}$ with $c_{\alpha}=-1$ and (b) $\alpha = \frac{k+1}{k}$ with $c_{\alpha}=1$. Here $k=1,2,4,6$, and $10$. In the case of $k=1,2,4,6$, $L_{\alpha,c}^S$ has a closed-form expression. See Example \ref{F-remark}.} 
\label{fig:SIGTRON}
\end{figure}

In Figure \ref{fig:SIGTRON}, we present the virtual SIGTRON-induced loss $L_{\alpha,c}^S(x)$ with $\abs{c_{\alpha}}=1$ and  $\abs{\alpha-1} = \frac{1}{k}$. Here $k=1,2,4,6$ and $10$ are the polynomial orders of $\exp_{\alpha,c}$. For $k=10$, $L_{\alpha,c}^S(x)$ is computed directly by \eqref{SIGTRON-induced-loss2} and \eqref{SIGTRON-induced-loss1}. For $k=1,2,4,6$, $L_{\alpha,c}^S(x)$ is expressed in a closed form by virtue of Example \ref{F-remark}. As we increase the polynomial order $k = \frac{1}{\abs{\alpha-1}}$, i.e. $\alpha \rightarrow 1$ and $\abs{c_{\alpha}}=1$, $L_{\alpha,c}^S(x)$ is getting close to the smoothed Perceptron loss function~\cite{vaswani19x}, not to the logistic loss.

\begin{example}\label{F-remark}
We make a list of $F(a;1/k)$ for $k=1,\cdots,6$. 
\begin{itemize}
\item $k=1:$ $F(a;1/1) = \ln(1+a)$
\item $k=2:$ $F(a;1/2) = \arctan(a)$
\item $k=3:$ $F(a;1/3) = \frac{1}{6}\log\left( 1 + \frac{3a}{a^2-a+1} \right)+\frac{1}{\sqrt{3}}\arctan\left(\frac{2a-1}{\sqrt{3}}\right) - \frac{1}{\sqrt{3}}\arctan\left(-\frac{1}{\sqrt{3}}\right)$
\item $k=4:$ $F(a;1/4) = \frac{1}{4\sqrt{2}}\log\left( 1 + \frac{2\sqrt{2}a}{a^2-\sqrt{2}a+1} \right) + \frac{1}{2\sqrt{2}} \arctan\left(1+\sqrt{2} a \right) - 2 \arctan(1 - \sqrt{2}a )$
\item $k=5:$ $F(a;1/5) =  \frac{(\sqrt{5} -1)}{20}\log(2a^2+(\sqrt{5}-1)a+2) - \frac{(\sqrt{5}+1)}{20}\log(2a^2-(\sqrt{5}+1)a+2) + \frac{\log(1+a)}{5}\\~~~~~~~~~~~~~~~~~~~~~~~~~~~~
 - \frac{\sqrt{10-2\sqrt{5}}}{10}\arctan\left(\frac{-4a + \sqrt{5} +1}{\sqrt{10-2\sqrt{5}}}\right) + \frac{\sqrt{10+2\sqrt{5}}}{10}\arctan\left(\frac{4a + \sqrt{5} - 1}{\sqrt{10+2\sqrt{5}}}\right)  - \frac{(\sqrt{5} -1)\log 2 + (\sqrt{5}+1)\log 2}{20}\\~~~~~~~~~~~~~~~~~~~~~~~ +\frac{\sqrt{10-2\sqrt{5}}}{10}\arctan\left(\frac{\sqrt{5}+1}{\sqrt{10-2\sqrt{5}}}\right) - \frac{\sqrt{10+2\sqrt{5}}}{10}\arctan\left(\frac{\sqrt{5} -1}{\sqrt{10 + 2\sqrt{5}}}\right)$
\item $k=6:$ $F(a;1/6) =  \sqrt{3}\log\left(\frac{a^2 + \sqrt{3}a+1}{a^2-\sqrt{3}a+1}\right) + \arctan( \sqrt{3} + 2a)/6 -\arctan( \sqrt{3} - 2a)/6 + \arctan(a)/3$
\end{itemize}
\end{example}

\subsection{Learning a hyperplane with SIC model\label{chardecision}}
Let us first consider the cost-sensitive convex minimization model~\eqref{costsensitive} to find a hyperplane $h^*(x) = 0$ from the given training dataset ${\cal D}$. The following is the realization of \eqref{costsensitive} through the virtual SIGTRON-induced loss function~\eqref{SIGTRON-induced-loss} and $\ell_2$-regularizer.
\begin{equation}\label{linmin}
h^* = \argmin_{ h \in {\cal H}}\;  {\cal F}(h)
\end{equation}
where ${\cal H} = \{ \inprod{w}{\cdot}+b \;|\; (w,b) \in \R^s \times \R \}$ and
\begin{equation}\label{lossf}
{\cal F}(h) = \sum_{i \in {\cal N}_+ } L^S_{\alpha_+,c_+}(h(x_i)) +  \sum_{j \in {\cal N}_- } L^S_{\alpha_-,c_-}(-h(x_j)) + \frac{\lambda}{2} \norm{w}_2^2.
\end{equation}
This minimization problem~\eqref{linmin} with~\eqref{lossf} is named as the SIGTRON-imbalanced classification(SIC) model. In the following example, we introduce a feature of the SIC model which is similar to the SVM.
\begin{example}\label{exactcase}
Let us assume that the training dataset ${\cal D}$ is separable, $\alpha_+, \alpha_- \in [0,1)$, $\lambda=0$, and $\norm{w^*}=1$. When ${\cal F}(h^*) = 0$, from Lemma \ref{SIGTRON-induced-loss-Th}, we have $h^*(x_i) \ge - (c_+)_{\alpha_+}$ for all $x_i \in {\cal N}_+$ and $h^*(x_j) \le (c_-)_{\alpha_-}$ for all $x_j \in {\cal N}_-$. Let $x_p^m$ be a point satisfying $h^*(x_p^m) = \min_{i \in {\cal N}_+} h^*(x_i)$ and $x_n^m$ be a point satisfying $h^*(x_n^m) = \max_{j \in {\cal N}_-} h^*(x_j)$. Then, if we set $-(c_+)_{\alpha_+} = h^*(x_p^m)$ and $(c_-)_{\alpha_-} = h^*(x_n^m)$, we have the following skewed hyperplane equation
\begin{equation}\label{SVM-like-skewed-hyperplane}
\left\{\; x \in \R^s \;|\; \inprod{w^*}{x - \frac{x_p^m + x_n^m}{2}} = \frac{(c_+)_{\alpha_+}-(c_-)_{\alpha_-}}{2} \;\right\}.
\end{equation}
Here, when $(c_+)_{\alpha_+} = (c_-)_{\alpha_-}$, we have the max-margin hyperplane $\{ x \in \R^s \;|\; \inprod{w^*}{x - \frac{x_p^m + x_n^m}{2}} = 0 \}$.
\end{example}
 In the max-margin region where ${\cal F}(h^*)=0$ with $\lambda=0$, $r_c$ does not affect the location of the hyperplane. We also notice that $\pi$-weight on the loss function is meaningless in this region. As observed in Example \ref{exactcase}, margin parameters $(c_+)_{\alpha_+}$ and $(c_-)_{\alpha_-}$ of the SIC model~\eqref{linmin} are mainly relevant to the location of the (skewed) hyperplane. See \cite{byrd19,ji20,soudry18} for related issues in deep learning where the max-margin region exists at infinity.

From now on, we are interested in $h^*$, satisfying the following $\epsilon$-optimal condition
\begin{equation}\label{firstoptc}
\norm{\nabla {\cal F}(h^*)}_{\infty} \le \epsilon,
\end{equation}
where ${\cal F}$ is defined in \eqref{lossf}. This $\epsilon$-optimal condition~\eqref{firstoptc} is used as a gradient-based stopping criterion while we do numerical experiments in Section \ref{secexperiment}. The gradient of ${\cal F}(h^*)$ becomes
\begin{equation}\label{gradF}
\nabla {\cal F}(h^*) = \sum_{i \in {\cal N}_+ } - p_+(h^*(x_i))\left[\begin{array}{c} x_i \\ 1 \end{array}\right] +  \sum_{j \in {\cal N}_- } p_-(-h^*(x_j))  \left[\begin{array}{c} x_j \\ 1 \end{array}\right] + \lambda w^*
\in [-\epsilon {\bf 1},+\epsilon {\bf 1}]
\end{equation}
where $p_{+} = -\nabla L_{\alpha_+,c_+}^S, p_{-} = -\nabla L_{\alpha_-,c_-}^S  \in (0,1)$, and ${\bf 1}$ is all one vector in $\R^{s+1}$. Let us assume that $p_{\pm}$ is twice differentiable and $\nabla p_{\pm}(0) \not= 0$. Then by linearization of $p_{\pm}$ at $0$~\cite{elsayed18}, we have $p_{\pm}(a) = p_{\pm}(0) + a \nabla p_{\pm}(0) + D_{p_{\pm}}(a | 0)$ where $D_{p_{\pm}}(a|0) = \int_{0}^{a} \nabla^2p_{\pm}(y)(a-y)dy$ and thus, in terms of $b^*$, \eqref{gradF} becomes
\begin{eqnarray*}
\sum_{i \in {\cal N}_+ } \left(p_{+}(0) + h^*(x_i) \nabla p_{+}(0) + D_{p_{+}}(h^*(x_i) | 0)\right) - \sum_{j \in {\cal N}_- }  \left(p_{-}(0) - h^*(x_j) \nabla p_{-}(0) + D_{p_{-}}(-h^*(x_j) | 0)\right)  \in [-\epsilon,+\epsilon].
\end{eqnarray*}
This equation simplifies to
\begin{equation}\label{gradF2}
r_c\left( p_{+}(0) + h^*(x_p^c)\nabla p_{+}(0) + E_+ \right) - \left( p_{-}(0) - h^*(x_n^c) \nabla p_{-}(0) + E_- \right) \in I_{\epsilon} 
\end{equation}
where $E_+ =  \frac{\sum_{i \in {\cal N}_+}D_{p_{+}}(h^*(x_i) | 0)}{\abs{{\cal N}_+ }}$, $E_- = \frac{\sum_{j \in {\cal N}_-}D_{p_{-}}(-h^*(x_j) | 0)}{\abs{{\cal N}_- }}$, and $I_{\epsilon} = [\frac{-\epsilon}{\abs{{\cal N}_-}}, \frac{+\epsilon}{\abs{{\cal N}_-}}]$. Since $h^*(x) =\inprod{w^*}{x}+b^*$, we have an expression below explaining $b^*$ from the $\epsilon$-optimal condition~\eqref{firstoptc}
\begin{equation}\label{gradF3}
\inprod{w^*}{A} +b^* \in -E,
\end{equation}
where $A = \frac{r_c \nabla p_+(0) x_p^c + \nabla p_-(0) x_n^c}{ r_c \nabla p_+(0) + \nabla p_-(0)}$ and $E =  \frac{(r_c p_+(0) - p_-(0)) + (r_cE_+ - E_-) +  I_{\epsilon}}{ r_c \nabla p_+(0) + \nabla p_-(0) }$. Now, the skewed hyperplane equation for $h^*(x)=0$ is expressed as
\begin{equation}\label{gradF4}
\{\; x \in \R^s \;|\; \inprod{w^*}{x-A} \in E \;\}.
\end{equation}
Through this skewed hyperplane equation, we could somehow understand a mysterious connection between the hyperplane decision boundary and the loss function. Depending on the skewness level $E$, the hyperplane decision boundary $h^*(x)=0$ may not be located between $x_p^c$ and $x_n^c$. The effect of $E$, however, can be discarded in a particular dataset structure. The details are following.  
\begin{example}\label{sym1}
Let us assume that the given dataset ${\cal D}$ has a symmetric structure with $r_c=1$. That is, there is a one-to-one correspondence between the positive dataset ${\cal N}_+$ and the negative dataset ${\cal N}_-$ with respect to the hyperplane $h^*(x) = 0$. For any $x_i \in \R^s$ where $i \in {\cal N}_+$, there is an unique $x_j \in \R^s$ where $j \in {\cal N}_-$, such that $h^*(x_i) \approx -h^*(x_j)$. The opposite is also true. If we additionally assume that $p_-(x)=p_+(x)$ then we obtain $E_+ \approx E_-$. Thus, the skewed hyperplane equation~\eqref{gradF4} becomes 
\begin{equation}\label{hyperEqx}
\left\{ x \in \R^s \;|\; \inprod{w^*}{x-A} \approx 0 
\right\}
\end{equation}
where $A = \frac{x_p^c+x_n^c}{2}$. 
\end{example}
By using a physically symmetric and separable dataset, \cite{lyu21} analyzes the structure of the Leaky ReLU two-layer neural network in the max-margin region at infinity. Please refer to \cite{vardi23} for a review of the implicit bias, including max-margin region at infinity, in machine learning including deep neural networks. In the following Theorem, we summarize the skewed hyperplane equation obtained by the SIC model~\eqref{linmin} under the condition $\abs{r_c E_+ - E_-} \ll 1$.

\begin{theorem}\label{thseparable}
Let $\norm{x_p^c-x_n^c}>a$ where $a>0$, $0< h^*(x_p^c)  < 1$, and $0< - h^*(x_n^c) < 1$, $\alpha_{\pm} \in [0,1)\cup(1,2]$. Additionally, assume that $\abs{r_c E_+ - E_-} \ll 1$. Then the skewed hyperplane equation~\eqref{gradF4} becomes
\begin{equation}\label{centerB}
\left\{\; x \in \R^s \;|\; \inprod{w^*}{ x - \left( \frac{ r_{c} c_+^{\alpha_+-1}x_p^c + c_-^{\alpha_--1}x_n^c}{r_{c} c_+^{\alpha_+-1} + c_-^{\alpha_- -1}} \right)}  \approx \frac{2(r_{c}-1)}{r_{c} c_+^{\alpha_+-1} +  c_-^{\alpha_--1}} \; \right\}.
\end{equation}
If $r_{c}=1$ then \eqref{gradF3} is expressed as $\inprod{w^*}{(1-\eta) x_p^c + \eta x_n^c} + b^* \approx 0$ and the signed distance of $x_p^c$ to the hyperplane $h^*(x)=0$ is approximately given as 
\begin{equation}\label{simplemargin}
\frac{h^*(x_p^c)}{\norm{w^*}} \approx \eta \norm{x_p^c-x_n^c} cos(\theta_+) 
\end{equation}
where $\eta = \frac{ c_-^{\alpha_--1}}{c_+^{\alpha_+-1} + c_-^{\alpha_- -1}} \in (0,1)$ and $cos(\theta_+) = \inprod{\frac{w^*}{\norm{w^*}}}{\frac{x_p^c-x_n^c}{\norm{x_p^c-x_n^c}}}>0$. In the same way, for $x_n^c$, we have $\frac{h^*(x_n^c)}{\norm{w^*}} \approx (\eta-1)\norm{x_p^c-x_n^c}cos(\theta_+)$.
\end{theorem}

 \begin{figure*}[t]
\centering
\includegraphics[width=3.15in]{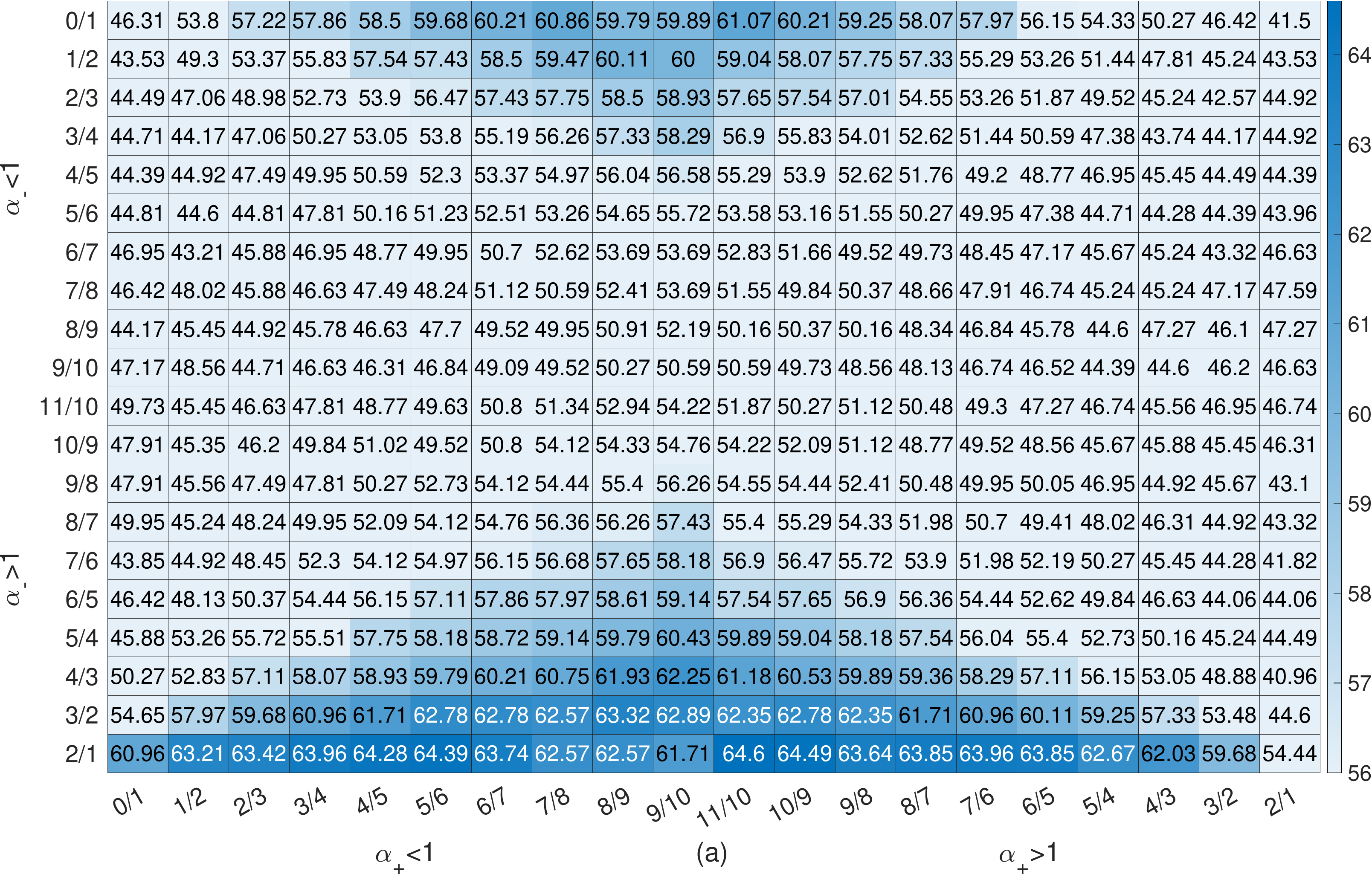} \hskip 0.25in
\includegraphics[width=3.15in]{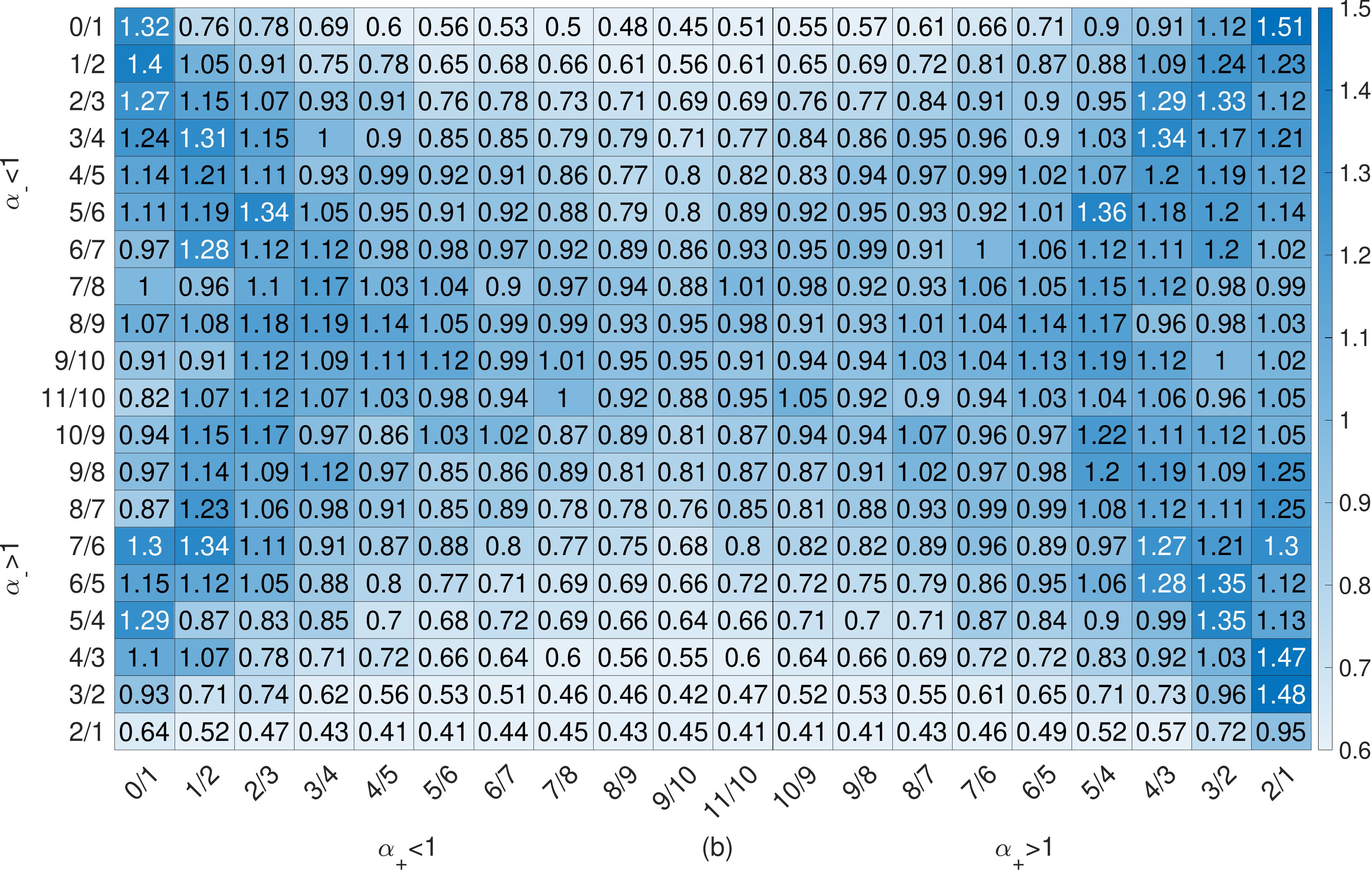} \vskip 0.1in
\includegraphics[width=3.15in]{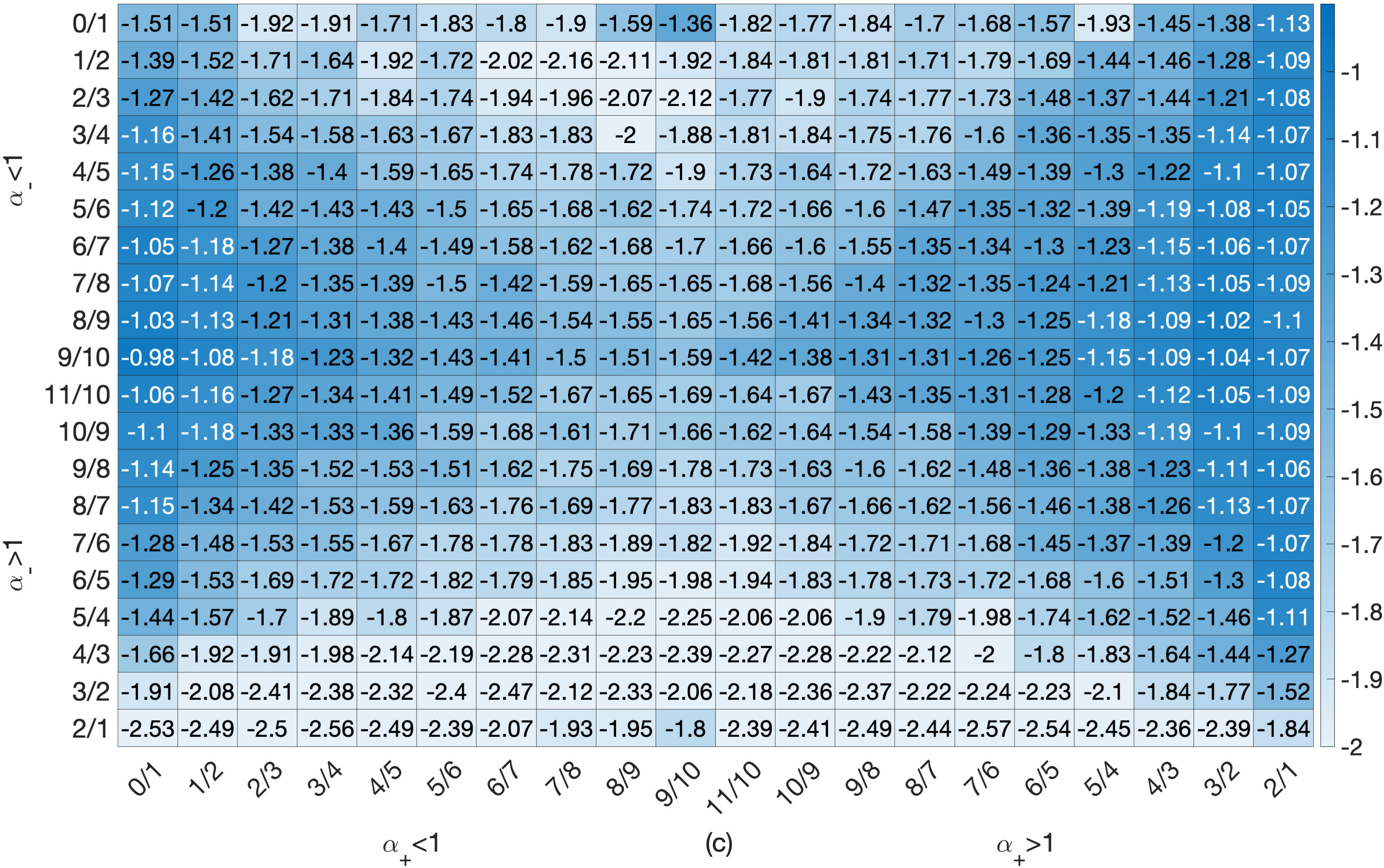} \hskip 0.25in 
\includegraphics[width=3.15in]{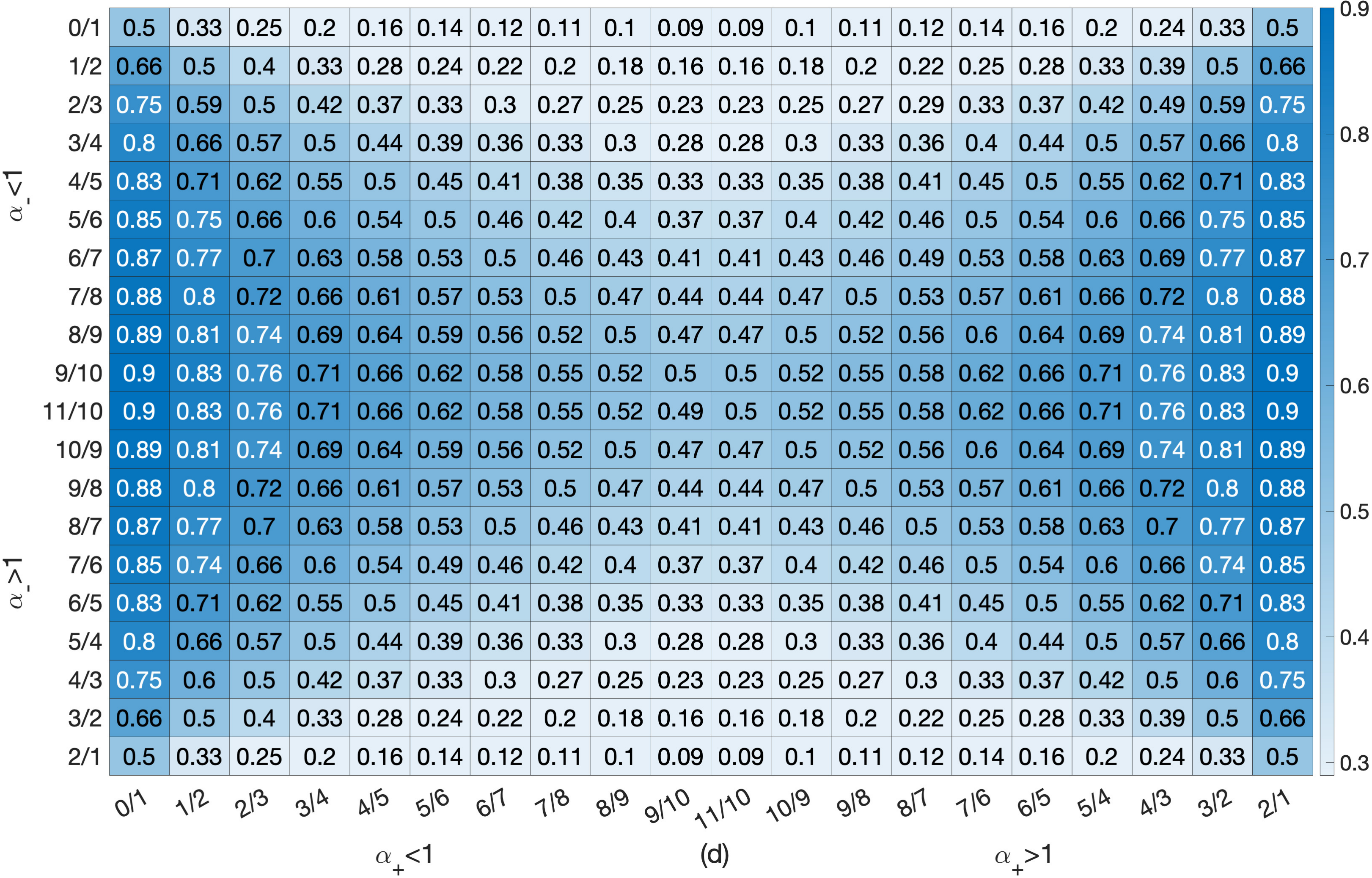}
\caption{Classification results with the {\it spectf} dataset in Table \ref{2classimb}. The test dataset has $r_{sc} = 0.26(r_c=0.09)$. However, the training dataset is well-balanced, i.e., $r_{sc}=1$. We have $20\times20$ hyperplanes $h^*_{(\alpha_+,\alpha_-)}(x)=0$ by solving $20\times20$ SIC models~\eqref{linmin} with the well-balanced training dataset. (a) The pattern of the test classification accuracy. (b) The pattern of the signed distance of the centroid of the positive test dataset to the hyperplane. (c) The pattern of the signed distance of the centroid of the negative test dataset to the hyperplane. (d) The pattern of $\eta$ in \eqref{simple-eta}. The best test accuracy is achieved at $(\alpha_+,\alpha_-) = (\frac{11}{10},2)$. This point is the smallest distance of the centroid of the positive test dataset to the hyperplane. And, it is contained in the group $\left\{ (\frac{11}{10},2), (\frac{9}{10},2), (\frac{11}{10},0), (\frac{9}{10},0) \right\}$ having the smallest $\eta = \frac{1}{11}$. See Example \ref{example2} for more details.}  
\label{fig:example1}
\end{figure*}

In practice, due to computational constraints, we normally choose polynomial functions for $\exp_{\alpha,c}$, i.e., positive integers for $\frac{1}{\abs{\alpha_{\pm}-1}}$. Assume that $\abs{(c_+)_{\alpha_{+}}} = \abs{(c_-)_{\alpha_{-}}}$ is a constant, $\frac{1}{\abs{\alpha_+-1}} = k_+$, and $\frac{1}{\abs{\alpha_--1}}=k_-$. Here, $k_{\pm}=1,2,3,\cdots$. Then we have
\begin{equation}\label{simple-eta}
\eta = \frac{1/c_-^{1-\alpha_-}}{1/c_+^{1-\alpha_+}+1/c_-^{1-\alpha_-}} = \frac{k_-}{k_++k_-}.
\end{equation}
The skewed hyperplane equation~\eqref{centerB} is interpreted as that the hyperplane $h^*(x)=0$ is tuned by the ratio of polynomial order of SIGTRON if $\cos(\theta_+)$ does not change much. We notice that \cite{cao19} has added a margin-related parameter, of which the role is  similar to $c_{\alpha}$ of our SIC model, inside of the loss function to adjust decision boundary in imbalanced deep neural network. Before we go further, we introduce a simplified model to show the role of $r_{sc}$ in the skewed hyperplane equation~\eqref{centerB}.
\begin{remark}
Let us assume that $x_i \approx x_p^c > 0$ for $i \in {\cal N}_+$, $x_j \approx x_n^c > 0$ for $j \in {\cal N}_-$, $\lambda=0$, and $p_{\pm} \in (0,1)$. Then \eqref{gradF} becomes
$
r_c p_+(h^*(x_p^c))\left[\begin{array}{c} x_p^c \\ 1 \end{array}\right] \approx p_-(-h^*(x_n^c))  \left[\begin{array}{c} x_n^c \\ 1 \end{array}\right]. 
$
We apply $\inprod{\cdot}{\left[\begin{array}{c} x_p^c \\ 1 \end{array}\right]}$ and $\inprod{\cdot}{\left[\begin{array}{c} x_n^c \\ 1 \end{array}\right]}$. Then we simplify the corresponding equations. As a result, we have $r_{sc} p_+(h^*(x_p^c)) \approx p_-(-h^*(x_n^c))$. After linearization at $0$~\cite{elsayed18}, we obtain \eqref{centerB} with $r_{sc}$, instead of data-sensitive $r_c$. Actually, as observed in Theorem \ref{rsc-standard}, by the dimension-wise standardization of the dataset, we always have $\abs{r_{sc}-1} < \abs{r_c - 1}$ where $r_{sc} \not=1$, irrespective of the domain of the dataset ${\cal D}$. Hence, empirically we could use $r_{sc}$, instead of $r_c$ (see Example \ref{example2}). In addition, refer to \cite{cui19} where they invent effective number to replace the role of $r_c$ in imbalanced classification. 
\end{remark}

The following Example \ref{example2} describes the tunable hyperplane via skewed hyperplane equation for $r_{sc}$-inconsistent dataset having the well-balanced training and imbalanced test datasets.

\begin{example}\label{example2}
Let us start with the two-class \squote{\it spectf} dataset in Table \ref{2classimb}. The training dataset is well-balanced, i.e., $r_{sc}=1$. However, the test dataset has $r_{sc} = 0.26$($r_c=0.09$). It indicates that the positive class of the test dataset is the minority class. The hyperplane to be learned should be located near the minority class to achieve better test classification accuracy. As observed in \eqref{simplemargin} and Figure \ref{fig:example1} (d), to move the hyperplane to the minority class as close as we can, we need to select the smallest $\eta = 1/11$. This $\eta$ corresponds to four $(\alpha_+,\alpha_-)$ candidates: $(11/10,2)$,$(9/10,2)$,$(11/10,0)$, and $(9/10,0)$. In fact, at $(\alpha_+,\alpha_-)=(11/10,2)$, we obtain the minimum distance of $x_{test,p}^c$ (the centroid of the positive class of test dataset) to the hyperplane $h_{(\alpha_+=11/10,\alpha_-=2)}^*(x)=0$ (Figure \ref{fig:example1} (b)) and the best test classification accuracy $64.6\%$ (Figure \ref{fig:example1} (a)). Note that the pattern of $\eta$ in Figure \ref{fig:example1} (d) is similar to the pattern of the distance of $x_{test,p}^c$ to the hyperplane in Figure \ref{fig:example1} (b). As Figure \ref{fig:example1} (a) shows, the region $\alpha_- \approx 2$ obtains better test classification accuracy than the region $\alpha_- \approx 0$. Additionally, note that  $cos(\theta_{test,+}) = \inprod{\frac{w_{(\alpha_+,\alpha_-)}^*}{\norm{w_{(\alpha_+,\alpha_-)}^*}}}{\frac{x_{test,p}^c-x_{test,n}^c}{\norm{x_{test,p}^c-x_{test,n}^c}}} \in [0.52,0.90]$ and $\E(cos(\theta_{test,+}))=0.71$. As a reference, we obtained $20\times20$ hyperplanes $h^*_{(\alpha_+,\alpha_-)}(x)=0$ by solving $20\times20$ SIC models~\eqref{linmin} with the well-balanced training dataset. The cross-validation was used for the best regularization parameter $\lambda$. We set $\abs{(c_+)_{\alpha_+}}=\abs{(c_-)_{\alpha_-}}=2$, $\frac{1}{\abs{1-\alpha_{+}}} = k_+ = 1,2,\cdots,10$, and $\frac{1}{\abs{1-\alpha_{-}}} = k_- = 1,2,\cdots,10$. 
\end{example}

In Example \ref{example2}, we saw that the hyperplane should be located near minority class for better performance. Depending on the pattern of the distribution of the data in the minority class, e.g. the variance of the minority class, it may not be true. For instance, \cite{cao19} recommend that the decision boundary be required to be placed near majority class. For better understanding of the pattern of the data distribution, we recommend to use $(\alpha_+,\alpha_-)$-matrix. 

 \begin{figure*}[t]
\centering
\includegraphics[width=7in]{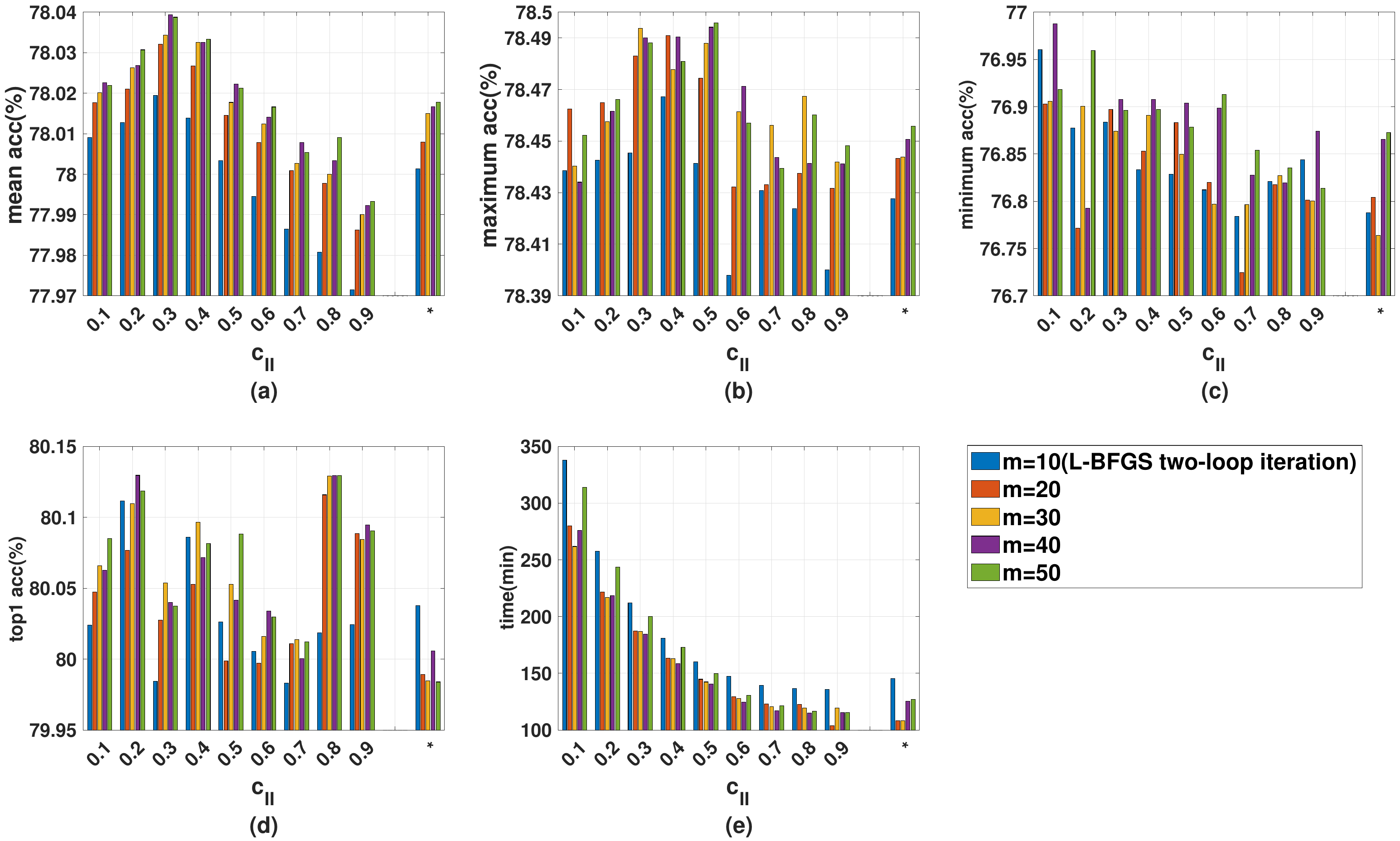}
\caption{A comparison of performanve between {\it quasi-Newton(L-BFGS) for virtual convex loss}, which uses the strong Wolfe condition~\eqref{strongwolfeD} with $c_{II} \in [0.1,0.9]$, and the classic L-BFGS(*), which uses the strong Wolfe condition~\eqref{strongwolfeD} with $c_{II}=0.9$ and the Armijo condition~\eqref{armijoC} with $c_{I}=10^{-4}$. Note that L-BFGS(*) uses the cubic-interpolation-based line search~\cite{nocedal06,schmidt05}. For our experiments, we use $12\times12$ SIC models with $\abs{c_{\alpha}}=1$ and $\abs{\alpha_{\pm}-1}=\frac{1}{k}$, where $k=1,2,3,4,5,6$. (a) Mean test accuracy of $12 \times 12$ SIC models. (b) Maximum test accuracy and (c) minimum test accuracy obtained by an SIC model with fixed $\alpha_{\pm}$ for each $c_{II}$. (d) Test accuracy of TOP$1$ for each $c_{II}$. (e) Total computation time of $12 \times 12$ SIC models. Here, we report the average values of five times repeated experiments with all datasets in Table \ref{2classimb} and \ref{mclassimb}. For $0.1 \le c_{II} \le 0.5$, in terms of mean test classification accuracy in (a), {\it quasi-Newton(L-BFGS) for virtual convex loss} outperforms L-BFGS(*), for each $m$ of two loop iterations.} 
\label{fig:losslessperformance}
\end{figure*}
 
\begin{remark}\label{focalremark}
Lately, \cite{lin20} has proposed two focal loss functions for imbalanced object detection. The first one is the non-convex focal loss function. It has $L_+(h(x)) = -\pi(1-p(h(x)))^{\gamma_g}\log(p(h(x)))$ and $L_-(h(x)) = -(1-\pi)p(h(x))^{\gamma_g}\log(1-p(h(x)))$ where $p(h(x)) \in (0,1)$ is a probability function, like canonical sigmoid $\sigma$ or reduced Sigtron. The second one is the convex focal loss function. It has $L_+(h(x))= -\pi\log(\sigma(\gamma h(x)+\xi))$ and $L_-(h(x)) =- (1-\pi)\log(1-\sigma(\gamma h(x) + \xi))$. Here $\pi \in (0,1)$ is known as a cost-sensitive parameter to be selected depending on $r_{c}$ (or $r_sc$) of the training dataset. Note that $\gamma \ge 1$ and $\xi \ge 0$ control the stiffness and shift of the convex focal loss, respectively. As \cite{lin20} mentioned, the performance gap between the two types of focal losses is negligible. Therefore, we exclusively compare the convex focal loss to the convex SIC model. To find additional information, please refer to Section \ref{Sec5A}.
\end{remark}

\section{Quasi-Newton optimization(L-BFGS) for virtual convex loss\label{SIGTRON-line-search}}
This Section presents {\it quasi-Newton optimization(L-BFGS) for virtual convex loss} framework. It includes the proposed interval-based bisection line search, which uses gradients of a virtual convex loss function. 

Let us discuss the SIC model~\eqref{linmin}, where ${\cal F}(h)$ is convex, differentiable, and bounded below. It is worth noting that the optimization framework we will be proposing for this model can also be used for cost-sensitive learning model~\eqref{costsensitive}, including the $\pi$-weighted convex focal loss. Before we proceed, let us take a moment to review the quasi-Newton optimization framework described in \cite{nocedal06}. The iterates $h_0,h_1,h_2,\cdots$ satisfy $h_{t+1} = h_t + \rho_tz_t$ where $\rho_t>0$ is a step length and $z_t = -B_t^{-1}\nabla {\cal F}(h_t)$ is a descent direction. Here, $B_t$ is a symmetric and positive definite rank-two approximation of the Hessian matrix $\nabla^2 {\cal F}(h_t)$. Interestingly, L-BFGS directly approximates $B_t^{-1}\nabla {\cal F}(h_t)$ by two-loop iterations with $m$ recent elements. Here, $m$ is the tuning parameter of L-BFGS. The performance comparison of the proposed optimization framework considering $m$ of L-BFGS is shown in Figure \ref{fig:losslessperformance}. For the initial point, we set $h_0=0$, corresponding to the probability-half point of SIGTRON in the gradient of the SIC model. It is well known that, to guarantee sufficient descent of ${\cal F}(h)$ and positive definiteness of low-rank matrix $B_t$, the step length $\rho_t$ of L-BFGS should satisfy the Armijo condition~\eqref{armijoC} and the Wolfe condition~\eqref{wolfeC}: 
\begin{equation}\label{armijoC}
{\cal F}(h_t + \rho_t z_t) - {\cal F}(h_t) \le c_{I} \rho_t \inprod{\nabla {\cal F}(h_t)}{z_t}
\end{equation}
and
\begin{equation}\label{wolfeC}
\inprod{\nabla {\cal F}(h_t + \rho_tz_t)}{z_t} \ge c_{II} \inprod{\nabla {\cal F}(h_t)}{ z_t}, 
\end{equation}
where $0<c_I < c_{II}<1$.
The Armijo condition~\eqref{armijoC} can be reformulated through the expectation of gradients: 
\begin{equation}\label{lineexpectation}
\phi(\rho_t) - \phi(0) =\int_0^{\rho_t} \phi'(\rho) d\rho = \rho_t\E_{[0,\rho_t]}(\phi'),
\end{equation}
where $\phi(\rho_t) = {\cal F}(h_t + \rho_t z_t)$ and $\phi'(\rho) = \inprod{\nabla {\cal F}(h_t + \rho z_t)}{ z_t }$. Note that $\phi'(0) = \inprod{\nabla {\cal F}(h_t)}{z_t} < 0$ where $z_t = - B_t^{-1}\nabla {\cal F}(h_t)$. Now, we get the reformulated Armijo condition
\begin{equation}\label{armijoD}
\E_{[0,\rho_t]}(\phi')  \le c_{I}\phi'(0)
\end{equation}
and the Wolfe condition
\begin{equation}\label{wolfeD}
c_{II} \phi'(0) \le \phi'(\rho_t),
\end{equation}
where \eqref{wolfeD} is also known as the curvature condition~\cite{more94}, which is clearly understood by the reformulation of \eqref{wolfeD} as $\E_{[0,\rho_t]}(\phi'') > (c_{II}-1)\rho_t\phi'(0) > 0$. The positive definiteness of $B_t$ in L-BFGS is adjusted by $c_{II} \in (0,1)$, normally set as $0.9$. For more details, see~\cite{nocedal06}. 

The reformulated Armijo condition~\eqref{armijoD} has several advantages, compared to the Armijo condition~\eqref{armijoC}. First, it is more intuitive about the descent condition of the loss function. The average slopes of $\phi$ in the interval $[0,\rho_t]$ must be less than the initial slope $\phi'(0)$. Second, for the SIC model \eqref{linmin}, using an approximation of \eqref{armijoD} is more practical. That is, $\E_{[0,\rho_t]}(\phi') \approx \sum_{i=1}^n a_i \phi'(\tilde{\rho}_i)$, where  $a_i \ge 0$, $\sum_{i=1}^n a_i = 1$, and $0 \le \tilde{\rho}_0 < \tilde{\rho}_1 < \cdots < \tilde{\rho}_n \le \rho_t$. This approach is workable for the general loss function, including virtual non-convex loss function. For a virtual convex loss function, however, we do not need to evaluate a relatively large number of directional derivatives in the interval $[0,\rho_t]$. Instead of \eqref{armijoD} and \eqref{wolfeD}, we can use the strong Wolf condition, i.e., (relative) strong Wolfe stopping criterion. 
\begin{equation}\label{strongwolfeD}
\abs{\phi'(\rho_t)} \le -c_{II} \phi'(0) 
\end{equation}
where $c_{II} \in (0,1)$ is a tuning parameter of the proposed {\it quasi-Newton(L-BFGS) optimization for virtual convex loss}. See also \cite{hager05,hager06} for related line search algorithms utilizing \eqref{strongwolfeD}. In this article, for the strong-Wolfe stopping criterion~\eqref{strongwolfeD}, we create a new interval-based bisection line search(Algorithm \ref{algox}). See \cite{hiriart-urruty96,nocedal06} for the various characteristics of the interval reduction method in general line search. The overall framework of {\it quasi-Newton(L-BFGS) optimization for virtual convex loss} is stated in Algorithm \ref{algoz}, which contains the interval-based bisection line search in Algorithm \ref{algox}. See also Theorem \ref{th:loss} for the convergence of Algorithm \ref{algox}. 

\begin{algorithm}
\caption{Quasi-Newton optimization(L-BFGS) for virtual convex loss}\label{algoz}
\textbf{Input:} $h_0=(w_0,b_0)=0$, $\epsilon_{tol1} = 10^{-2}$, $\epsilon_{tol2}=10^{-4}$, $c_{II}=0.4$, $MaxIter=100$.\\
\textbf{Output:} $h^* = (w^*,b^*)$.\\
\For{ $0 \le r \le MaxIter$}{
\If{$\norm{\nabla {\cal F}(h_t)}_{\infty} \le \epsilon_{tol1}$ or $\norm{h_{t+1}-h_t}_{\infty} \le \epsilon_{tol2}$}{$h^*=h_t$ and STOP}
\textbf{Compute descent direction $z_t$} by Quasi-Newton optimization(L-BFGS)\\ 
$$
z_t = \hbox{L-BFGS}(\nabla{\cal F}(h_t))
$$
\textbf{Compute step-Length $\rho_t$} by Algorithm \ref{algox} with strong-Wolfe stopping criterion \eqref{strongwolfeD}\\
$$
\rho_{r} = \argmin_{\rho \in \R_{\ge 0}} {\cal F}(h_t + \rho z_t)
$$
\textbf{Update} $h_{t+1} = h_t + \rho_tz_t$
}
$h^* = h_{MaxIter+1}$
\end{algorithm}

\begin{theorem}\label{th:loss}
Let $\phi$ be convex, differentiable, and bounded below. Then Algorithm \ref{algox} with an initial condition $\phi'(0)<0$ converges to $\rho^*$ satisfying \eqref{strongwolfeD}, where $c_{II} \in (0,1)$, in finite steps.
\end{theorem}
\begin{proof}
Let us first consider the case that $\phi$ is a coercive function. Since $\phi'(0)<0$ and $\phi'$ is a non-decreasing function, there is $\rho_{opt} >0$ such that $\phi'(\rho) \ge 0$ for all $\rho \ge \rho_{opt}$. As noticed in line $6-7$ and line $10-11$ of Algorithm \ref{algox}, there is $i$th iteration such that $\phi'(\rho_i) < \phi'(\rho_{opt}) = 0 < \phi'(2\rho_i)$. Therefore, the interval, which includes $\rho_{opt}$, is established as $[\rho_L,\rho_U] = [\rho_i,2\rho_i]$. Then by the bisection algorithm in line $6-9$ and line $13$, $[\rho_L,\rho_U]$ is shrinking to $\rho_{opt}$ and the strong-Wolfe stopping criterion in line $4$ of Algorithm \ref{algox} is satisfied within finite steps. Now, we consider the case that $\phi$ is not a coercive function. Since $\phi$ is convex, bounded below, and $\phi'(0)<0$, $\lim_{\rho \rightarrow +\infty} \phi'(\rho) \rightarrow 0$(line $11$). Therefore, it stops by strong-Wolfe stopping criterion in line $4$. 
\end{proof}

\begin{algorithm}
\caption{Interval-based bisection line search with \eqref{strongwolfeD}}\label{algox}
\textbf{Input:} $c_{II}=0.4$, $itermax=100$, $\rho_0$, $\phi'(\rho) = \inprod{\nabla {\cal F}(h_t+\rho z_t)}{z_t}$,  $\phi'(0)<0$, $[\rho_L,\rho_U] = [0,\infty]$\\
\textbf{Output:} $\rho^*$\\
\For{$0 \leq i \leq$ itermax}{
\If{ $\abs{\phi'(\rho_{i})}  \le - c_{II} \phi'(0)$ }{
$\rho^*=\rho_{i}$ and STOP\\
}
\vskip 0.05in
\If{$0 < \phi'(\rho_{i})$}{
        $\rho_U    = \rho_{i}$~~
        }
\ElseIf{$\phi'(\rho_{i}) < 0$}{
        $\rho_L  = \rho_{i}$~~
}
\vskip 0.05in
\If{$\rho_U = \infty$}{

$\rho_{i+1} = 2\rho_{L}$
}
\Else{
$\rho_{i+1} = \frac{1}{2}(\rho_L + \rho_U)$
}
}
$\rho^* = 0$\\
\end{algorithm}

\begin{remark}
Besides Armijo~\eqref{armijoC} and Wolfe~\eqref{wolfeC} criteria for line search, there is an additional criterion known as Goldstein condition~\cite{nocedal06}. By way of \eqref{lineexpectation}, it is reformulated as
\begin{equation}\label{goldsteinC}
(1-c_{III}) \phi'(0) \le\E_{[0,\rho_t]}(\phi') \le c_{III}\phi'(0)
\end{equation}
where $c_{III} \in (0,1/2)$ and $\phi'(0)<0$. Unfortunately, this condition does not always include the solution of $\min_{\rho} \phi(\rho)$. To plug it into the quasi-Newton(L-BFGS) optimization for virtual loss, We need an additional curvature condition~\eqref{wolfeD}. 
\end{remark}

 \begin{figure}[t]
\centering
\includegraphics[width=2.3in]{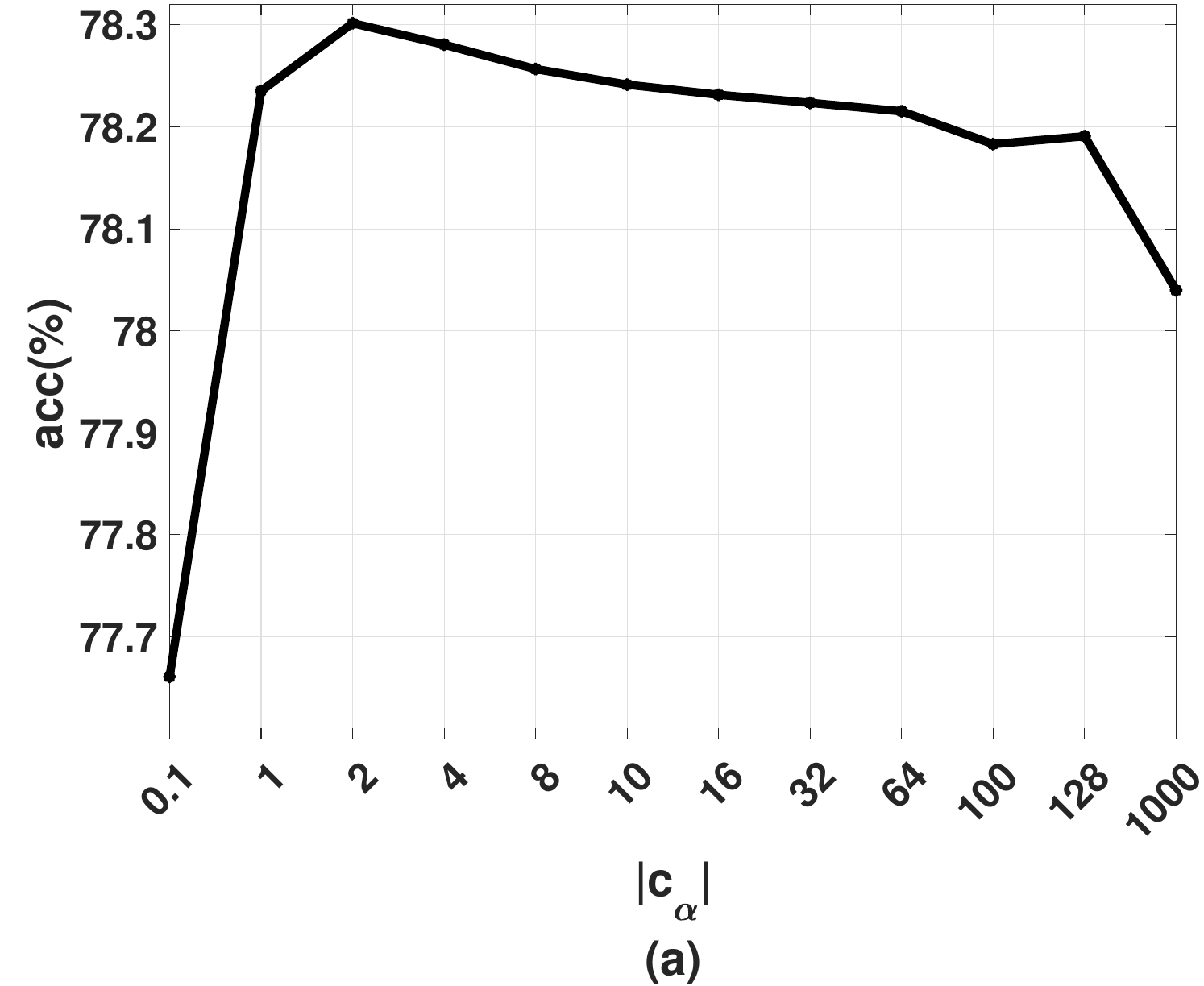} \hskip 0.1in
\includegraphics[width=2.4in]{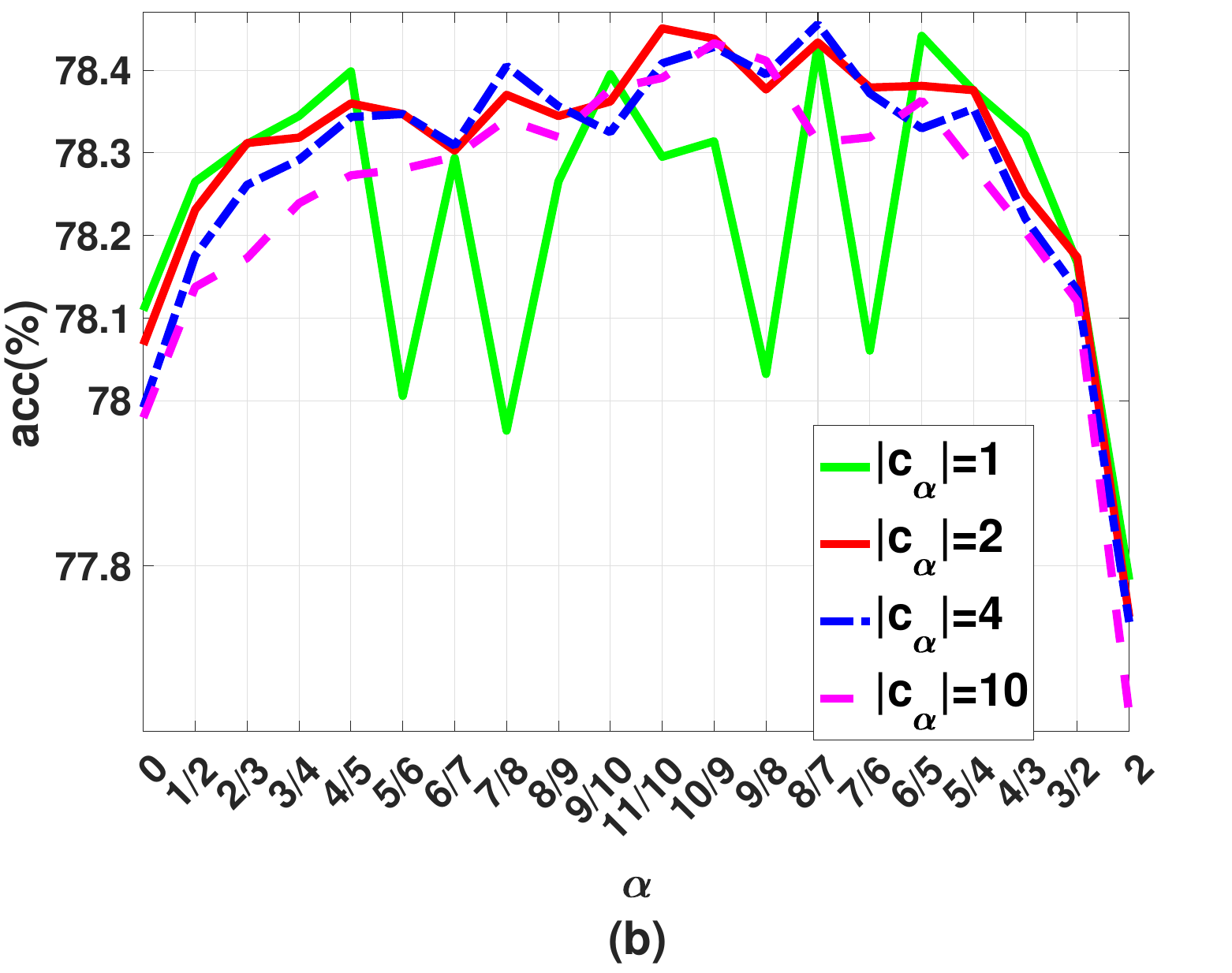}
\caption{Graphs of classification performance of the SIC model~\eqref{linmin} with $\alpha=\alpha_-=\alpha_+ \in [0,2]$. (a) Test accuracy($\%$) vs. $\abs{c_{\alpha}}$. The test accuracy is the average of all results obtained with $\alpha$ in \eqref{setofalpha}. When $\abs{c_{\alpha}} = 2$, the best performance is achieved. (b) Test accuracy($\%$) vs. $\alpha$. When $\frac{1}{5} \le \abs{\alpha-1} \le 1$, the best performance is achieved with $\abs{c_{\alpha}} = 1$. However, when $\abs{\alpha-1}< \frac{1}{5}$, the SIC model with $\abs{c_{\alpha}} \ge 2$ shows better performance.}
\label{fig:alphavsc}
\end{figure}

\section{Numerical experiments with the $20 \times 20$ SIC models\label{secexperiment}}
This Section reports the classification results acquired by the SIC model~\eqref{linmin} and {\it quasi-Newton(L-BFGS) for virtual convex loss}(Algorithm \ref{algoz} and \ref{algox}). We compare the proposed methodology with well-known classifiers: $\pi$-weighted convex focal loss~\cite{lin20}, LIBLINEAR(logistic regression, SVM, and L2SVM)~\cite{fan08,galli22}, and LIBSVM(C-SVC with RBF kernel)~\cite{chang11}. Note that {\it Quasi-Newton(L-BFGS) for virtual convex loss} is mainly implemented in {\sc Matlab}(version R2023b) based on \cite{schmidt05}. This optimization algorithm is used for the SIC model~\eqref{linmin} and the $\pi$-weighted convex focal loss~\cite{lin20}. LIBLINEAR(version 2.4.5)~\cite{galli22} and LIBSVM(version 3.3.2)~\cite{chang11} are mainly implemented in {\sc C/C++} language with {\sc Matlab} interface. All runs are performed on APPLE M2 Ultra with a 24-core CPU and 192GB memory. The operating system is MacOS Sonoma(version 14.1). We use {\it parfor} in {\sc Matlab} for parallel processing of all models, including LIBLINEAR and LIBSVM, in a 24-core CPU. In terms of multi-class datasets, the OVA(one-vs-all) strategy is used for all linear classification models. The OVO(one-vs-one) strategy is used for the kernel-based classification model LIBSVM~\cite{chang11}. 

Concerning {\it quasi-Newton(L-BFGS) for virtual convex loss}, as observed in Figure \ref{fig:losslessperformance}, it is recommended to select $m \in [20, 50]$ for two-loop iterations of L-BFGS and $c_{II} \in [0.1,0.5]$ for the interval-based bisection line search(Algorithm \ref{algox}).  We choose $m=40$ and $c_{II}=0.4$ considering performance-computation complexity. For stopping criterions of {\it quasi-Newton(L-BFGS) for virtual convex loss}, we use $\norm{\nabla {\cal F}(h_t)}_{\infty} \le \epsilon_{tol1}$ and $\norm{h_{t+1}-h_t}_{\infty} \le \epsilon_{tol2}$ where $\epsilon_{tol1}=10^{-2}$ and $\epsilon_{tol2}=10^{-4}$(Algorithm \ref{algoz}). We could select a smaller $c_{II}$ for exact line search, used in other quasi-Newton optimization, such as nonlinear conjugate gradient~\cite{hager05}. 

In order to use the $\pi$-weighted convex focal loss~\cite{lin20} discussed in Remark \ref{focalremark}, we need to set three parameters: $\gamma$, $\xi$, and $\pi$. Following the recommendations in \cite{lin20}, we choose $\gamma=1, 2,3,4$ and $\xi=0,1$. As for $\pi$, we select $19$ regular points ranging from $0.05$ to $0.95$. This gives us $152$ convex focal losses, expressed as a $(\pi, \gamma:\xi)$-matrix.

We have selected LIBLINEAR~\cite{fan08} and LIBSVM~\cite{chang11} as our standard for balanced linear classification models and non-linear classification models, respectively. For logistic regression, we use logistic loss, hinge-loss for SVM, and squared hinge-loss for L2SVM. To learn an inhomogeneous hyperplane, we set $B=1$. We use the primal formulation ($s=0$) for logistic regression and the dual formulation ($s=3$) for SVM. As for L2SVM, we use the primal formulation ($s=2$). In LIBSVM~\cite{chang11}, we use C-SVC(support vector classification) ($s=0$) with the RBF kernel $K(x_i,x_j) = \exp(-\nu \norm{x_i-x_j}^2)$ ($t=2$).

All models have an $\ell_2$-regularizer $\frac{\lambda}{2}\norm{w}^2$. In terms of regularization parameter $\lambda$ for the cost-sensitive learning framework~\eqref{costsensitive}, including $20\times20$ SIC models~\eqref{linmin} and $19\times8$ $\pi$-weighted convex focal loss models in Remark \ref{focalremark}, we use CV(cross-validation) with candidates in \eqref{lambdaD} as recommended in LIBSVM~\cite{chang11}.
\begin{equation}\label{lambdaD}
\lambda = 2^{r}, \; r = -14,-13,-12,\cdots,5
\end{equation} 
In LIBLINEAR and LIBSVM, the regularization parameter $\lambda$ is located on the loss function. Therefore, we use $C = \lambda^{-1}$ with \eqref{lambdaD} for CV. For LIBSVM, in addition to the regularization parameter $C$ on the loss function, the RBF kernel parameter $\nu$ is cross-validated with candidates $\nu = 2^r$ and $r=-14,..,5$.

Regarding benchmark datasets~\cite{delgado14}, they are pre-processed and normalized in each feature dimension with mean zero and variance one~\cite{ioffe15}, except for when the variance of the raw data is zero. This process reduces the effect of scale imbalance of datasets. The scale-class-imbalance ratio $r_{sc}$~\eqref{imbrsc} of two-class and multi-class datasets is presented in Table \ref{2classimb} and Table \ref{mclassimb}, respectively. In the case of two-class datasets in Table \ref{2classimb}, the mean value of $r_{sc}$ of training dataset is $\E(r_{sc}{\bf T}) \approx 1.61$. Also, we have $\min r_{sc}{\bf T} = 0.49$ and $\max r_{sc}{\bf T} =7$. Thus, the two-class datasets used in our experiments are roughly well-balanced. However, most of the two-class datasets have variations between $r_{sc}{\bf T}$ and $r_{sc}$ of test dataset ($r_{sc}{\bf Te}$). The raw format of each benchmark dataset is available in the UCI machine learning repository~\cite{ucidata}. As commented in~\cite{wainberg16}, we reorganize datasets in \cite{delgado14}. Each dataset is separated into the non-overlapped training and test datasets. The training dataset of each dataset is randomly shuffled for $4$-fold CV~\cite{chang11}. Table~\ref{2classimb}($51$ two-class datasets) and Table~\ref{mclassimb}($67$ multi-class datasets) include all information of datasets such as number of instances, size of training dataset, size of test dataset, feature dimension, number of classes, class-imbalance ratio $r_c$ for combined/training/test dataset, and scale-class-imbalance ratio $r_{sc}$ for combined/training/test dataset. The experiments are conducted five times using randomly selected CV datasets, with a fixed initial condition of $(w_0,b_0) = (0,0)$. For $\alpha$ and $c_{\alpha}$ of SIC model~\eqref{linmin}, we conducted a preliminary experiment with the reduced class of SIC model ($\alpha=\alpha_+=\alpha_-$). We found that the best test classification accuracy is obtainable when $\abs{c_{\alpha}}=2$. For general purposes, $\abs{c_{\alpha}} \in [1,10]$ is a possible choice. When $\alpha$ is not close to $1$, the SIC model with $\abs{c_{\alpha}}=1$ shows the best performance. The detailed information is provided in Figure \ref{fig:alphavsc}. For the experiments in this Section, we set $c_{\alpha} = 2$ for $\alpha>1$ and $c_{\alpha}=-2$ for $\alpha<1$. Thus, $\alpha_{\pm}$ are the only tuning parameters for which we use the following $20$ different values in $[0,2]$:
\begin{equation}\label{setofalpha}
\alpha_{\pm} \in \left\{\;\frac{0}{1},\frac{1}{2},\frac{2}{3},\cdots,\frac{9}{10},\;\frac{11}{10},\frac{10}{9},\frac{9}{8},\cdots,\frac{3}{2},\frac{2}{1} \;\right\}
\end{equation}
This gives us $20 \times 20$ SIC models. The characteristic of each dataset could be captured by the large class of hyperplanes $h^*_{(\alpha_+,\alpha_-)}=(w^*_{(\alpha_+,\alpha_-)},b^*_{(\alpha_+,\alpha_-)})$ learned via the $20 \times 20$ SIC models~\eqref{linmin}, as noticed in Theorem \ref{thseparable} and Example \ref{example2}. The details are as follows.

\begin{table*}
\centering
\begin{tiny}\begin{tabular}{l||c|c|c|c||c|c|c|c||c|c|c||c}
\hline
\multicolumn{1}{l||}{{\bf MODEL}}&\multicolumn{4}{c||}{{\bf SIC model}\eqref{linmin}}&\multicolumn{4}{c||}{\textbf{Convex Focal Loss}\cite{lin20}}&\multicolumn{3}{c||}{\textbf{LIBLINEAR}\cite{fan08}}&\multicolumn{1}{c}{\textbf{LIBSVM}\cite{chang11}}\\\hline
\textbf{SubModel}&\textbf{TOP$1$}&\textbf{MaxA}&\textbf{Max$2$}&\textbf{MaxM}&\textbf{TOP$1$-FL}&\textbf{MaxA-FL}&\textbf{Max$2$-FL}&\textbf{MaxM-FL}&\textbf{Logistic}&\textbf{SVM}&\textbf{L2SVM}&\textbf{C-SVC}\\
$(\alpha_+,\alpha_-)$ or $(\pi,\gamma,\xi)$&-&$(\frac{7}{8},\frac{8}{7})$&$(\frac{3}{4},0)$&$(\frac{8}{9},\frac{11}{10})$&-&$(0.5,2,1)$&$(0.5,3,0)$&$(0.6,2,1)$&(primal)&(dual)&(primal)&\textbf{RBF} Kernel\\\hline\hline
\textbf{Mean acc($\%$)} of Two Class&{\bf 83.96}&82.49&{\it 82.51}&82.36&83.80&82.14&82.37&81.39&82.11&81.59&82.06&83.22\\
\textbf{Mean acc($\%$)} of Multi Class&77.30&{\it 75.57}&75.19&{\it 75.57}&76.68&75.53&75.36&75.55&74.75&72.94&74.18&{\bf 79.96}\\
\textbf{Mean acc($\%$)} of All Class&80.18&{\it 78.56}&78.35&78.50&79.76&78.39&78.39&78.07&77.93&76.68&77.58&{\bf 81.37}\\\hline
\textbf{Time} of all class&423m&106s&122s&98s&81m&88s&88s&88s&60s&109s&{\bf 57s}&2077m\\\hline
\end{tabular}
\caption{A comparison of the SIC model~\eqref{linmin} with benchmark models: $\pi$-weighted convex focal loss~\cite{lin20}, LIBLINEAR~\cite{fan08}, and LIBSVM~\cite{chang11}. In this comparison, TOP$1$ refers to a group of SIC models with the best test accuracy for each dataset, while MaxA/Max$2$/MaxM is an SIC model with the best test accuracy for all-, two-, and multi-class. The same notations are used for $\pi$-weighted convex focal loss: TOP$1$-FL, MaxA-FL, Max$2$-FL, and MaxM-FL. In the two-class problems, TOP$1$ performs better than kernel-based LIBSVM and TOP$1$-FL. Max$2$ of the SIC model offers the best test classification accuracy among various linear classifiers. In the multi-class problems, although the SIC model's TOP$1$ accuracy is less than kernel-based LIBSVM, it still achieves $77.30\%$ accuracy, which is $0.62\%$ better than TOP$1$-FL. When the models' parameters are fixed, the SIC model performs similarly to the convex focal loss having the external $\pi$-weight parameter.}
\label{table:allperformance}
\end{tiny}
\end{table*}

\subsection{Performance evaluation of $20 \times 20$ SIC models\label{Sec5A}}
Table \ref{table:allperformance} summarizes the classification accuracy ($\%$) and computation time of all experiments conducted on $118$ datasets. The acronym TOP$1$ refers to a group of SIC models that have the highest test accuracy for each dataset, while MaxA/Max$2$/MaxM refers to an SIC model with the best test accuracy for all-, two-, and multi-class datasets. The same notations are used for $\pi$-weighted convex focal loss: TOP$1$-FL, MaxA-FL, Max$2$-FL, and MaxM-FL. The test classification accuracy of each dataset is reported in Table \ref{table:TwoclassTOP5} for two class datasets and in Table \ref{table:MxclassTOP5} for multi-class datasets. Note that MaxA$(\alpha_+=\frac{7}{8},\alpha_-=\frac{8}{7})$ achieves $78.56\%$. On the other hand, MaxA-FL$(\pi=0.5,\gamma=2,\xi=1)$ obtains $78.39\%$. Over half of all SIC models obtain at least $78.20\%$ accuracy. Out of all the $\pi$-weighted convex focal losses, only $10\%$ can achieve the same level of accuracy as the proposed SIC model. This implies that the SIC model is less sensitive to the parameter than the $\pi$-weighted convex focal loss. Therefore, the SIC model could serve as an alternative cost-sensitive learning framework without external $\pi$-weight. Refer to Figure \ref{fig:ALLclass} for additional information. The details are as follows.

In the case of two-class, of which the training dataset is close to the well-balanced condition, TOP$1$ achieves the best results, i.e., $0.74\%$ better than the kernel-based classifier LIBSVM(C-SVC with RBF kernel) and $0.16\%$ better than TOP$1$-FL. When the parameters of the SIC model are fixed, its performance is still better than other linear classifiers, such as $\pi$-weighted convex focal loss and LIBLINEAR. For instance, Max$2$$\left(\alpha_+=\frac{3}{4},\alpha_-=0\right)$ has $82.51\%$ accuracy, which is $0.14\%$ better than Max$2$-FL and $0.4\%$ better than logistic regression, the best model of LIBLINEAR. As shown in Figure \ref{fig:X} (a), the test accuracy of all SIC models is in the range of $[80.64\%, 82.51\%]$. More than $35\%$ of all SIC models achieve at least $82.20\%$ test accuracy. On the other hand, the test accuracy of all convex focal losses is in the range of $[65.93\%,82.37\%]$. Out of all the convex focal loss, only $2\%$ can achieve $82.20\%$ test accuracy. It appears that the SIC models are quite resilient to internal parameter changes. Specifically, Figure \ref{fig:TWOandMULTI} (a) shows an ${\bm X}$-shaped pattern. This pattern covers a much larger area compared to the best test accuracy area of convex focal losses in Figure \ref{fig:TWOandMULTI} (c). The ${\bm X}$-shaped pattern relates to the pattern of $\eta$ in Figure \ref{fig:example1} (d). It represents a small deviation from the balanced SIC model, which has $\abs{\alpha_+-1}=\abs{\alpha_--1}$. Essentially, the virtual SIGTRON-induced loss functions $L_{\alpha_+,c_+}^S$ and $L_{\alpha_-,c_-}^S$ of the SIC model have similar polynomial orders, i.e., $k_+ \approx k_-$. Figure \ref{fig:pattern} (b) {\it horse-colic} demonstrates the ${\bm X}$-shaped pattern. 

It is important to note that {\it spectf} dataset in Table \ref{2classimb} is a typical $r_{sc}$-inconsistent dataset. By using this dataset, the connection between $\eta = \frac{k_-}{k_-+k_+}$ and the movement of the hyperplane $h^*_{\alpha_+,\alpha_-}(x)=0$ is empirically demonstrated in Figure \ref{fig:example1}. Notably, the best test accuracy of the dataset is observed in the region $(\alpha_+,\alpha_-)=(-,2)$, which is outside the ${\bm X}$-shaped pattern.

 \begin{figure*}[t]
\centering
\includegraphics[width=7in]{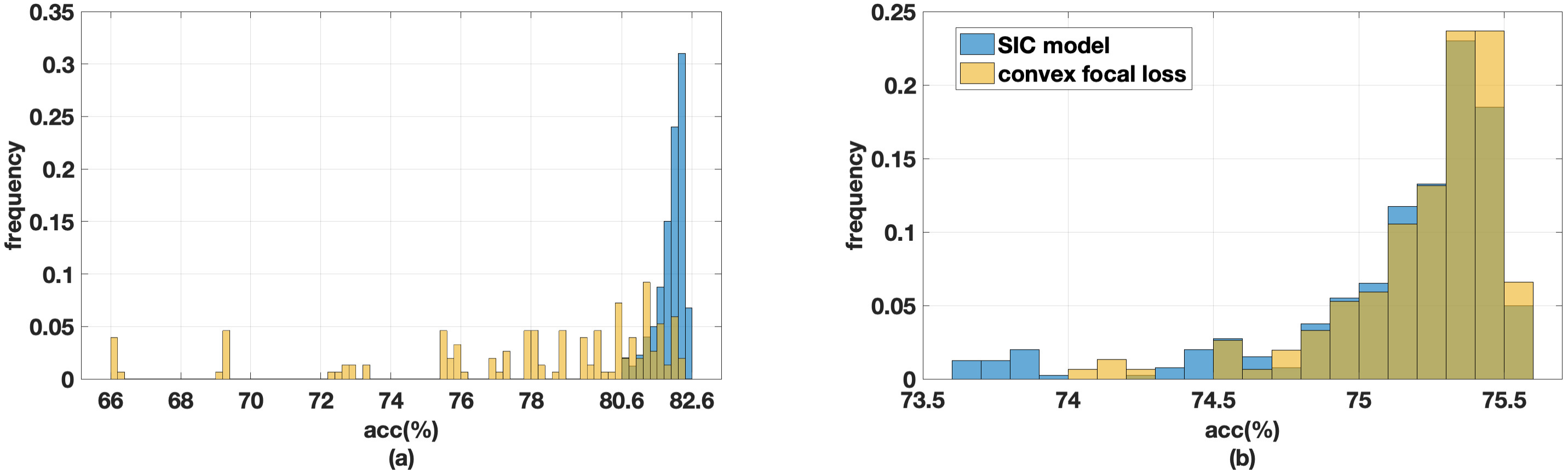}
\caption{The statistical distributions of $20\times20$ SIC models and $19\times 8$ convex focal losses with respect to test classification accuracy. (a) The two-class histogram. The SIC models are in the $[80.64\%, 82.51\%]$ range, and convex focal losses are in the $[65.93\%,82.37\%]$ range. The SIC model is less sensitive to internal parameters $\alpha_{\pm}$ and outperforms the $\pi$-weighted convex focal losses. (b) The multi-class histogram. The SIC models are in the $[73.64\%, 75.57\%]$ range, and convex focal losses are in the $[74.08\%,75.55\%]$ range. Although the SIC models have no external $\pi$-weight parameters, they show comparable performance to $\pi$-weighted convex focal loss. For more information on the matrix pattern of the models, refer to Figure \ref{fig:TWOandMULTI}. }
\label{fig:X}
\end{figure*}

 \begin{figure*}[t]
\centering
\includegraphics[width=3.4in]{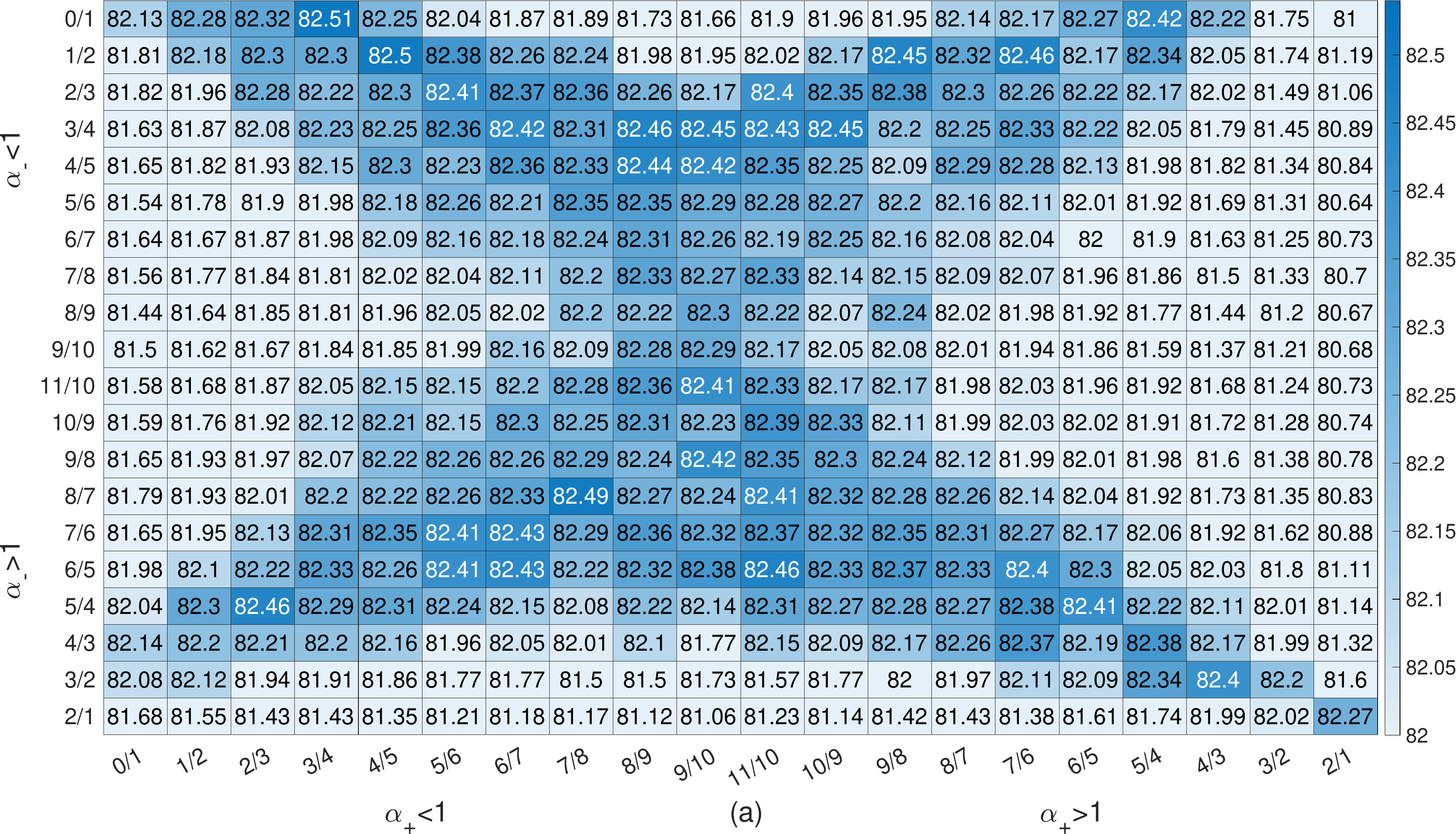} \hskip 0.5cm
\includegraphics[width=3.4in]{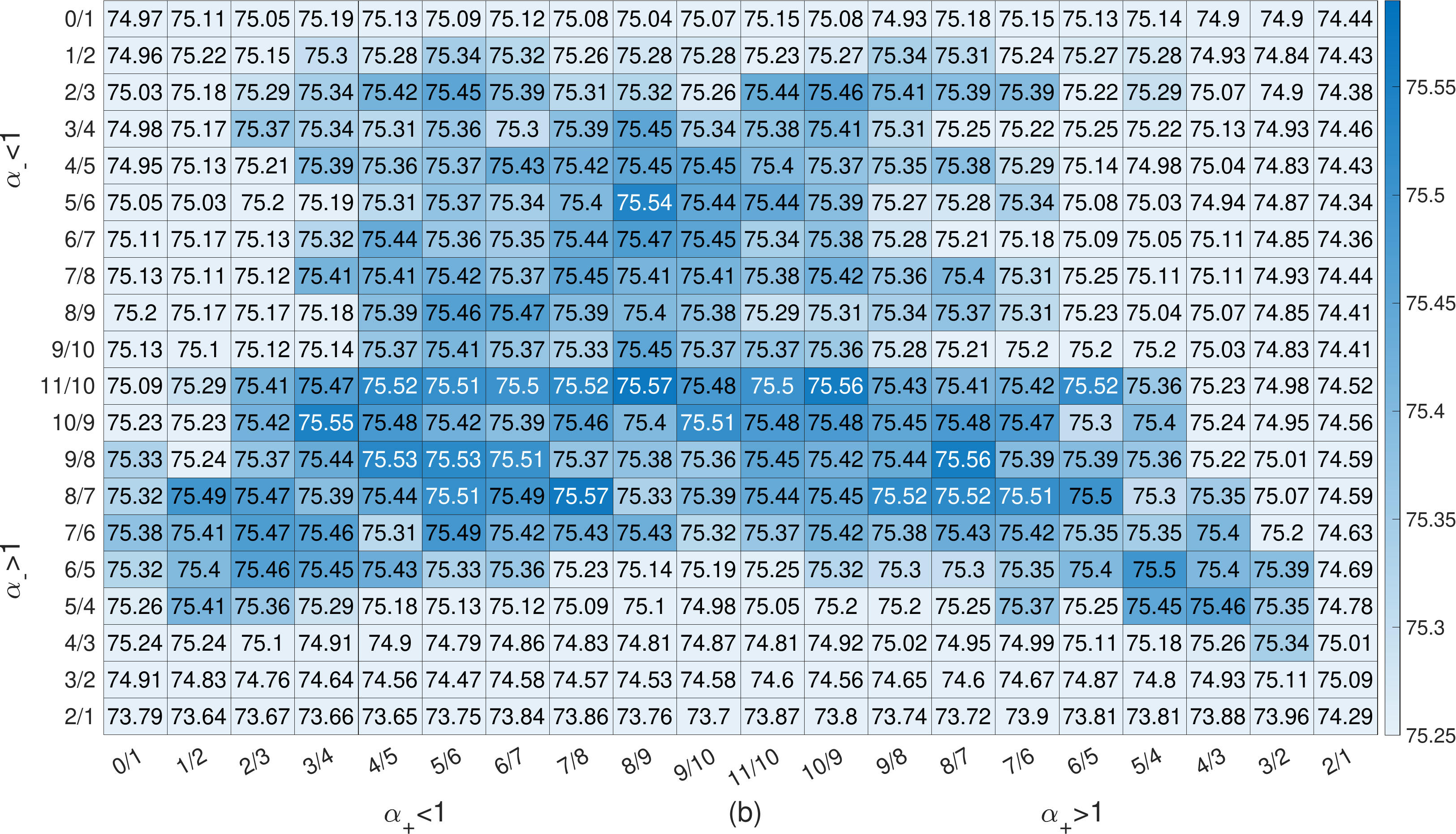} \vskip 0.5cm
\includegraphics[width=3.4in]{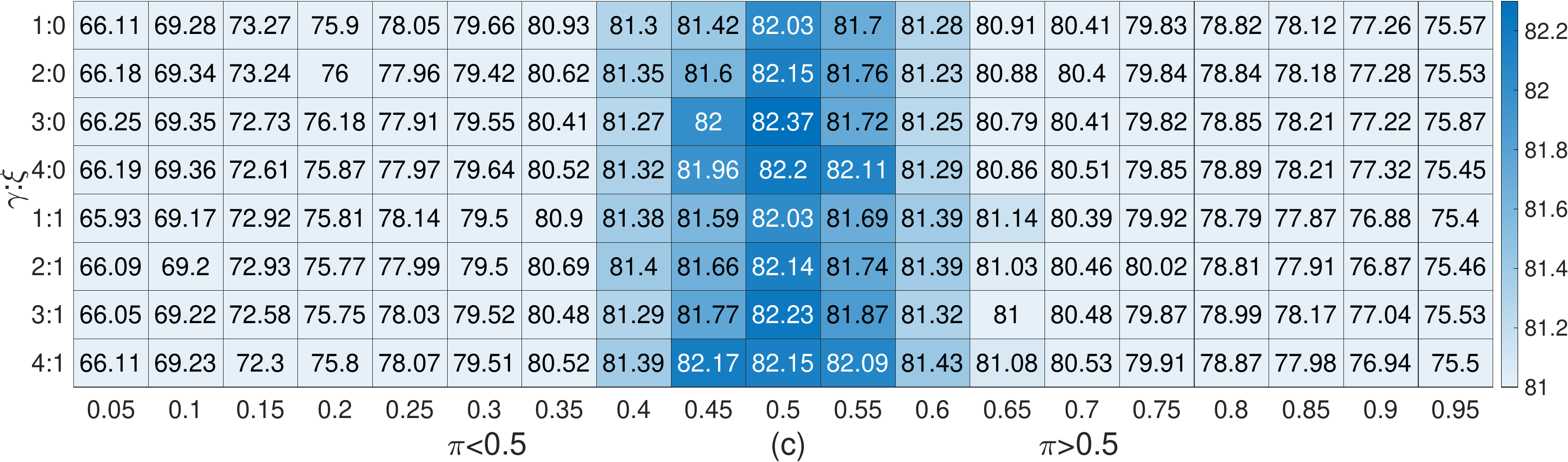} \hskip 0.5cm
\includegraphics[width=3.4in]{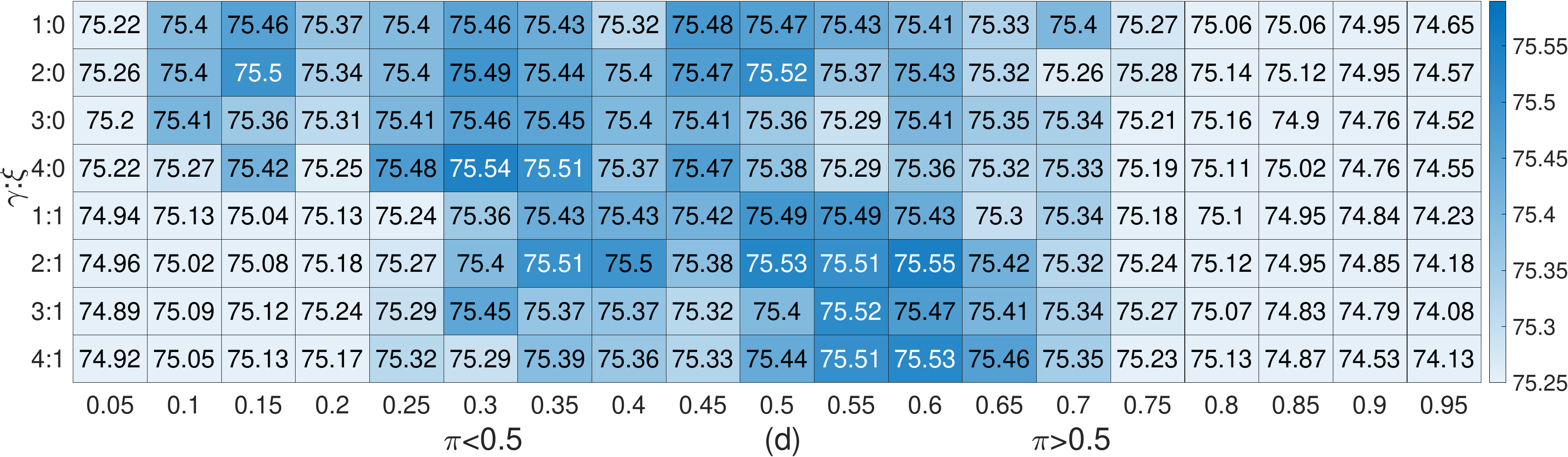}
\caption{Test accuracy(\%) patterns of $20\times20$ SIC models for (a) two-class and (b) multi-class. Test accuracy(\%) patterns of $19\times8$ convex focal losses for (c) two-class and (d) multi-class. As shown in (a), SIC models exhibit an ${\bm X}$-shaped pattern, i.e., $\abs{\alpha_--1} \approx \abs{\alpha_+-1}$. This indicates that the best-performing SIC models have similar polynomial order between $L_{\alpha_+,c_+}^S$ and $L_{\alpha_-,c_-}^S$. This ${\bm X}$-shaped pattern covers a much larger region than the $\pi$-weighted convex focal losses in (c). For more information and statistical distribution of the models, refer to Figure \ref{fig:X}.}
\label{fig:TWOandMULTI}
\end{figure*}

 \begin{figure*}[t]
\centering
\includegraphics[width=3.4in]{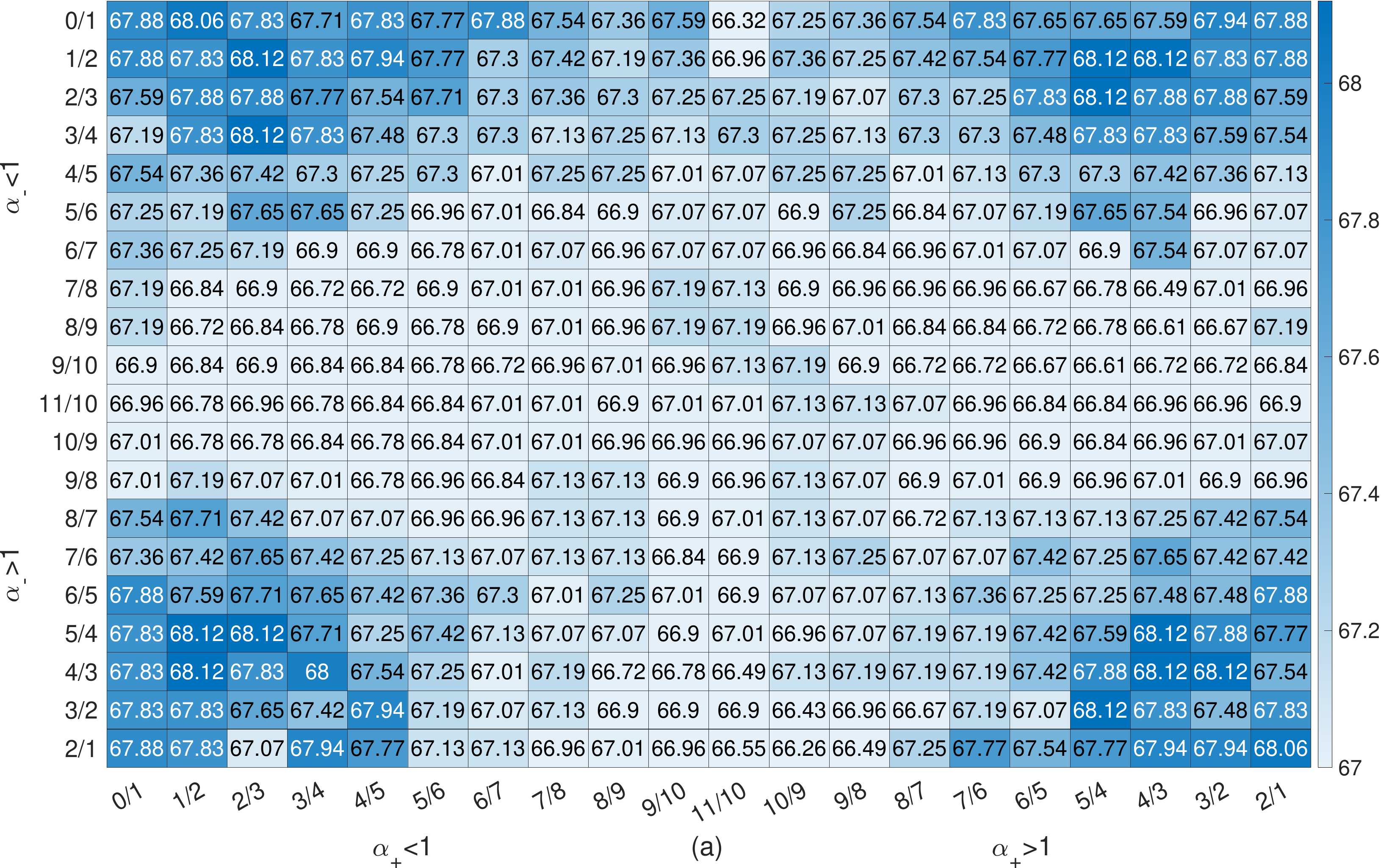} \hskip 0.07in
\includegraphics[width=3.4in]{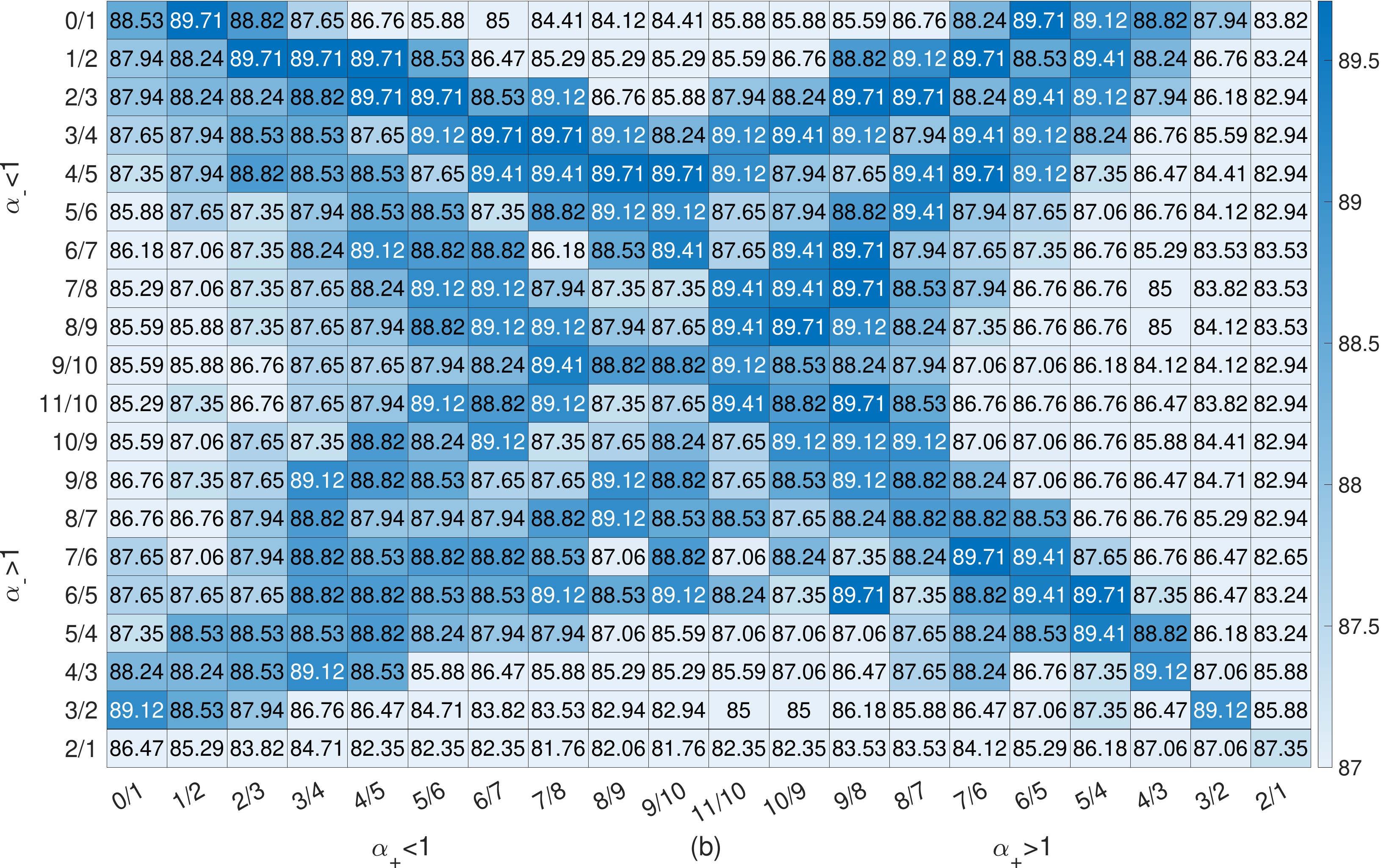}
\vskip 0.05in
\includegraphics[width=3.4in]{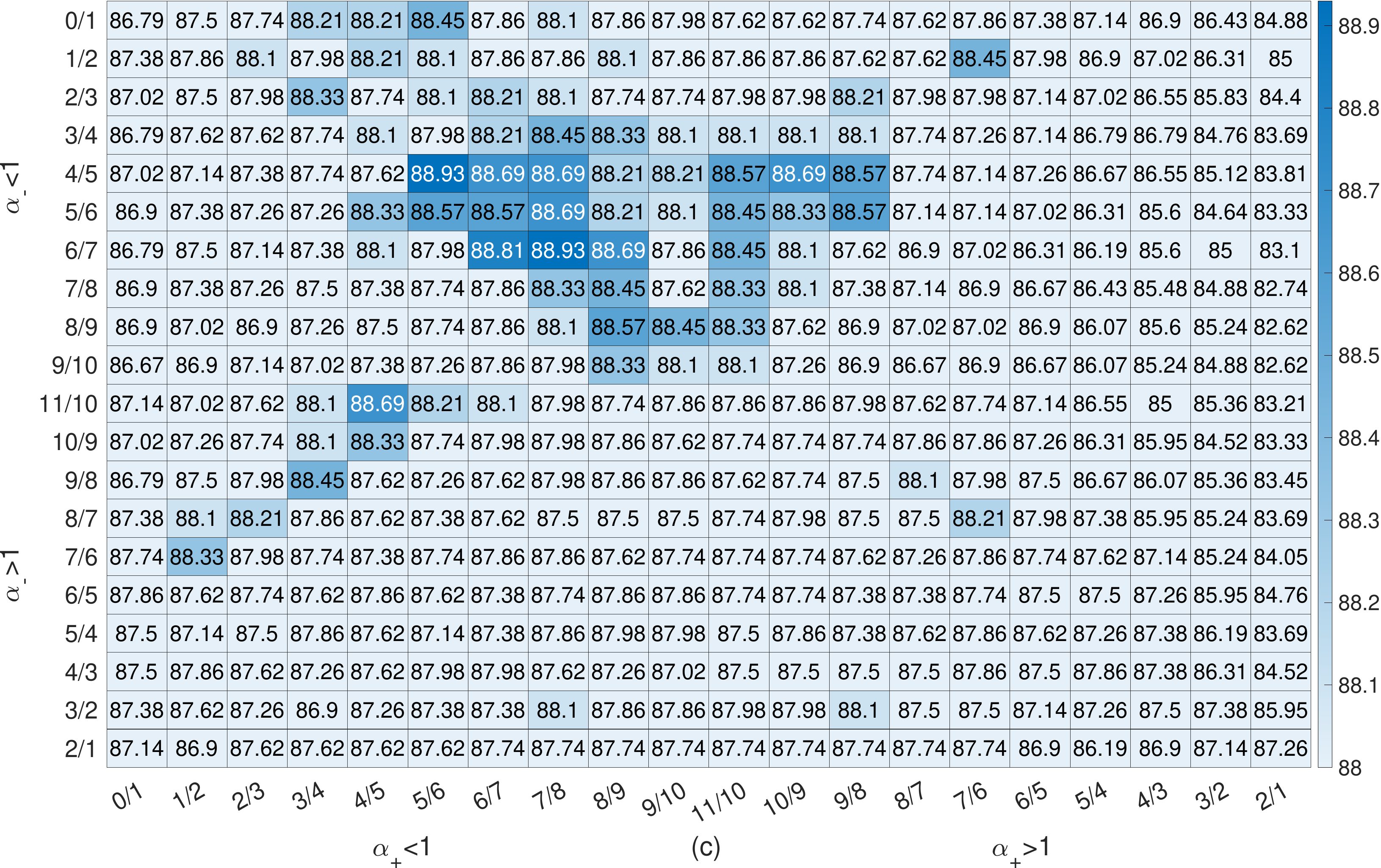} \hskip 0.07in
\includegraphics[width=3.4in]{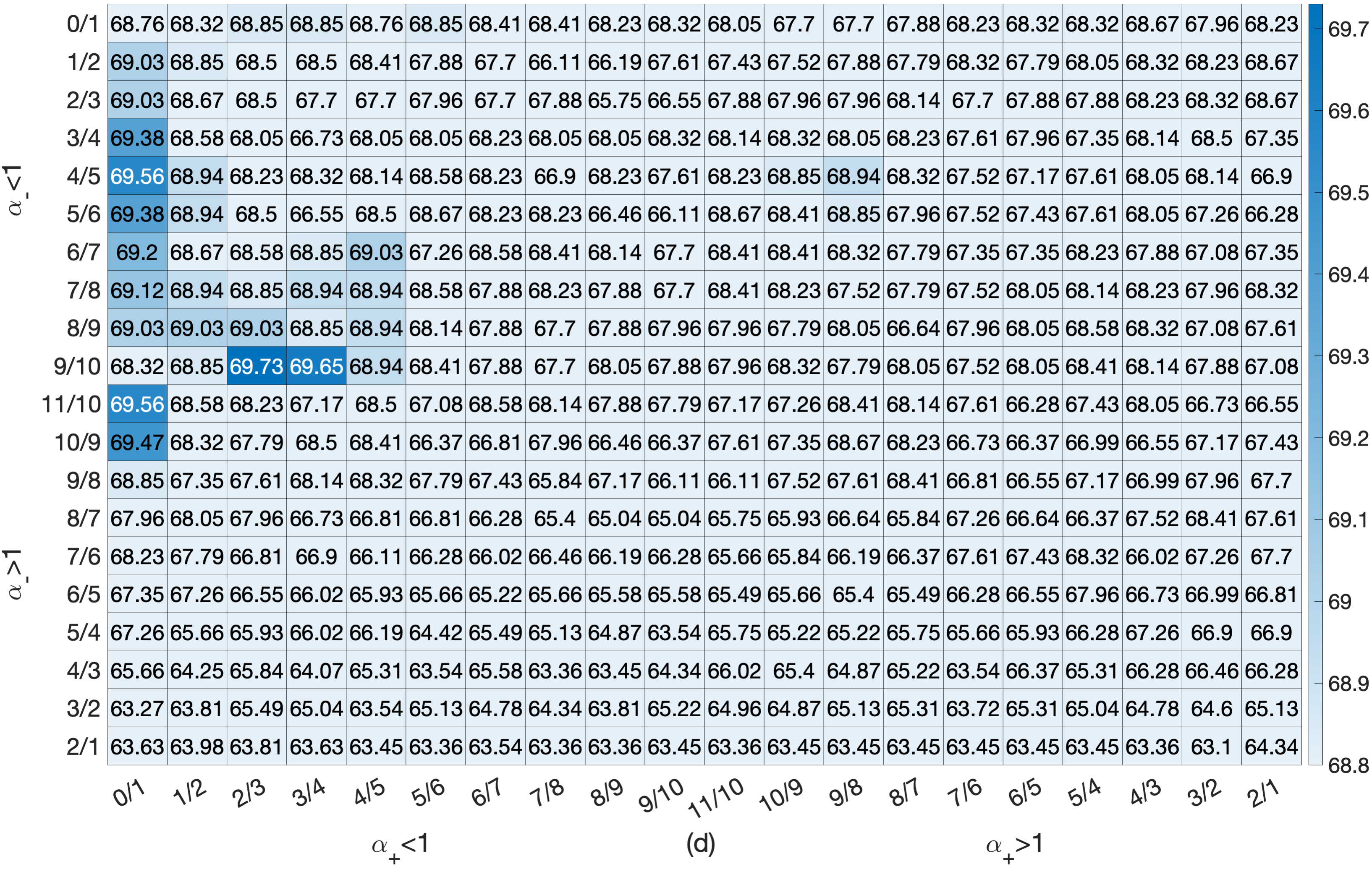}
\vskip 0.05in
\includegraphics[width=3.4in]{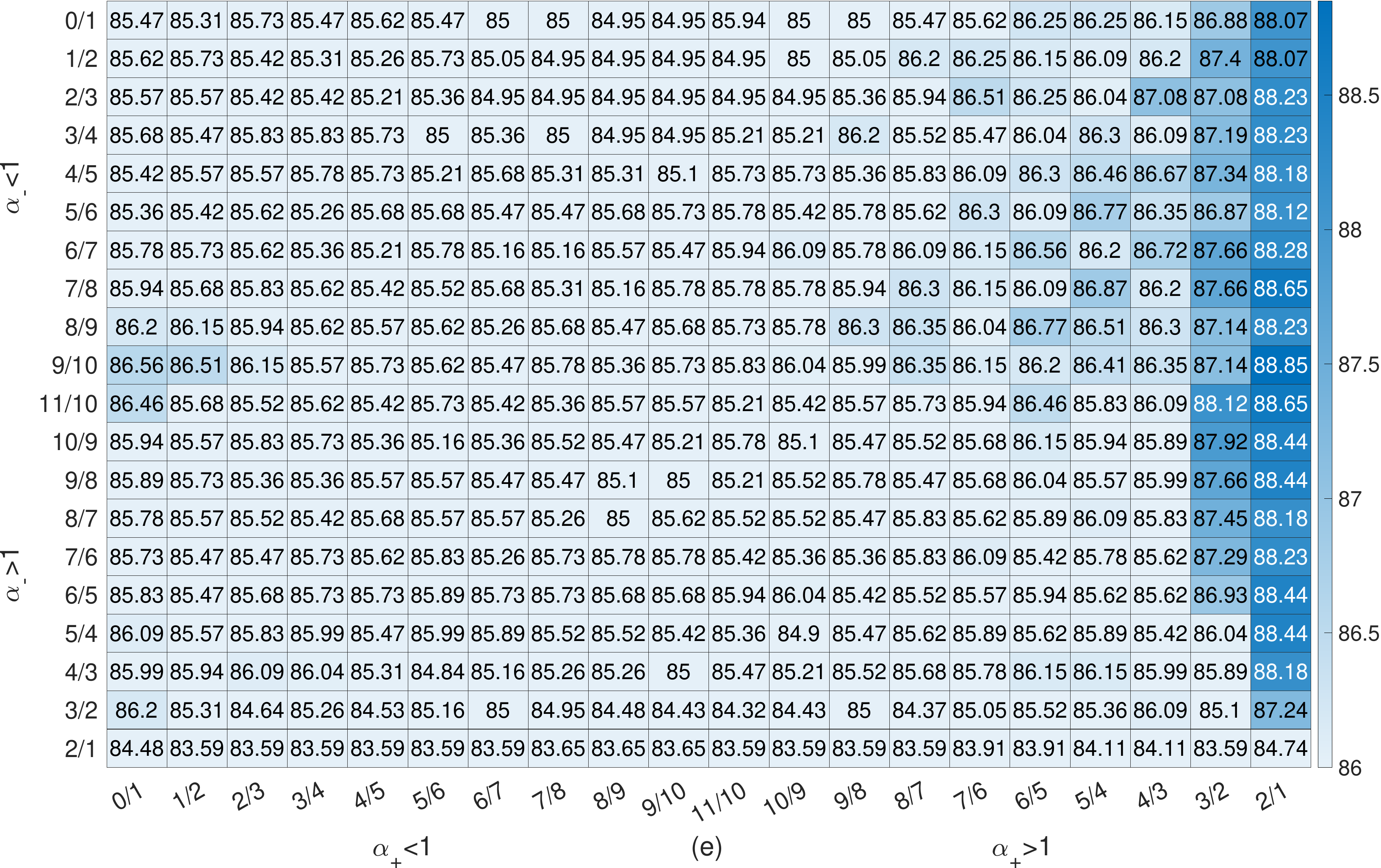} \hskip 0.07in
\includegraphics[width=3.4in]{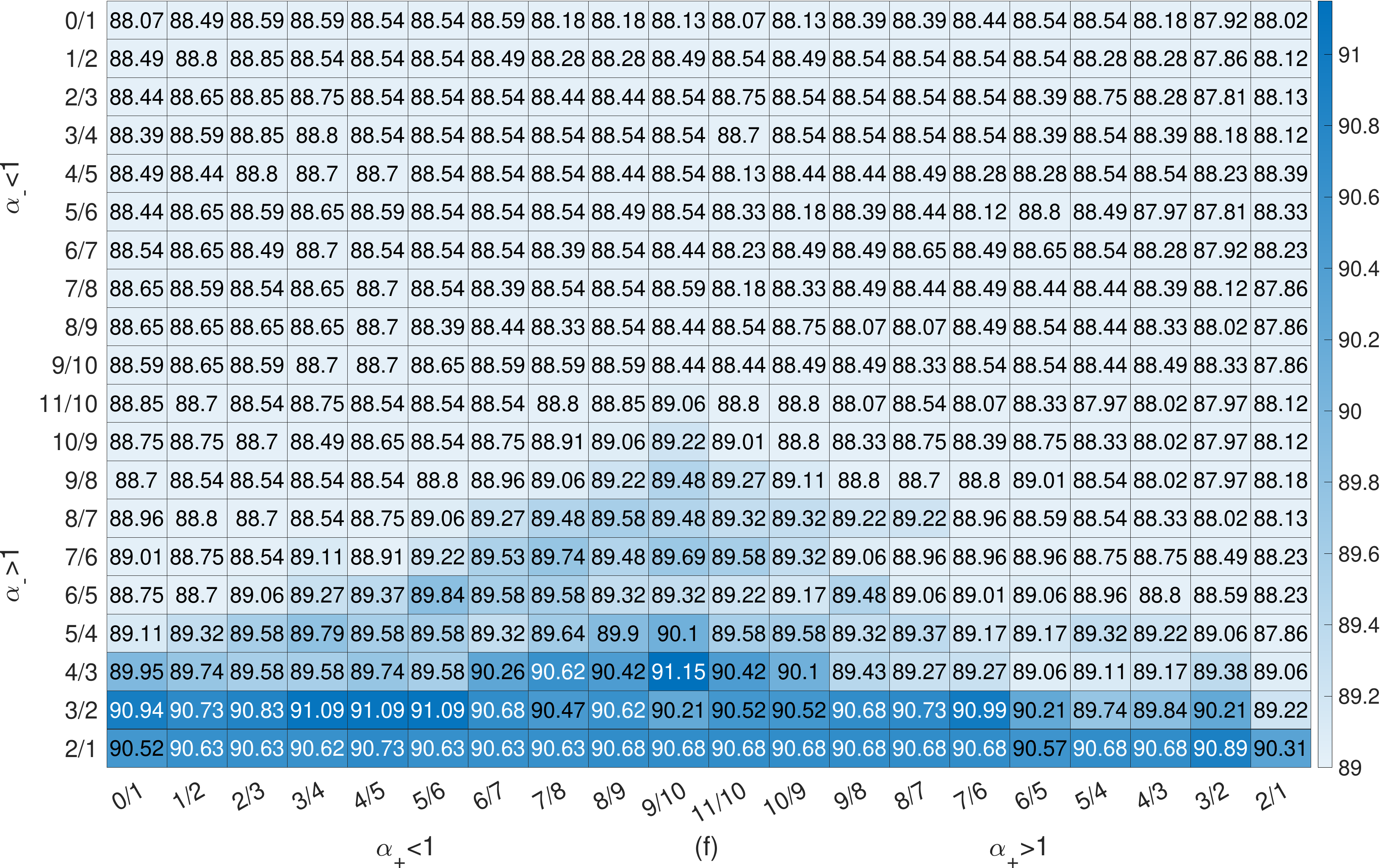}
\caption{Matrix patterns of the $20\times20$ SIC models of the test classification accuracy of (a) {\it statlog-australian-credit}, (b) {\it horse-colic}, (c) {\it ecoli}, (d) {\it arrhythmia}, (e) {\it energy-y1}, and (f)  {\it energy-y2}. Here {\it statlog-australian-credit} and {\it horse-colic} are two-class datasets. And  {\it ecoli}, {\it arrhythmia}, {\it energy-y1}, and {\it energy-y2} are multi-class datasets. Both {\it energy-y1} and {\it energy-y2} have the same input dataset but look for hyperplanes for opposite outputs, i.e., {\it y1} for the heating load, while {\it y2} for the cooling load~\cite{tsanas12}. The best performing $(\alpha_+,\alpha_-)$ for {\it energy-y1} and {\it energy-y2} exhibit opposite patterns: $(\alpha_+,\alpha_-)\approx (2,-)$ for {\it energy-y1}, and $(\alpha_+,\alpha_-) \approx (-, 2)$ for {\it energy-y2}. Overall, the best-performing region of multi-class datasets is more localized than that of two-class datasets.}
\label{fig:pattern}
\end{figure*}

Regarding multi-class datasets, the kernel-based classifier LIBSVM(C-SVC with RBF kernel) achieves the highest test classification accuracy. As presented in Table \ref{table:allperformance}, although the test accuracy of TOP$1$ is less than kernel-based LIBSVM, it still achieves a respectable $77.30\%$, which is $0.62\%$ better than TOP$1$-FL. In Figure \ref{fig:X} (b), we observe that the SIC model, which has only internal polynomial order parameters $k_{\pm} = \frac{1}{\abs{\alpha_{\pm} -1}}$, performs similarly to the convex focal loss, which has the external $\pi$-weight parameter and the internal $\xi$ and $\gamma$ parameters. Note that Figure \ref{fig:TWOandMULTI} (b) shows a pattern of the best-performing SIC model in the $(\alpha_+,\alpha_-)$-matrix. Compared to two-class, the ${\bm X}$-shaped pattern is rounded and biased toward $\alpha_->1$. In the case of convex focal loss in Figure \ref{fig:TWOandMULTI} (d), the best-performing region is much larger than the two-class convex focal loss. The region is shifted towards $\pi<0.5$.

Lastly, regarding computation time, L2SVM(primal) and logistic regression(primal) of LIBLINEAR are the fastest models. These models use the truncated Newton method~\cite{galli22} that is based on the unique Hessian structure of the large-margin linear classifier. On the other hand, for both the SIC model and convex focal loss, the proposed {\it Quasi-Newton(L-BFGS) optimization for virtual convex loss} is used. As shown in Figure \ref{fig:ALLclass} (b) in Appendix \ref{appB}, the $\pi$-weighted convex focal loss with $\pi=0.5,\gamma=1,\xi=0$, which corresponds to the logistic loss of LIBLINEAR, achieves reasonable performance-computation complexity, resulting in $78.30\%$ test accuracy at $83$ seconds. It is worth noting that the logistic regression of LIBLINEAR only obtains $77.93\%$ test accuracy at $60$ seconds.

Figure \ref{fig:pattern} demonstrates patterns of test classification accuracy for two-class datasets, {\it statlog-australian-credit} and {\it horse-colic} and for multi-class datasets, {\it ecoli}, {\it arrhythmia}, {\it energy-y1}, and {\it energy-y2}. Overall, the best test accuracy regions of multi-class datasets are more localized than those of two-class datasets. The ${\bm X}$-shaped pattern in Figure \ref{fig:TWOandMULTI} (a) is also observed in the {\it horse-colic} dataset in Figure \ref{fig:pattern} (b). Both {\it energy-y1} and {\it energy-y2} have the same input dataset but look for hyperplanes for opposite outputs. Specifically, {\it energy-y1} is used to determine the heating load, while {\it energy-y2} is used to determine the cooling load~\cite{tsanas12}. The best performing $(\alpha_+,\alpha_-)$ for {\it energy-y1} and {\it energy-y2} exhibit opposite patterns: $(\alpha_+,\alpha_-)\approx (2,-)$ for {\it energy-y1}, and $(\alpha_+,\alpha_-) \approx (-, 2)$ for {\it energy-y2}. Refer to Figure \ref{fig:pattern} (e) and (f) for further details. Understanding the correlation between the pattern of $(\alpha_+,\alpha_-)$-matrix and the structure of each dataset can be a valuable tool for multi-label classification and imbalanced classification.

\begin{table*}
\centering
\begin{tiny}\begin{tabular}{l||c|c|c|c||c|c|c|c||c|c|c||c}
\hline
\multicolumn{1}{l||}{{\bf MODEL}}&\multicolumn{4}{c||}{{\bf SIC model}\eqref{linmin}}&\multicolumn{4}{c||}{\textbf{Convex Focal Loss}\cite{lin20}}&\multicolumn{3}{c||}{\textbf{LIBLINEAR}\cite{fan08}}&\multicolumn{1}{c}{\textbf{LIBSVM}\cite{chang11}}\\\hline
\textbf{SubModel}&\textbf{TOP$1$}&\textbf{MaxA}&\textbf{Max$2$}&\textbf{MaxM}&\textbf{TOP$1$-FL}&\textbf{MaxA-FL}&\textbf{Max$2$-FL}&\textbf{MaxM-FL}&\textbf{Logistic}&\textbf{SVM}&\textbf{L2SVM}&\textbf{C-SVC}\\
$(\alpha_+,\alpha_-)$ or $(\pi,\gamma,\xi)$&-&$(\frac{7}{8},\frac{8}{7})$&$(\frac{3}{4},0)$&$(\frac{8}{9},\frac{11}{10})$&-&$(0.5,2,1)$&$(0.5,3,0)$&$(0.6,2,1)$&(primal)&(dual)&(primal)&\textbf{RBF} Kernel\\\hline\hline
\textbf{acute-inflammation  }&{\bf 100.00}&{\bf 100.00}&{\bf 100.00}&{\bf 100.00}&{\bf 100.00}&{\bf 100.00}&{\bf 100.00}&{\bf 100.00}&{\bf 100.00}&{\bf 100.00}&{\bf 100.00}&{\bf 100.00}\\\hline
\textbf{acute-nephritis  }&{\bf 100.00}&{\bf 100.00}&{\bf 100.00}&{\bf 100.00}&{\bf 100.00}&{\bf 100.00}&{\bf 100.00}&{\bf 100.00}&{\bf 100.00}&{\bf 100.00}&{\bf 100.00}&{\bf 100.00}\\\hline
\textbf{adult  }&84.31&84.24&84.00&84.26&84.38&84.38&84.29&83.73&84.28&84.33&84.05&{\bf 85.03}\\\hline
\textbf{balloons  }&{\bf 87.50}&{\bf 87.50}&{\bf 87.50}&{\bf 87.50}&{\bf 87.50}&{\bf 87.50}&{\bf 87.50}&72.50&{\bf 87.50}&{\bf 87.50}&{\bf 87.50}&{\bf 87.50}\\\hline
\textbf{bank  }&89.00&88.79&88.82&88.81&{\bf 89.51}&88.75&88.81&88.68&88.83&88.19&88.81&88.75\\\hline
\textbf{blood  }&{\bf 77.17}&76.58&76.68&76.15&76.74&75.99&76.20&76.74&76.20&76.20&75.67&76.84\\\hline
\textbf{breast-cancer  }&72.73&72.03&71.61&71.89&73.29&71.61&71.47&72.45&71.05&69.51&70.77&{\bf 75.66}\\\hline
\textbf{breast-cancer-wisc  }&{\bf 96.85}&96.62&96.56&96.62&{\bf 96.85}&96.50&96.33&95.82&96.50&96.62&96.68&95.99\\\hline
\textbf{breast-cancer-wisc-diag  }&{\bf 98.38}&98.17&97.75&98.10&98.31&98.24&98.17&98.24&98.24&97.46&97.96&{\bf 98.38}\\\hline
\textbf{breast-cancer-wisc-prog  }&81.41&79.60&78.59&79.80&{\bf 82.02}&78.79&79.19&79.60&78.38&75.15&79.39&78.38\\\hline
\textbf{chess-krvkp  }&96.93&96.12&96.77&96.46&97.01&96.71&96.53&96.78&96.48&96.33&96.68&{\bf 98.82}\\\hline
\textbf{congressional-voting  }&60.92&58.06&59.17&58.80&{\bf 61.94}&59.26&58.53&60.46&57.70&59.91&57.70&57.33\\\hline
\textbf{conn-bench-sonar-mines-rocks  }&79.23&76.35&75.96&77.88&78.46&75.77&77.12&75.19&75.58&77.50&75.77&{\bf 84.42}\\\hline
\textbf{connect-4  }&75.49&75.45&75.39&75.45&75.50&75.42&75.47&75.38&75.47&75.38&75.41&{\bf 86.26}\\\hline
\textbf{credit-approval  }&89.39&88.58&87.48&89.16&{\bf 89.62}&88.29&88.75&87.36&88.58&87.54&87.88&86.84\\\hline
\textbf{cylinder-bands  }&74.84&73.91&73.28&73.67&74.92&72.42&73.12&71.41&73.83&74.14&73.52&{\bf 77.34}\\\hline
\textbf{echocardiogram  }&86.77&84.92&84.92&84.62&87.38&84.62&84.62&86.15&85.85&{\bf 87.69}&86.15&87.38\\\hline
\textbf{fertility  }&89.60&88.40&88.00&89.60&{\bf 90.00}&84.80&89.60&86.80&85.60&87.20&86.00&86.80\\\hline
\textbf{haberman-survival  }&74.12&73.59&73.59&73.59&74.38&73.59&73.59&72.55&73.86&{\bf 74.77}&73.86&71.63\\\hline
\textbf{heart-hungarian  }&87.89&86.67&86.67&86.67&{\bf 88.03}&85.99&86.53&84.90&86.67&85.03&86.67&86.12\\\hline
\textbf{hepatitis  }&81.04&79.22&77.92&77.66&{\bf 81.82}&75.84&77.40&76.88&77.66&75.06&76.62&81.04\\\hline
\textbf{hill-valley  }&{\bf 92.08}&83.43&{\bf 92.08}&81.39&86.34&81.06&84.75&76.63&80.86&56.70&80.00&65.41\\\hline
\textbf{horse-colic  }&{\bf 89.71}&88.82&87.65&87.35&{\bf 89.71}&88.24&89.41&87.65&87.65&87.06&87.06&85.00\\\hline
\textbf{ilpd-indian-liver  }&{\bf 73.40}&{\bf 73.40}&72.10&72.16&72.58&72.23&71.96&71.55&71.96&71.48&73.06&71.41\\\hline
\textbf{ionosphere  }&88.80&86.29&86.40&86.51&87.09&86.63&86.63&86.97&88.34&87.77&86.74&{\bf 94.63}\\\hline
\textbf{magic  }&79.64&79.12&79.64&79.01&79.45&79.08&79.15&78.20&79.04&78.85&78.97&{\bf 87.20}\\\hline
\textbf{miniboone  }&90.35&90.16&89.77&90.26&90.30&90.30&90.26&90.14&89.79&89.98&89.04&{\bf 93.50}\\\hline
\textbf{molec-biol-promoter  }&{\bf 84.53}&77.36&76.60&76.60&79.25&76.98&77.36&76.60&78.11&76.23&76.23&76.98\\\hline
\textbf{mammographic  }&{\bf 83.75}&83.25&82.96&82.75&83.58&82.87&83.46&83.46&83.50&83.17&83.08&83.21\\\hline
\textbf{mushroom  }&95.24&94.52&95.02&94.60&97.46&95.62&94.58&94.85&94.40&97.64&93.89&{\bf 100.00}\\\hline
\textbf{musk-1  }&84.45&81.93&83.28&82.86&85.46&82.69&82.18&84.87&81.51&83.28&83.19&{\bf 91.18}\\\hline
\textbf{musk-2  }&94.99&94.63&94.56&94.76&95.25&94.80&94.63&95.06&94.67&95.02&94.72&{\bf 99.19}\\\hline
\textbf{oocytes-merluccius-nucleus-4d  }&{\bf 83.01}&82.04&82.04&81.41&82.54&82.07&82.11&81.57&81.80&80.74&82.74&82.07\\\hline
\textbf{oocytes-trisopterus-nucleus-2f  }&81.32&79.91&80.39&79.08&80.88&80.39&78.55&78.42&78.51&78.68&78.73&{\bf 82.19}\\\hline
\textbf{ozone  }&97.22&97.13&97.13&97.08&{\bf 97.32}&97.07&97.11&97.08&97.15&97.10&97.13&96.99\\\hline
\textbf{parkinsons  }&84.12&82.68&82.47&83.09&85.77&83.92&83.09&82.06&82.47&83.51&83.71&{\bf 91.34}\\\hline
\textbf{pima  }&76.98&75.52&75.36&76.61&{\bf 77.14}&76.72&76.61&76.09&76.35&75.31&76.35&73.80\\\hline
\textbf{pittsburg-bridges-T-OR-D  }&89.02&87.84&88.24&88.63&{\bf 90.20}&88.63&88.24&87.84&89.80&86.67&{\bf 90.20}&87.06\\\hline
\textbf{planning  }&70.55&67.47&65.27&67.03&{\bf 71.43}&65.93&67.47&67.91&64.40&70.11&65.27&69.45\\\hline
\textbf{ringnorm  }&77.95&77.37&77.86&77.11&77.87&77.25&76.89&74.16&76.86&77.51&77.11&{\bf 98.74}\\\hline
\textbf{spambase  }&92.88&92.82&92.72&92.44&93.03&92.77&92.45&91.42&92.37&92.78&92.22&{\bf 93.23}\\\hline
\textbf{spect  }&66.67&62.90&62.90&62.26&{\bf 67.85}&61.83&62.04&59.78&64.30&65.16&62.37&60.00\\\hline
\textbf{spectf  }&{\bf 64.60}&56.36&57.86&52.94&61.50&50.48&52.09&46.31&48.98&47.27&49.30&46.84\\\hline
\textbf{statlog-australian-credit  }&{\bf 68.12}&67.13&67.71&66.90&67.94&67.54&67.13&62.55&67.13&67.83&66.96&66.26\\\hline
\textbf{statlog-german-credit  }&77.16&76.96&76.44&76.60&{\bf 77.40}&75.92&76.20&74.96&76.80&76.24&77.16&76.08\\\hline
\textbf{statlog-heart  }&{\bf 88.59}&87.11&86.52&88.00&87.56&87.41&87.26&84.15&86.67&84.59&87.70&84.74\\\hline
\textbf{tic-tac-toe  }&97.91&97.91&97.91&97.91&97.91&97.91&97.91&97.91&97.91&97.91&97.91&{\bf 98.62}\\\hline
\textbf{titanic  }&78.09&77.55&77.55&77.55&78.09&77.55&77.55&77.55&77.55&77.55&77.55&{\bf 78.42}\\\hline
\textbf{trains  }&{\bf 64.00}&60.00&60.00&60.00&60.00&60.00&60.00&60.00&60.00&60.00&60.00&40.00\\\hline
\textbf{twonorm  }&97.66&97.57&97.54&97.56&97.68&97.62&97.44&97.40&97.57&97.49&97.49&{\bf 97.69}\\\hline
\textbf{vertebral-column-2clases  }&85.68&83.23&81.29&82.97&{\bf 87.61}&83.35&83.10&86.06&82.97&81.94&82.32&82.32\\\hline\hline
\textbf{Mean}&{\bf 83.96}&82.49&82.51&82.36&83.80&82.14&82.37&81.39&82.11&81.59&82.06&83.22\\\hline
\end{tabular}
\caption{A comparison of two-class test classification accuracy of the SIC model~\eqref{linmin} to $\pi$-weighted convex focal loss~\cite{lin20}, LIBLINEAR~\cite{fan08}, and LIBSVM~\cite{chang11}. Our results indicate that TOP$1$ performs better than kernel-based LIBSVM and the corresponding TOP$1$-FL. Max$2$ offers the best test classification accuracy among various linear classifiers.}
\label{table:TwoclassTOP5}
\end{tiny}
\end{table*}

\begin{table*}
\centering
\begin{tiny}\begin{tabular}{l||c|c|c|c||c|c|c|c||c|c|c||c}
\hline
\multicolumn{1}{l||}{{\bf MODEL}}&\multicolumn{4}{c||}{{\bf SIC model}\eqref{linmin}}&\multicolumn{4}{c||}{\textbf{Convex Focal Loss}\cite{lin20}}&\multicolumn{3}{c||}{\textbf{LIBLINEAR}\cite{fan08}}&\multicolumn{1}{c}{\textbf{LIBSVM}\cite{chang11}}\\\hline
\textbf{SubModel}&\textbf{TOP$1$}&\textbf{MaxA}&\textbf{Max$2$}&\textbf{MaxM}&\textbf{TOP$1$-FL}&\textbf{MaxA-FL}&\textbf{Max$2$-FL}&\textbf{MaxM-FL}&\textbf{Logistic}&\textbf{SVM}&\textbf{L2SVM}&\textbf{C-SVC}\\
$(\alpha_+,\alpha_-)$ or $(\pi,\gamma,\xi)$&-&$(\frac{7}{8},\frac{8}{7})$&$(\frac{3}{4},0)$&$(\frac{8}{9},\frac{11}{10})$&-&$(0.5,2,1)$&$(0.5,3,0)$&$(0.6,2,1)$&(primal)&(dual)&(primal)&\textbf{RBF} Kernel\\\hline\hline
\textbf{abalone  }&65.78&65.23&65.43&65.08&65.57&65.31&65.12&65.36&65.18&60.31&65.14&{\bf 66.22}\\\hline
\textbf{annealing  }&87.77&86.27&87.02&86.12&87.12&86.22&85.91&86.97&86.87&87.77&86.97&{\bf 92.08}\\\hline
\textbf{arrhythmia  }&{\bf 69.73}&65.40&68.85&67.88&69.47&68.32&68.76&67.88&68.50&65.93&64.16&67.52\\\hline
\textbf{audiology-std  }&{\bf 80.00}&76.00&78.40&76.00&{\bf 80.00}&72.80&71.20&73.60&70.40&68.80&68.80&65.60\\\hline
\textbf{balance-scale  }&88.46&88.14&88.01&88.27&88.46&88.40&88.14&88.40&88.21&88.21&88.08&{\bf 98.85}\\\hline
\textbf{breast-tissue  }&{\bf 69.06}&66.04&65.28&66.04&66.42&66.42&66.42&66.42&65.28&64.15&66.04&65.66\\\hline
\textbf{car  }&82.64&82.52&81.48&82.41&82.94&82.36&82.36&82.04&82.41&79.93&81.34&{\bf 98.47}\\\hline
\textbf{cardiotocography-10clases  }&79.49&78.46&77.57&78.19&78.74&77.46&77.99&77.84&78.04&73.89&77.38&{\bf 80.94}\\\hline
\textbf{cardiotocography-3clases  }&90.05&89.73&89.56&89.76&90.48&89.78&89.71&89.50&89.84&90.08&89.76&{\bf 91.95}\\\hline
\textbf{chess-krvk  }&28.07&27.83&27.38&27.77&28.10&28.03&27.82&27.96&27.84&17.36&27.59&{\bf 76.09}\\\hline
\textbf{conn-bench-vowel-deterding  }&57.65&52.73&54.24&52.80&55.68&53.26&52.65&54.09&51.89&50.53&54.47&{\bf 96.06}\\\hline
\textbf{contrac  }&51.66&50.98&50.65&50.82&51.22&50.38&50.87&50.49&50.87&46.44&50.03&{\bf 52.47}\\\hline
\textbf{dermatology  }&{\bf 99.34}&97.81&97.38&98.03&98.14&97.60&98.03&98.03&97.81&96.72&97.27&98.14\\\hline
\textbf{ecoli  }&{\bf 88.93}&87.50&88.21&87.74&88.57&88.10&88.21&88.21&87.98&86.07&87.98&85.60\\\hline
\textbf{energy-y1  }&88.85&85.26&85.47&85.57&89.06&85.00&85.57&85.89&86.20&85.16&86.09&{\bf 94.06}\\\hline
\textbf{energy-y2  }&91.15&89.48&88.59&88.85&90.10&88.75&88.23&88.54&89.43&88.85&89.69&{\bf 92.71}\\\hline
\textbf{flags  }&54.85&53.40&53.81&53.40&54.23&52.99&53.20&53.20&52.99&{\bf 55.05}&53.40&53.61\\\hline
\textbf{glass  }&68.22&62.06&63.93&62.80&65.42&64.86&62.80&64.11&63.93&60.00&65.23&{\bf 70.09}\\\hline
\textbf{hayes-roth  }&62.14&53.57&53.57&53.57&57.86&53.57&53.57&53.57&50.71&42.14&45.00&{\bf 77.14}\\\hline
\textbf{heart-cleveland  }&{\bf 64.11}&61.99&62.25&62.52&63.84&62.65&62.65&61.85&61.32&61.19&62.78&59.60\\\hline
\textbf{heart-switzerland  }&{\bf 41.97}&39.34&39.02&39.34&39.67&37.05&38.03&38.03&35.74&37.38&35.74&40.33\\\hline
\textbf{heart-va  }&{\bf 32.20}&29.40&29.40&29.60&{\bf 32.20}&29.60&29.40&29.80&28.20&31.60&27.00&28.00\\\hline
\textbf{image-segmentation  }&91.14&90.13&90.70&90.67&{\bf 91.30}&91.28&90.55&90.89&90.39&90.82&90.18&90.94\\\hline
\textbf{iris  }&{\bf 97.33}&93.33&94.67&94.40&96.00&94.93&94.67&94.67&94.67&91.47&94.40&96.00\\\hline
\textbf{led-display  }&{\bf 72.12}&71.12&70.80&71.56&71.80&70.36&71.32&70.48&70.44&69.16&70.32&69.60\\\hline
\textbf{lenses  }&{\bf 83.33}&75.00&75.00&75.00&75.00&75.00&75.00&75.00&75.00&75.00&75.00&75.00\\\hline
\textbf{letter  }&72.61&72.49&70.06&72.38&72.27&71.95&72.25&72.24&72.22&60.57&70.26&{\bf 96.53}\\\hline
\textbf{libras  }&64.56&61.56&61.78&62.11&65.00&62.56&62.33&62.22&62.67&61.56&61.11&{\bf 78.22}\\\hline
\textbf{low-res-spect  }&89.96&88.15&88.60&88.15&88.83&88.23&87.77&87.55&88.30&88.30&88.23&{\bf 90.42}\\\hline
\textbf{lung-cancer  }&{\bf 62.50}&61.25&57.50&{\bf 62.50}&{\bf 62.50}&{\bf 62.50}&{\bf 62.50}&{\bf 62.50}&57.50&60.00&58.75&56.25\\\hline
\textbf{lymphography  }&84.59&83.24&81.62&83.78&{\bf 85.14}&82.16&82.97&81.89&82.43&82.16&82.43&80.27\\\hline
\textbf{molec-biol-splice  }&84.13&82.98&82.70&82.75&82.81&82.71&82.46&82.62&82.41&81.93&81.98&{\bf 85.03}\\\hline
\textbf{nursery  }&90.64&89.87&89.88&89.84&90.04&89.86&89.86&89.98&89.84&89.51&89.81&{\bf 99.31}\\\hline
\textbf{oocytes-merluccius-states-2f  }&{\bf 92.41}&91.86&92.02&91.94&92.09&91.70&91.74&91.70&91.51&92.02&91.66&91.66\\\hline
\textbf{oocytes-trisopterus-states-5b  }&93.29&92.68&92.85&92.72&92.89&92.76&92.81&92.76&92.50&93.07&92.59&{\bf 93.60}\\\hline
\textbf{optical  }&94.98&94.78&94.64&94.82&95.05&94.92&94.79&94.75&94.74&94.31&94.76&{\bf 97.28}\\\hline
\textbf{page-blocks  }&{\bf 96.83}&96.43&96.21&96.39&96.62&96.24&96.24&96.12&96.29&95.96&96.05&96.64\\\hline
\textbf{pendigits  }&89.86&89.68&89.53&89.38&89.85&89.52&89.54&89.66&89.85&89.47&89.67&{\bf 97.66}\\\hline
\textbf{pittsburg-bridges-MATERIAL  }&88.30&87.92&87.17&86.79&87.55&86.42&86.42&86.79&87.55&{\bf 88.68}&{\bf 88.68}&85.66\\\hline
\textbf{pittsburg-bridges-REL-L  }&{\bf 69.80}&67.45&67.06&67.45&68.63&68.24&63.92&67.84&65.49&66.27&64.31&{\bf 69.80}\\\hline
\textbf{pittsburg-bridges-SPAN  }&79.57&76.09&69.57&74.35&{\bf 80.00}&77.39&74.35&77.39&75.65&74.78&75.22&70.00\\\hline
\textbf{pittsburg-bridges-TYPE  }&{\bf 65.77}&63.85&{\bf 65.77}&62.31&65.38&62.31&62.69&63.46&64.62&58.85&63.46&61.15\\\hline
\textbf{plant-margin  }&78.67&77.70&76.12&77.90&78.53&78.03&78.07&78.00&69.50&58.50&64.80&{\bf 80.12}\\\hline
\textbf{plant-shape  }&55.25&54.22&49.47&54.65&54.62&53.65&53.85&53.40&50.37&40.92&47.10&{\bf 67.00}\\\hline
\textbf{plant-texture  }&80.50&78.15&79.32&78.90&79.77&79.05&78.92&78.97&76.35&70.46&75.02&{\bf 80.88}\\\hline
\textbf{post-operative  }&{\bf 67.11}&66.67&63.56&63.11&66.67&64.89&62.67&64.44&54.67&57.33&55.56&63.11\\\hline
\textbf{primary-tumor  }&{\bf 49.09}&46.55&43.64&45.70&47.76&47.52&45.33&46.55&45.21&41.21&42.42&45.21\\\hline
\textbf{seeds  }&{\bf 94.29}&93.71&93.52&{\bf 94.29}&{\bf 94.29}&93.52&93.71&93.52&92.76&90.48&92.19&91.43\\\hline
\textbf{semeion  }&92.69&92.19&91.56&91.91&92.46&91.48&91.96&91.58&89.12&86.18&85.43&{\bf 94.65}\\\hline
\textbf{soybean  }&86.93&84.84&85.88&85.10&86.80&85.88&84.97&85.62&85.36&{\bf 88.24}&86.67&87.71\\\hline
\textbf{statlog-image  }&93.11&92.28&92.16&92.28&93.65&92.92&92.24&92.31&91.76&91.90&91.26&{\bf 95.64}\\\hline
\textbf{statlog-landsat  }&83.99&82.20&81.30&81.99&81.73&81.58&81.68&81.54&81.91&80.37&81.15&{\bf 91.89}\\\hline
\textbf{statlog-shuttle  }&96.54&93.49&93.41&93.50&94.55&93.88&93.61&93.63&93.12&89.87&92.49&{\bf 99.90}\\\hline
\textbf{statlog-vehicle  }&78.68&77.68&78.53&77.78&79.10&77.73&77.73&77.92&77.78&77.07&77.68&{\bf 81.89}\\\hline
\textbf{steel-plates  }&71.61&71.03&69.75&70.95&71.53&70.76&70.62&70.52&70.47&70.02&70.62&{\bf 75.03}\\\hline
\textbf{synthetic-control  }&97.93&97.27&93.00&97.33&97.73&94.27&96.60&96.07&91.33&87.87&89.20&{\bf 99.00}\\\hline
\textbf{teaching  }&49.60&48.27&46.67&48.53&50.13&48.80&48.80&48.53&48.00&40.80&46.67&{\bf 52.53}\\\hline
\textbf{thyroid  }&96.04&94.75&94.84&94.89&96.06&95.06&95.04&94.99&95.06&95.23&94.46&{\bf 96.66}\\\hline
\textbf{vertebral-column-3clases  }&{\bf 86.97}&85.55&85.16&85.81&86.19&84.52&85.16&83.87&84.00&83.61&84.77&83.35\\\hline
\textbf{wall-following  }&72.85&69.86&69.27&69.77&72.17&69.93&69.52&69.92&69.57&72.89&66.54&{\bf 90.66}\\\hline
\textbf{waveform  }&87.40&86.86&87.09&86.84&{\bf 87.43}&87.08&86.88&87.25&86.66&86.74&86.70&86.92\\\hline
\textbf{waveform-noise  }&{\bf 86.15}&85.86&85.78&85.97&86.06&85.88&86.02&85.86&85.98&85.77&85.90&85.44\\\hline
\textbf{wine  }&{\bf 98.88}&98.43&{\bf 98.88}&{\bf 98.88}&{\bf 98.88}&97.75&98.43&98.20&{\bf 98.88}&{\bf 98.88}&98.65&97.75\\\hline
\textbf{wine-quality-red  }&58.22&57.25&56.52&56.47&58.12&56.42&56.47&56.40&56.85&56.35&56.40&{\bf 61.20}\\\hline
\textbf{wine-quality-white  }&53.87&53.13&53.07&53.32&54.02&53.74&53.51&54.02&53.56&47.10&53.27&{\bf 60.21}\\\hline
\textbf{yeast  }&60.75&59.95&59.65&59.62&{\bf 61.11}&60.49&60.19&60.38&60.16&52.26&59.97&60.75\\\hline
\textbf{zoo  }&{\bf 96.00}&{\bf 96.00}&{\bf 96.00}&{\bf 96.00}&{\bf 96.00}&{\bf 96.00}&{\bf 96.00}&{\bf 96.00}&{\bf 96.00}&95.60&{\bf 96.00}&{\bf 96.00}\\\hline\hline
\textbf{Mean}&77.30&{\it 75.57}&75.19&{\it 75.57}&76.68&75.53&75.36&75.55&74.75&72.94&74.18&{\bf 79.96}\\\hline
\end{tabular}
\caption{A comparison of multi-class test classification accuracy of the SIC model~\eqref{linmin} to $\pi$-weighted convex focal loss~\cite{lin20}, LIBLINEAR~\cite{fan08}, and LIBSVM~\cite{chang11}. Although the accuracy of TOP$1$ is less than kernel-based LIBSVM, it still achieves $77.30\%$ accuracy, which is $0.62\%$ better than TOP$1$-FL. Within linear classifiers, the SIC model performs similarly to the convex focal loss having the external $\pi$-weight parameter, i.e., MaxM achieves $75.57\%$ vs. MaxM-FL achieves $75.55\%$.}
\label{table:MxclassTOP5}
\end{tiny}
\end{table*}

\section{Conclusion\label{seccon}}
This article introduces SIGTRON, an extended asymmetric sigmoid function with Perceptron, and its virtualized loss function called virtual SIGTRON-induced loss function. Based on this loss function, we propose the SIGTRON-imbalanced classification (SIC) model for cost-sensitive learning. Unlike other models, the SIC model does not use an external $\pi$-weight on the loss function but instead has an internal two-dimensional parameter $(\alpha_+,\alpha_-)$-matrix. We show that when a training dataset is close to a well-balanced condition, the SIC model is moderately resilient to variations in the dataset through the skewed hyperplane equation. When $r_{sc}$ is not severe, the proposed SIC model could be used as an alternative cost-sensitive learning model that does not require an external $\pi$-weight parameter. Additionally, we introduce {\it quasi-Newton(L-BFGS) optimization for virtual convex loss} with an interval-based bisection line search. This optimization is a competitive framework for a convex minimization problem, compared to conventional L-BFGS with cubic-interpolation-based line search. We utilize the proposed optimization framework for the SIC model and the $\pi$-weighted convex focal loss. Our SIC model has shown better performance in terms of test classification accuracy with 118 diverse datasets compared to the $\pi$-weighted convex focal loss and LIBLINEAR. In binary classification problems, where the severity of $r_{sc}$ is not high, selecting the best SIC model for each dataset(TOP$1$) can lead to better performance than the kernel-based LIBSVM. Specifically, TOP$1$ achieves a test classification accuracy of $83.96\%$, which is $0.74\%$ better than the accuracy of LIBSVM and $0.16\%$ better than the accuracy of TOP$1$-FL of the convex focal loss. In multi-class classification problems, although the test accuracy of TOP$1$ is lower than that of kernel-based LIBSVM, it achieves an accuracy of $77.30\%$, which is $0.62\%$ better than the accuracy of TOP$1$-FL of the convex focal loss. Last but not least, the proposed SIC model, which includes an $(\alpha_+,\alpha_-)$-matrix parameter, could be a valuable tool for analyzing various structures of datasets, such as $r_{sc}$-inconsistency and multi-label structures.

\appendix
\subsection{Proof of Theorem \ref{th:derivative-bernoulli-like}}\label{proofofnthderivative}

Let $\alpha=1$, then we have $s_{\alpha,c}(x) = \frac{1}{1+\exp(-x)} \in C^{\infty}(\R)$. Therefore, we only consider the case $\alpha\not=1$. Note that  $\nabla^n s_{\alpha,c}(x) = 0$, for all $x \in \R \setminus \dom(\sigma_{\alpha,c})$ and $n=1,2,3,\cdots$. Also, for $x \in int(\dom(\sigma_{\alpha,c}))$, let $y = 1 + \frac{x}{c_{\alpha}}>0$ and $a = \frac{1}{1-\alpha}$, then we get $s_{\alpha,c}(c_{\alpha}(y-1)) = \frac{1}{1 + y^a}$. Thus, it is not difficult to see $s_{\alpha,c}(x) = \sigma_{\alpha,c}(x) \in C^{\infty}(int(\dom(\sigma_{\alpha,c})))$. 

For the continuity of $\nabla^n s_{\alpha,c}(x)$ on $x \in \R$, we only need to check $s_{\alpha,c}(-c_{\alpha})=0$. Let us assume that \eqref{nth-derivative_s} is true. (A) When $0 \le \alpha<1$, $\exp_{\alpha,c}(-x)=0$ at $x=-c_{\alpha}$. Thus, we only need to consider numerator of $F_{n,k}(x)$, i.e., $(\exp_{\alpha,c}(-x))^{k-n(1-\alpha)}$. $F_{n,k}(c_{\alpha}) = 0$, if $k - n(1-\alpha) > 0$ for all $k=1,2,\cdots,n$. Now, we get $1>\alpha>1-1/n$. (B) When $\alpha>1$, $\exp_{\alpha,c}(-x)=+\infty$ at $x=-c_{\alpha}$. Thus, we have
$F_{n,k}(-c_{\alpha}) = \lim_{x \rightarrow -c_{\alpha}} c A_{n,k}\left(\frac{1}{1-\alpha}\right) (\exp_{\alpha,c}(-x))^{-n(1-\alpha) -1} = 0$
if $-n(1-\alpha)-1<0$. It means $1<\alpha< 1 + 1/n$. From (A) and (B), we get $\nabla^n s_{\alpha,c}(x) \in C^{n}(\R)$ for $\alpha \in \left(1-\frac{1}{n},1+\frac{1}{n}\right)$.

Now, we want to show \eqref{nth-derivative_s} by induction for $x \in int(\dom(\sigma_{\alpha,c}))$. Let $y=1+\frac{x}{c_{\alpha}}$ and $a = \frac{1}{1-\alpha}$, then \eqref{nth-derivative_s} becomes
\begin{equation}\label{sim-nth-derivative_s}
\frac{d^n}{dy^n} \frac{1}{1+y^a} = \sum_{k=1}^n B_{n,k}(a)\frac{y^{ka-n}}{(1+y^a)^{k+1}}
\end{equation}
where $B_{n,k}(a) = (-1)^{n+k}k! \sum_{l=0}^n {\left[n \atop l\right]}{\left\{ l \atop k \right\}} (-a)^{l}$. 

({\bf I}) Let $n=1$. Then the left-hand side of \eqref{sim-nth-derivative_s} is
$\frac{d}{dy}\frac{1}{1+y^a} = \frac{-ay^{a-1}}{(1+y^a)^2}$. The right-hand side is $B_{1,1}(a)\frac{x^{a-1}}{(1+x^a)^2}$ where $B_{1,1}(a) = \left[ 1 \atop 0 \right]\left\{ 0 \atop 1 \right\} + \left[ 1 \atop 1 \right]\left\{ 1 \atop 1 \right\}(-a) = -a$. For the computation of the Stirling number of the first and second kind, we use the following convention and rule in \cite{graham94}: $\left\{ 0 \atop 0 \right\} = \left[ 0 \atop 0 \right] = 1$ and $\left\{ a \atop 0 \right\} = \left[ a \atop 0  \right]=0$ for $a \ge 1$. Also, we have $\left\{ a \atop 1 \right\} = 1$ and $\left[ a \atop 1 \right] = (a-1)!$ with $0!=1$, for $a\ge1$. Additionally, $\left\{ a \atop b \right\} = \left[ a \atop b \right] = 0$ if $b>a \ge 0$. 

({\bf II}) For $n>1$, let \eqref{sim-nth-derivative_s} be true. Then, for $n+1$, we need to show that
\begin{equation}\label{th:eq1}
\frac{d}{dy}\left( \sum_{k=1}^n B_{n,k}(a)\frac{y^{ak-n}}{(1+y^a)^{k+1}} \right) 
= \sum_{k=1}^{n+1}B_{n+1,k}(a)\frac{y^{ak-(n+1)}}{(1+y^a)^{k+1}},
\end{equation}
where
\begin{equation}\label{th:eq2}
\frac{d}{dy}\left( \sum_{k=1}^n B_{n,k}(a)\frac{y^{ak-n}}{(1+y^a)^{k+1}} \right) = \sum_{k=1}^n B_{n,k}(a)\left(\frac{(ak-n)y^{ak-(n+1)}}{(1+y^a)^{k+1}} - (k+1)\frac{ay^{a(k+1)-(n+1)}}{(1+y^a)^{k+2}} \right).
\end{equation}
From \eqref{th:eq1} and \eqref{th:eq2}, we get
\begin{equation}\label{eqA}
B_{n+1,k}(a) = (-ka) B_{n,k-1}(a) + (ka-n) B_{n,k}(a)
\end{equation}
where $B_{n,0}(a) = B_{n,n+1}(a)=0$. It comes from the rule of the Stirling number in ({\bf I}). Now, we only need to prove \eqref{eqA}. The left-hand side of \eqref{eqA} is
\begin{eqnarray*}
B_{n+1,k}(a) &=& (-1)^k k! \sum_{l=0}^{n+1}\left[n+1 \atop l \right]\left\{ l \atop k \right\}(-1)^{n+1 - l}a^l  \\
&=& (-1)^k k! \left(\left\{ n+1 \atop k \right\} a^{n+1}  +  \sum_{l=1}^{n}\left( n\left[ n \atop l \right] + \left[ n \atop l-1 \right] \right)\left\{ l \atop k \right\}(-1)^{n+1 - l}a^l \right)   
\end{eqnarray*}
where $\left[ n+ 1 \atop 0 \right] = 0$, $\left[ n+ 1 \atop n+1 \right] = 1$, and $\left[ n+1 \atop l \right] = n\left[ n \atop l \right] + \left[ n \atop l-1 \right]$(see the rule of the Stirling number in ({\bf I}) and \cite{graham94}). By using the equivalence $\left\{l+1 \atop k \right\} = \left\{l \atop k-1 \right\} + k \left\{ l \atop k \right\}$ and $\left[ n \atop 0 \right]=0$, the right-hand side of \eqref{eqA} is simplified to the following equation.
\begin{eqnarray*}
(-ka)B_{n,k-1}(a) + (ka-n)B_{n,k}(a) = (-1)^{k} k! \sum_{l=0}^n \left( a\left\{ l+1 \atop k \right\} - n \left\{ l \atop k\right\} \right) \left[ n \atop l \right] (-1)^{n-l} a^l. 
\end{eqnarray*}
By dividing $(-1)^k k!$ and adding $\sum_{l=1}^n n\left[ n \atop l \right]\left\{l \atop k \right\}(-1)^{n-l}a^l$ on both sides, we obtain the equivalence \eqref{eqA}.

\subsection{The structure of the dataset and classification results of all-class\label{appB}}
We summarize the structure of $118$ datasets used in our experiments and present classification accuracy of the SIC model and $\pi$-weighted convex focal loss of all-class. Table \ref{2classimb} summarizes two-class datasets. For each dataset, we describe the number of instances, the size of training dataset, the size of test dataset, the size of class, and the feature dimension. Additionally, we show imbalancedness, i.e., $r_c/r_c{\bf T}/r_c{\bf Te}$ for class-imbalance ratio of combined/training/test dataset and $r_{sc}/r_{sc}{\bf T}/r_{sc}{\bf Te}$ for scale-class-imbalance ratio of combined/training/test dataset. Table \ref{mclassimb} summarizes multi-class datasets. For each dataset, we describe the number of instances, the size of training dataset, the size of test dataset, the feature dimension, and the size of class. Additionally, we show imbalancedness, i.e., minimum and maximum of $r_c$ for combined/training dataset: $r_c{\bf m}/r_c{\bf M}/r_c{\bf Tm}/r_c{\bf TM}$. Also, minimum and maximum of $r_{sc}$ for combined/training dataset: $r_{sc}{\bf m}/r_{sc}{\bf M}/r_{sc}{\bf Tm}/r_{sc}{\bf TM}$.

Figure \ref{fig:ALLclass} presents classification accuracy matrix and the corresponding histogram of all-class for $20\times20$ SIC models and $19\times8$ convex focal losses. The test classification accuracy of all SIC models ranges between $76.88\%$ and $78.56\%$. On the other hand, the test classification accuracy of all convex focal losses ranges between $71.04\%$ and $78.39\%$. 

\begin{table*}
\centering
\begin{tiny}\begin{tabular}{l||c|c|c|c|c||c|c||c|c||c|c}
\hline
&\textbf{instance}&\textbf{train}&\textbf{test}&\textbf{dim}&\textbf{class}&\textbf{$r_c$}&\textbf{$r_{sc}$}&\textbf{$r_c$T}&\textbf{$r_{sc}$T}&\textbf{$r_c$Te}&\textbf{$r_{sc}$Te}\\\hline\hline
\textbf{acute-inflammation  }&120&60&60&6&2&1.03&1.01&1&0.95&1.07&1.06\\\hline
\textbf{acute-nephritis  }&120&60&60&6&2&1.40&1.14&1.40&1.14&1.40&1.14\\\hline
\textbf{adult  }&48842&32561&16281&14&2&3.18&2.10&3.15&2.09&3.23&2.14\\\hline
\textbf{balloons  }&16&8&8&4&2&1.29&1.18&1&0.72&1.67&2.15\\\hline
\textbf{bank  }&4521&2261&2260&16&2&7.68&4.43&7.66&4.38&7.69&4.45\\\hline
\textbf{blood  }&748&374&374&4&2&3.20&2.63&3.20&2.41&3.20&2.80\\\hline
\textbf{breast-cancer  }&286&143&143&9&2&2.36&1.92&2.33&1.91&2.40&1.89\\\hline
\textbf{breast-cancer-wisc  }&699&350&349&9&2&1.90&1.12&1.89&1.13&1.91&1.11\\\hline
\textbf{breast-cancer-wisc-diag  }&569&285&284&30&2&1.68&1.06&1.69&0.99&1.68&1.14\\\hline
\textbf{breast-cancer-wisc-prog  }&198&99&99&33&2&3.21&2.06&3.12&2.42&3.30&1.98\\\hline
\textbf{chess-krvkp  }&3196&1598&1598&36&2&0.91&0.95&0.92&0.97&0.91&0.95\\\hline
\textbf{congressional-voting  }&435&218&217&16&2&1.59&1.53&1.60&1.54&1.58&1.51\\\hline
\textbf{conn-bench-sonar-mines-rocks  }&208&104&104&60&2&1.14&1.04&1.12&1.02&1.17&1.06\\\hline
\textbf{connect-4  }&67557&33779&33778&42&2&3.06&2.94&3.06&2.95&3.06&2.92\\\hline
\textbf{credit-approval  }&690&345&345&15&2&0.80&0.90&0.81&0.99&0.80&0.83\\\hline
\textbf{cylinder-bands  }&512&256&256&35&2&0.64&0.75&0.64&0.79&0.64&0.74\\\hline
\textbf{echocardiogram  }&131&66&65&10&2&2.05&1.47&2&1.52&2.10&1.44\\\hline
\textbf{fertility  }&100&50&50&9&2&7.33&5.47&7.33&4.96&7.33&6.07\\\hline
\textbf{haberman-survival  }&306&153&153&3&2&2.78&2.53&2.73&2.22&2.83&2.79\\\hline
\textbf{heart-hungarian  }&294&147&147&12&2&1.77&1.26&1.77&1.37&1.77&1.17\\\hline
\textbf{hepatitis  }&155&78&77&19&2&0.26&0.52&0.26&0.65&0.26&0.44\\\hline
\textbf{hill-valley  }&606&303&303&100&2&1.03&1.03&1.02&0.84&1.03&0.83\\\hline
\textbf{horse-colic  }&368&300&68&25&2&1.71&1.30&1.75&1.31&1.52&1.22\\\hline
\textbf{ilpd-indian-liver  }&583&292&291&9&2&2.49&2.04&2.48&2.07&2.51&2.06\\\hline
\textbf{ionosphere  }&351&176&175&33&2&0.56&0.81&0.56&0.86&0.56&0.83\\\hline
\textbf{magic  }&19020&9510&9510&10&2&1.84&1.51&1.84&1.51&1.84&1.52\\\hline
\textbf{miniboone  }&130064&65032&65032&50&2&0.39&0.55&0.39&0.55&0.39&0.55\\\hline
\textbf{molec-biol-promoter  }&106&53&53&57&2&1&1&0.96&1.01&1.04&0.99\\\hline
\textbf{mammographic  }&961&481&480&5&2&1.16&1.09&1.16&1.11&1.16&1.06\\\hline
\textbf{mushroom  }&8124&4062&4062&21&2&1.07&1.03&1.07&1.05&1.07&1\\\hline
\textbf{musk-1  }&476&238&238&166&2&1.30&1.07&1.29&1.53&1.31&0.93\\\hline
\textbf{musk-2  }&6598&3299&3299&166&2&5.49&1.74&5.48&1.78&5.49&1.73\\\hline
\textbf{oocytes-merluccius-nucleus-4d  }&1022&511&511&41&2&0.49&0.59&0.49&0.55&0.49&0.62\\\hline
\textbf{oocytes-trisopterus-nucleus-2f  }&912&456&456&25&2&0.73&0.79&0.73&0.85&0.73&0.71\\\hline
\textbf{ozone  }&2536&1268&1268&72&2&33.74&6.98&33.27&7&34.22&7.31\\\hline
\textbf{parkinsons  }&195&98&97&22&2&0.33&0.73&0.32&0.92&0.33&0.56\\\hline
\textbf{pima  }&768&384&384&8&2&1.87&1.53&1.87&1.51&1.87&1.53\\\hline
\textbf{pittsburg-bridges-T-OR-D  }&102&51&51&7&2&6.29&4.10&6.29&4.71&6.29&3.41\\\hline
\textbf{planning  }&182&91&91&12&2&2.50&2.46&2.50&2.39&2.50&2.45\\\hline
\textbf{ringnorm  }&7400&3700&3700&20&2&0.98&0.99&0.98&1&0.98&0.98\\\hline
\textbf{spambase  }&4601&2301&2300&57&2&1.54&1.16&1.54&1.16&1.54&1.17\\\hline
\textbf{spect  }&265&79&186&22&2&1.41&1.21&2.04&1.41&1.21&1.11\\\hline
\textbf{spectf  }&267&80&187&44&2&0.26&0.52&1&1&0.09&0.26\\\hline
\textbf{statlog-australian-credit  }&690&345&345&14&2&0.47&0.49&0.47&0.49&0.47&0.48\\\hline
\textbf{statlog-german-credit  }&1000&500&500&24&2&2.33&1.82&2.33&1.81&2.33&1.84\\\hline
\textbf{statlog-heart  }&270&135&135&13&2&1.25&1.09&1.25&1.15&1.25&1.02\\\hline
\textbf{tic-tac-toe  }&958&479&479&9&2&0.53&0.60&0.53&0.61&0.53&0.59\\\hline
\textbf{titanic  }&2201&1101&1100&3&2&2.10&1.76&2.09&1.75&2.10&1.77\\\hline
\textbf{trains  }&10&5&5&29&2&1&1&0.67&0.70&1.50&1.21\\\hline
\textbf{twonorm  }&7400&3700&3700&20&2&1&1&1&1&1&1\\\hline
\textbf{vertebral-column-2clases  }&310&155&155&6&2&2.10&1.56&2.10&1.51&2.10&1.63\\\hline
\end{tabular}
\end{tiny}
\caption{Two-class datasets in \cite{delgado14,wainberg16}. Note that $r_c$ is the class-imbalance ratio and $r_{sc}$ is the scale-class-imbalance ratio defined in~\eqref{imbrsc}. Additionally, $r_c${\bf T} and $r_{sc}${\bf T} (or $r_c${\bf Te} and $r_{sc}${\bf Te}) are $r_c$ and $r_{sc}$ of training dataset (or test dataset).
}\label{2classimb}
\end{table*}

\begin{table*}
\centering
\begin{tiny}\begin{tabular}{l|c|c|c|c|c||c|c||c|c||c|c||c|c}
\hline
&\textbf{inst.}&\textbf{train}&\textbf{test}&\textbf{dim}&\textbf{cls}&\textbf{$r_c$m}&\textbf{$r_c$M}&\textbf{$r_{sc}$m}&\textbf{$r_{sc}$M}&\textbf{$r_{c}$Tm}&\textbf{$r_{c}$TM}&\textbf{$r_{sc}$Tm}&\textbf{$r_{sc}$TM}\\\hline\hline
\textbf{abalone  }&4177&2089&2088&8&3&0.46&0.53&0.53&0.82&0.46&0.53&0.54&0.82\\\hline
\textbf{annealing  }&798&399&399&31&5&0.01&3.20&0.06&2.06&0.01&3.20&0.11&2.12\\\hline
\textbf{arrhythmia  }&452&226&226&262&13&0&1.18&0.05&1.04&0&1.11&0.06&1.02\\\hline
\textbf{audiology-std  }&196&171&25&59&18&0.01&0.32&0.05&0.59&0.01&0.36&0.05&0.59\\\hline
\textbf{balance-scale  }&625&313&312&4&3&0.09&0.85&0.09&0.91&0.09&0.85&0.09&0.94\\\hline
\textbf{breast-tissue  }&106&53&53&9&6&0.15&0.26&0.31&0.71&0.15&0.26&0.27&0.58\\\hline
\textbf{car  }&1728&864&864&6&4&0.04&2.34&0.08&1.77&0.04&2.32&0.08&1.79\\\hline
\textbf{cardiotocography-10clases  }&2126&1063&1063&21&10&0.03&0.37&0.07&0.55&0.03&0.37&0.07&0.54\\\hline
\textbf{cardiotocography-3clases  }&2126&1063&1063&21&3&0.09&3.51&0.34&1.85&0.09&3.50&0.31&1.89\\\hline
\textbf{chess-krvk  }&28056&14028&14028&6&18&0&0.19&0&0.24&0&0.19&0&0.24\\\hline
\textbf{conn-bench-vowel-deterding  }&528&264&264&11&11&0.10&0.10&0.11&0.25&0.10&0.10&0.12&0.26\\\hline
\textbf{contrac  }&1473&737&736&9&3&0.29&0.75&0.37&0.78&0.29&0.74&0.39&0.77\\\hline
\textbf{dermatology  }&366&183&183&34&6&0.06&0.44&0.39&0.90&0.06&0.43&0.40&0.87\\\hline
\textbf{ecoli  }&336&168&168&7&8&0.01&0.74&0.01&0.89&0.01&0.71&0.01&0.84\\\hline
\textbf{energy-y1  }&768&384&384&8&3&0.22&0.88&0.31&0.97&0.22&0.87&0.29&0.99\\\hline
\textbf{energy-y2  }&768&384&384&8&3&0.33&0.99&0.58&1&0.33&0.99&0.60&0.95\\\hline
\textbf{flags  }&194&97&97&28&8&0.02&0.45&0.06&0.68&0.02&0.43&0.05&0.66\\\hline
\textbf{glass  }&214&107&107&9&6&0.04&0.55&0.11&0.60&0.05&0.51&0.10&0.57\\\hline
\textbf{hayes-roth  }&160&132&28&3&3&0.24&0.68&0.40&0.74&0.29&0.63&0.46&0.69\\\hline
\textbf{heart-cleveland  }&303&152&151&13&5&0.04&1.18&0.11&1.07&0.05&1.14&0.14&1.14\\\hline
\textbf{heart-switzerland  }&123&62&61&12&5&0.04&0.64&0.06&0.67&0.05&0.63&0.08&0.66\\\hline
\textbf{heart-va  }&200&100&100&12&5&0.05&0.39&0.10&0.46&0.05&0.37&0.09&0.48\\\hline
\textbf{image-segmentation  }&2310&210&2100&18&7&0.17&0.17&0.27&0.69&0.17&0.17&0.27&0.69\\\hline
\textbf{iris  }&150&75&75&4&3&0.50&0.50&0.58&0.82&0.50&0.50&0.56&0.84\\\hline
\textbf{led-display  }&1000&500&500&7&10&0.09&0.12&0.18&0.30&0.09&0.13&0.19&0.29\\\hline
\textbf{lenses  }&24&12&12&4&3&0.20&1.67&0.35&1.37&0.20&1.40&0.35&1.39\\\hline
\textbf{letter  }&20000&10000&10000&16&26&0.04&0.04&0.06&0.14&0.04&0.04&0.06&0.14\\\hline
\textbf{libras  }&360&180&180&90&15&0.07&0.07&0.12&0.54&0.07&0.07&0.13&0.51\\\hline
\textbf{low-res-spect  }&531&266&265&100&9&0&1.08&0.05&1.01&0&1.06&0.04&0.98\\\hline
\textbf{lung-cancer  }&32&16&16&56&3&0.39&0.68&0.78&0.87&0.45&0.60&0.72&0.87\\\hline
\textbf{lymphography  }&148&74&74&18&4&0.01&1.21&0.08&1.10&0.01&1.18&0.08&1.04\\\hline
\textbf{molec-biol-splice  }&3190&1595&1595&60&3&0.32&1.08&0.51&1.05&0.32&1.08&0.51&1.07\\\hline
\textbf{nursery  }&12960&6480&6480&8&5&0&0.50&0&0.67&0&0.50&0&0.67\\\hline
\textbf{oocytes-merluccius-states-2f  }&1022&511&511&25&3&0.06&2.19&0.37&1.14&0.06&2.17&0.37&1.07\\\hline
\textbf{oocytes-trisopterus-states-5b  }&912&456&456&32&3&0.02&1.36&0.08&1.04&0.02&1.35&0.09&1.08\\\hline
\textbf{optical  }&5620&3823&1797&62&10&0.11&0.11&0.30&0.51&0.11&0.11&0.31&0.53\\\hline
\textbf{page-blocks  }&5473&2737&2736&10&5&0.01&8.77&0.06&4.07&0.01&8.74&0.06&4.01\\\hline
\textbf{pendigits  }&10992&7494&3498&16&10&0.11&0.12&0.23&0.44&0.11&0.12&0.25&0.44\\\hline
\textbf{pittsburg-bridges-MATERIAL  }&106&53&53&7&3&0.12&2.93&0.20&1.67&0.13&2.79&0.29&1.37\\\hline
\textbf{pittsburg-bridges-REL-L  }&103&52&51&7&3&0.17&1.06&0.24&1.04&0.18&1&0.25&1.20\\\hline
\textbf{pittsburg-bridges-SPAN  }&92&46&46&7&3&0.31&1.09&0.45&1.07&0.31&1.09&0.39&1.08\\\hline
\textbf{pittsburg-bridges-TYPE  }&105&53&52&7&6&0.11&0.72&0.14&0.78&0.10&0.66&0.17&0.79\\\hline
\textbf{plant-margin  }&1600&800&800&64&100&0.01&0.01&0.03&0.14&0.01&0.01&0.03&0.14\\\hline
\textbf{plant-shape  }&1600&800&800&64&100&0.01&0.01&0.01&0.25&0.01&0.01&0.01&0.25\\\hline
\textbf{plant-texture  }&1599&800&799&64&100&0.01&0.01&0.03&0.11&0.01&0.01&0.03&0.12\\\hline
\textbf{post-operative  }&90&45&45&8&3&0.02&2.46&0.06&2.31&0.02&2.46&0.06&1.95\\\hline
\textbf{primary-tumor  }&330&165&165&17&15&0.02&0.34&0.04&0.55&0.02&0.32&0.04&0.50\\\hline
\textbf{seeds  }&210&105&105&7&3&0.50&0.50&0.65&0.86&0.50&0.50&0.69&0.87\\\hline
\textbf{semeion  }&1593&797&796&256&10&0.11&0.11&0.52&0.74&0.11&0.11&0.52&0.69\\\hline
\textbf{soybean  }&307&154&153&35&18&0.01&0.15&0.10&0.49&0.01&0.15&0.10&0.49\\\hline
\textbf{statlog-image  }&2310&1155&1155&18&7&0.17&0.17&0.27&0.69&0.17&0.17&0.27&0.68\\\hline
\textbf{statlog-landsat  }&6435&4435&2000&36&6&0.11&0.31&0.26&0.83&0.10&0.32&0.25&0.84\\\hline
\textbf{statlog-shuttle  }&58000&43500&14500&9&7&0&3.67&0&1.67&0&3.63&0&1.67\\\hline
\textbf{statlog-vehicle  }&846&423&423&18&4&0.31&0.35&0.48&0.62&0.31&0.35&0.35&0.72\\\hline
\textbf{steel-plates  }&1941&971&970&27&7&0.03&0.53&0.09&0.76&0.03&0.53&0.10&0.71\\\hline
\textbf{synthetic-control  }&600&300&300&60&6&0.20&0.20&0.28&0.85&0.20&0.20&0.29&0.87\\\hline
\textbf{teaching  }&151&76&75&5&3&0.48&0.53&0.53&0.58&0.49&0.52&0.54&0.54\\\hline
\textbf{thyroid  }&7200&3772&3428&21&3&0.02&12.48&0.08&5.19&0.03&12.28&0.08&4.96\\\hline
\textbf{vertebral-column-3clases  }&310&155&155&6&3&0.24&0.94&0.43&0.98&0.24&0.94&0.43&0.92\\\hline
\textbf{wall-following  }&5456&2728&2728&24&4&0.06&0.68&0.24&0.78&0.06&0.68&0.23&0.80\\\hline
\textbf{waveform  }&5000&2500&2500&21&3&0.49&0.51&0.75&0.85&0.49&0.51&0.74&0.85\\\hline
\textbf{waveform-noise  }&5000&2500&2500&40&3&0.49&0.51&0.76&0.84&0.49&0.51&0.77&0.86\\\hline
\textbf{wine  }&178&89&89&13&3&0.37&0.66&0.78&0.87&0.37&0.65&0.76&0.89\\\hline
\textbf{wine-quality-red  }&1599&800&799&11&6&0.01&0.74&0.02&0.81&0.01&0.74&0.02&0.81\\\hline
\textbf{wine-quality-white  }&4898&2449&2449&11&7&0&0.81&0&0.82&0&0.81&0&0.82\\\hline
\textbf{yeast  }&1484&742&742&8&10&0&0.45&0.03&0.51&0&0.45&0.03&0.51\\\hline
\textbf{zoo  }&101&51&50&16&7&0.04&0.68&0.13&0.93&0.04&0.65&0.13&0.91\\\hline
\end{tabular}
\end{tiny}
\caption{Multi-class datasets in \cite{delgado14,wainberg16}. Here, $0$ means $<0.01$. Note that $r_c${\bf m/M} is the minimum/maximum class-imbalance ratio of datasets and $r_{sc}${\bf m/M} is the minimum/maximum scale-class-imbalance ratio defined in~\eqref{imbrsc}. Additionally, $r_c${\bf Tm/M} and $r_{sc}${\bf Tm/M} are the corresponding imbalance ratios for training datasets. OVA strategy is used for $r_c$ and $r_{sc}$.}
\label{mclassimb}
\end{table*}

 \begin{figure*}[t]
\centering
\includegraphics[width=4in]{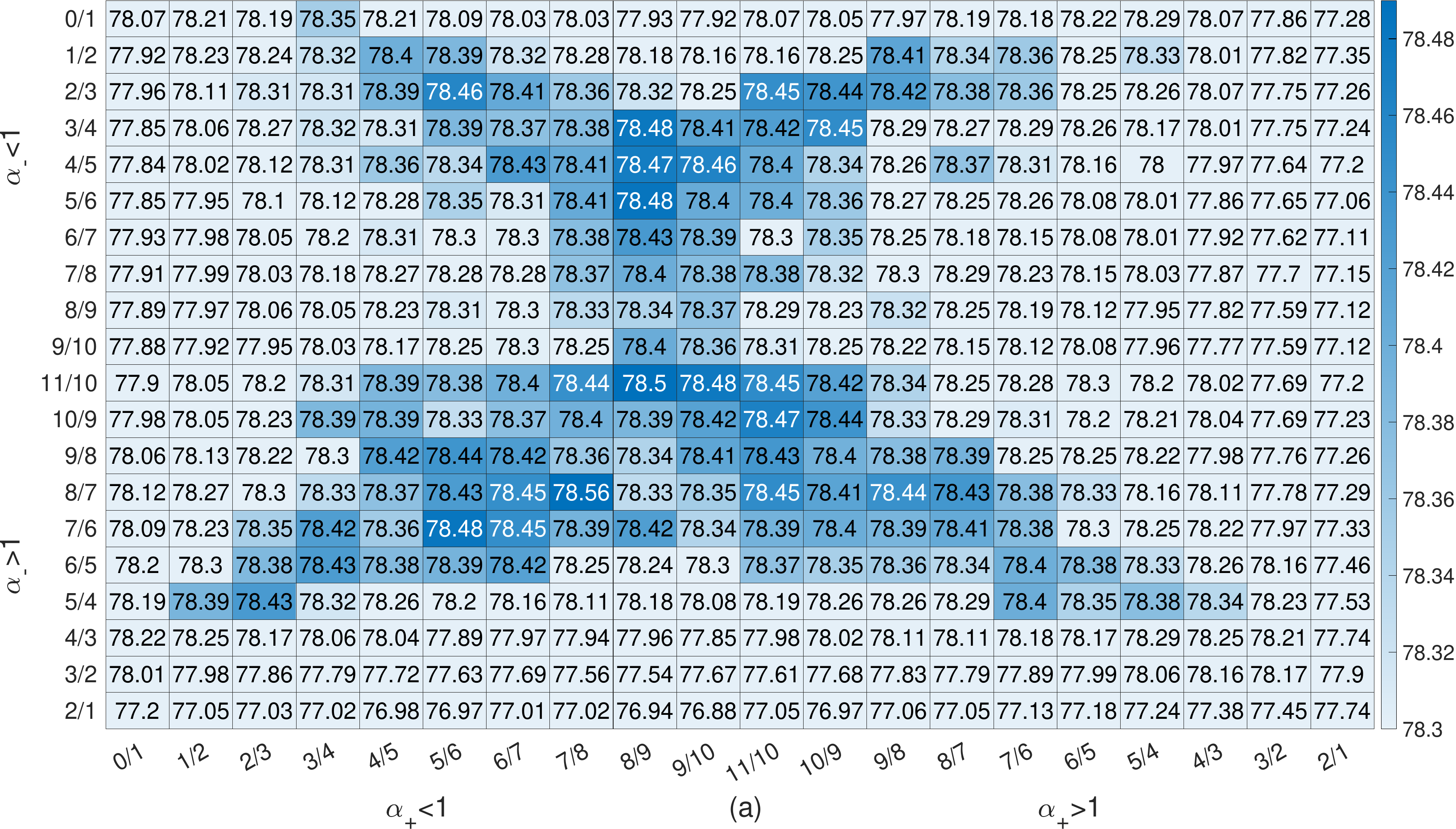}  \hskip 0.2in \vskip 0.2in
\includegraphics[width=4in]{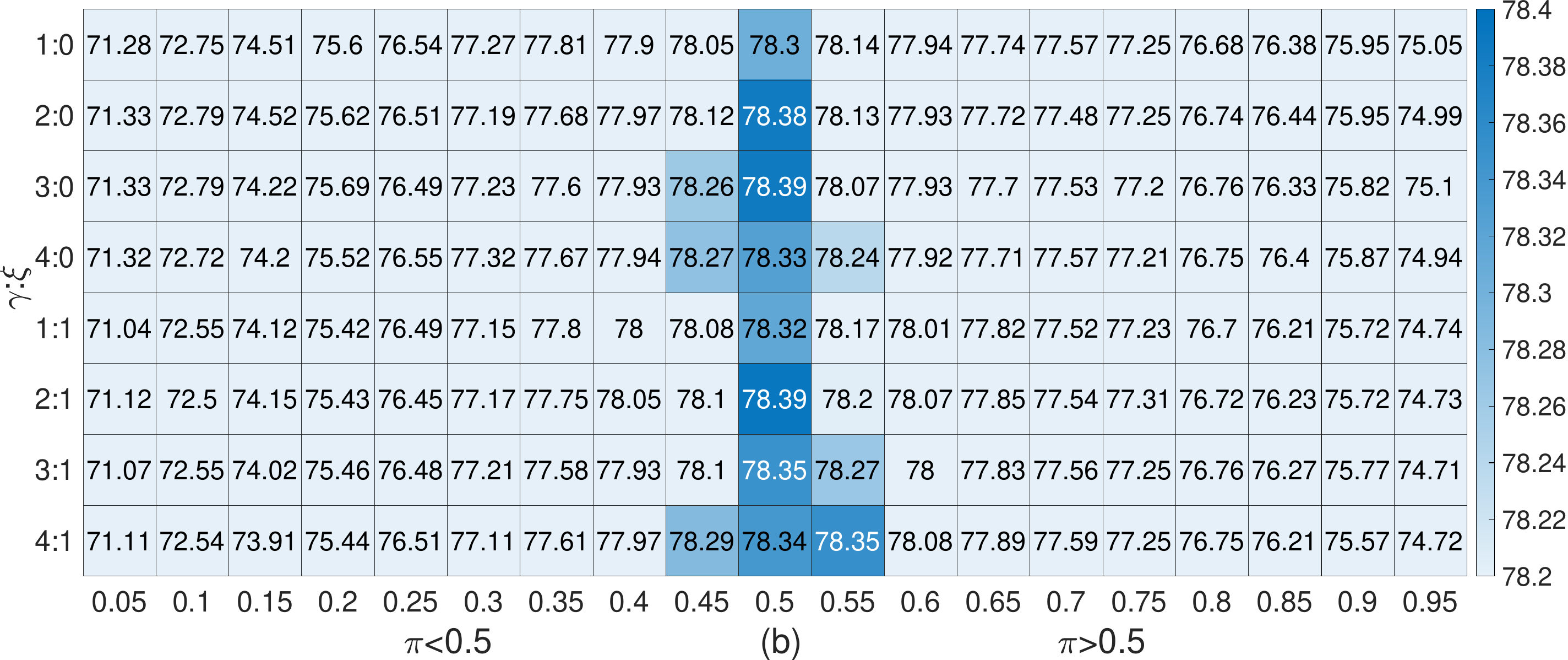} \hskip 0.3in  
\includegraphics[width=2.5in]{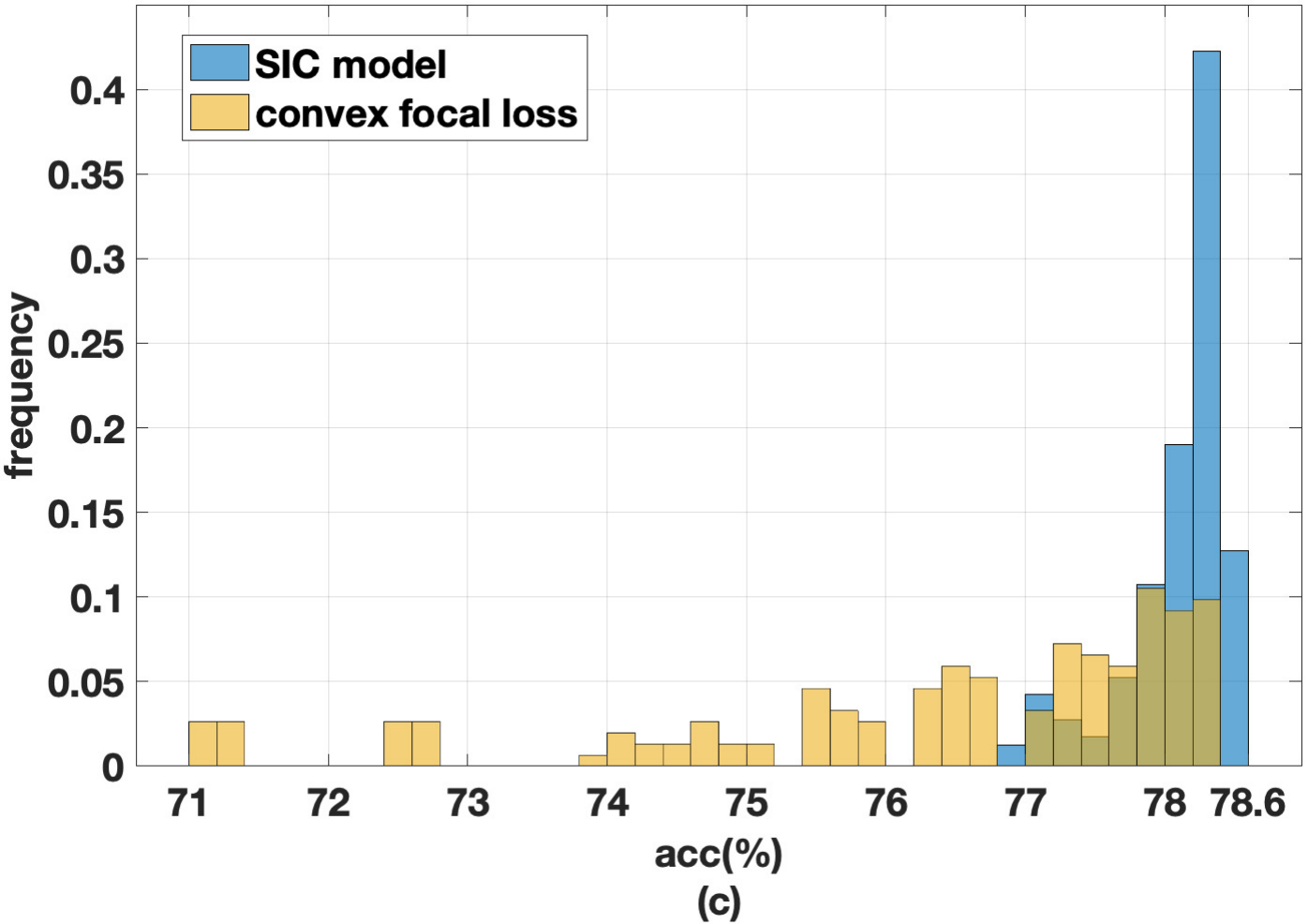}
\caption{Classification accuracy(\%) patterns of all-class are exhibited for (a) $20\times20$ SIC models and (b) $19\times8$ convex focal losses. The corresponding histograms are shown in (c). In the case of the SIC model, the best test classification accuracy $78.56\%$ is achieved when $(\alpha_+,\alpha_-) = (7/8,8/7)$. On the other hand, the convex focal loss model achieves the best test accuracy $78.39\%$ at $(\pi,\gamma,\xi)=(0.5,3,0), (0.5,2,1)$. Moreover, a histogram comparison between the two models shows that more than half of all SIC models obtain at least $78.20\%$ accuracy, while only $10\%$ of $\pi$-weighted convex focal losses obtain that accuracy.}
\label{fig:ALLclass}
\end{figure*}

\section*{Acknowledgments}
H. Woo is supported by Logitron X.


\begin{thebibliography}{30}
\bibitem{amari16}{\sc S. Amari,} {\em Information geometry and its applications}, Springer, 2016.

\bibitem{bach06}{\sc F. R. Bach, D. Heckerman, and E. Horvitz,} ``Considering cost asymmetry in learning classifiers,", {\em Journal of Machine Learning Research,} vol. 7, pp. 1713-1741, 2006.

\bibitem{brebisson16}{\sc A. de Br\'ebisson and P. Vincent,}  ``An exploration of softmax alternatives belonging to the spherical loss family", {\em arXiv:1511.05042v3}, 2016.

\bibitem{byrd19}{\sc J. Byrd and Z. C. Lipton,}  ``What is the effect of importance weighting in deep learning", {\em Proceedings of the 36th International Conference on Machine Learning}, 2019.

\bibitem{cao19}{\sc K. Cao, C. Wei, A. Gaidon, N. Arechiga, and T. Ma,} ``Learning imbalanced datasets with label-distribution-aware margin loss", {\em 33rd Conference on Neural Information Processing Systems (NeurIPS 2019)}, 2019.

\bibitem{chang11}{\sc C.-C. Chang and C.-J. Lin,} ``LIBSVM : a library for support vector machines", {\em ACM Transactions on Intelligent Systems and Technology,} vol. 2, pp. 27:1-27:27, 2011.

\bibitem{cui19}{\sc Y. Cui, M. Jia, T.-Y. Lin, Y. Song, S. Belongie,} ``Class-balanced loss based on effective number of samples", {\em arXiv:1901.5555v1}, 28 September 2018.

\bibitem{delgado14}{\sc M. F.-Delgado, E. Cernadas, S. Barro, and D. Amorim,} ``Do we need hundreds of classifiers to solve real world classification problems?", {\em Journal of Machine Learning Research}, vol. 15, pp. 3133-3181, 2014.

\bibitem{ding10}{\sc N. Ding and S.V.N. Vishwanathan,} ''t-logistic regression", {\em Advances in Neural Information Processing Systems 23}, 2010.

\bibitem{dubey22}{\sc S. R. Dubey, S. K. Singh, and B. B. Chaudhuri,} ``Activation functions in deep learning: A comprehensive survey and benchmark", {\em arXiv:2109.14545v3}, 2022.

\bibitem{elsayed18}{\sc G. F. Elsayed, D. Krishnan, H. Mobahi, K. Regan, and S. Bengio,} ``Large margin deep networks for classification", {\em 32nd Conference on Neural Information Processing Systems (NeurIPS)}, 2018.

\bibitem{fan08}{\sc R.-E. Fan, K.-W. Chang, C.-J. Hsieh, X.-R. Wang, and C.-J. Lin,} "LIBLINEAR: A library for large linear classification", {\em Journal of Machine Learning Research}, vol. 9, pp. 1871-1874, 2008.

\bibitem{fang21}{\sc C. Fang, H. He, Q. Long, and W. J. Su,} ``Exploring deep neural networks via layer-peeled model: minority collapse in imbalanced training", {\em Proc. Natl. Acad. Sci. U.S.A.}, vol. 118, pp. 1-12, 2021.

\bibitem{fernandez18}{\sc A. Fern\'andez, S. Garc\'ia, M. Galar, R. C. Prati, B. Krawczyk, and F. Herrera,} {\em Learning from imbalanced datasets,} Springer-Verlag, 2018.

\bibitem{galli22}{\sc L. Galli and C.-J. Lin,} ``A study on truncated Newton methods for linear classification", {\em IEEE Transactions on Neural Networks and Learning Systems}, vol. 33, pp. 2828-2841, 2022.

\bibitem{garcia15}{\sc S. Garcia, J. Luengo, and F. Herrera,} {\em Data preprocessing in data mining}. Springer-Verlag, 2015.

\bibitem{graham94}{R. Graham, D. Knuth, and O. Patashnik,} {\em Concrete Mathematics: A foundation for computer science}, Second Edition, Addison-Wesley, 1994.

\bibitem{guo17}{C. Guo, G. Pleiss, Y. Sun, and K. Q. Weinberger,} ``On calibration of modern neural networks", {\em Proceedings of the 34th International Conference on Machine Learning}, 2017.

\bibitem{hager05}{\sc W. W. Hager and H. Zhang,} ``A new conjugate gradient method with guaranteed descent and an efficient line search", SIAM Journal on Optimization, vol. 16, pp. 170-192, 2005.

\bibitem{hager06}{\sc W. W. Hager and H. Zhang,} ``Algorithm 851: CG\_DESCENT, a conjugate gradient method with guaranteed descent", ACM Transactions on Mathematical Software, vol. 32, pp. 113-137, 2006.

\bibitem{he09}{\sc H. He and E. A. Garcia,} ``Learning from imbalanced data", {\em IEEE Transactions on Knowledge and Data Engineering,} vol. 21, pp. 1263-1284, 2009. 

\bibitem{hiriart-urruty96}{\sc J.-B. Hiriart-Urruty and C. Lemarechal,} {\em Convex Analysis and Minimization Algorithms I}. Springer-Verlag, 1996.


\bibitem{ioffe15}{\sc S. Ioffe and C. Szegedy,} ``Batch normalization: Accelerating deep network training by reducing internal covariate shift", {\em arXiv:1502.03167v3}, 2015.  

\bibitem{janocha17}{\sc K. Janocha and W. M. Czarnecki,} ``On loss functions for deep neural networks in classification", {\em arXiv:1702.05659v1}, 2017.  

 \bibitem{ji20}{\sc Z. Ji, M. Dudik, R. E. Schapire, and M. Telgarsky,} ``Gradient descent follows the regularization path for general losses", {\em 33rd Annual Conference on Learning Theory}, 2020.

\bibitem{jiang18}{\sc Y. Jiang, D. Krishnan, H. Mobahi, S. Bengio,} ``Predicting the generalization gap in deep networks with margin distributions", {\em arXiv:1810.00113v1}, 28 September 2018.

\bibitem{johnson19}{\sc J. M. Johnson and T. M. Khoshgoftaar,} ``Survey on deep learning with class imbalance", {\em Journal of Big Data,} vol. 6, pp. 1-54., 2019.

\bibitem{jorgensen97}{\sc B. Jorgensen,} {\em The Theory of Dispersion Models}, Chapman \& Hall, 1997.


\bibitem{kurgan01}{\sc L. A. Kurgan, K. J. Cios, R. Tadeusiewicz, M. Ogiela, L. Goodenday,}  ``Knowledge discovery approach to automated cardiac SPECT diagnosis,", {\em Artificial Intelligence in Medicine}, vol. 23, pp. 149-169, 2001. 

\bibitem{liao18}{\sc J. Liao, O. Kosut, L. Sankar, F. du Pin Calmon,} ``Tunable measures for information leakage and applications to privacy-utility tradeoffs", {\em IEEE Transactions on Information Theory}, vol. 65, pp. 8043-8066, 2019.

\bibitem{lin20}{\sc T.-Y. Lin, P. Goyal, R. Girshick, K. He, and P. Doll\'ar,} ``Focal loss for dense object detection", {\em IEEE Transactions on Pattern Analysis and Machine Intelligence,} vol. 42, pp. 318-327, 2020.

\bibitem{lin02}{\sc Y. Lin,} ``Support Vector Machines and the Bayes rule in classification", {\em Data Mining and Knowledge Discovery}, vol. 6, pp. 259-275, 2002.

\bibitem{lin07}{\sc H.-T. Lin, C.-J. Lin, R. C. Weng,} ``A note on Platt's probabilistic outputs for support vector machines", {\em Mach. Learn.}, 68 (2007), pp. 267-276. 

\bibitem{lyu21}{\sc K. Lyu, Z. Li, R. Wang, and S. Arora,} ``Gradient descent on two-layer nets: margin maximization and simplicity bias", {\em Advances in Neural Information Processing Systems 35}, 2021.

\bibitem{martins16}{\sc A. F. T. Martins and R. F. Astudillo,} ``From softmax to sparsemax: A sparse model of attention and multi-label classification", {\em Proceedings of the 33th International Conference on Machine Learning}, 2016.

\bibitem{masnadi08}{\sc H. Masnadi-Shirazi and N. Vasconcelos,}  ``On the design of loss functions for classification: theory, robustness to outliers, and savageboost", {\em Advances in Neural Information Processing Systems 21}, 2008.

\bibitem{mccullagh89}{\sc P. McCullagh and J. A. Nelder,} {\em Generalized Linear Models}, Second Edition, Chapman \& Hall/CRC, 1989.

\bibitem{more94}{\sc J. J. Mor\'e and D. J. Thuente,} ``Line search algorithm with guaranteed sufficient decrease", {\em ACM Transactions on Mathematical Software} vol. 20, pp. 286-307, 1994.

\bibitem{murphy12}{\sc K. P. Murphy,} {\em Machine Learning}, MIT Press, 2012.

\bibitem{mutschler20}{\sc M. Mutschler and A. Zell,} ``Parabolic approximation line search for DNNs", {\em Advances in Neural Information Processing Systems 33},  2020.

\bibitem{nocedal06}{\sc J. Nocedal and S. J. Wright,} {\em Numerical Optimization}, Second Edition, Springer-Verlag, 2006.

\bibitem{oksuz21}{\sc K. Oksuz, B. C. Cam, S. Kalkan, and E. Akbas,} ``Imbalance problems in object detection: a review", {\em IEEE Transactions on Pattern Analysis and Machine Intelligence,} vol. 43, pp. 3388-3415, 2021. 

\bibitem{ollivier15}{\sc Y. Ollivier,} ``Riemannian metrics for neural networks I: feedforward networks", {\em Information and Inference: A Journal of the IMA}, vol. 4. pp. 108-153, 2015.

\bibitem{platt99}{\sc J. C. Platt,} ``Probabilistic outputs for support vector machines and comparisons to regularized likelihood methods", in Advances in Large Margin Classifiers, A.J. Smola, P. Bartlett, B. Sch\"olkopf, D, Schuurmans eds, MIT Press., 1999.

\bibitem{reid11}{\sc M. D. Reid and R. C. Williamson,} ``Information, divergence and risk for binary experiments", {\em Journal of Machine Learning Research}, vol. 12, pp. 731-817, 2011.

\bibitem{rockafellar70}{\sc R. T. Rockafellar,} {\em Convex Analysis},
Princeton University Press, Princeton, 1970.

\bibitem{schmidt05}{\sc M. Schmidt,} {\em minFunc: unconstrained differentiable multivariate optimization in Matlab}, \url{http://www.cs.ubc.ca/~schmidtm/Software/minFunc.html}, 2005.

\bibitem{sigmoidwiki} {\em Wikipedia - sigmoid function} \url{https://en.wikipedia.org/wiki/Sigmoid_function}.

\bibitem{soudry18}{\sc D. Soudry, E. Hoffer, M. S. Nacson, S. Gunasekar, N. Srebro,} ``The implicit bias of gradient descent on separable data", {\em Journal of Machine Learning Research},  vol. 19, pp. 1-57, 2018.

\bibitem{sypherd19}{\sc T. Sypherd, M. Diaz, J. K. Cava, G. Dasarathy, P. Kairouz, and L. Sankar,} ``A tunable loss function for robust classification: calibration, landscape, and generalization", {\em IEEE Transactions on Information Theory}, vol. 68, pp. 6021-6051, 2022.

\bibitem{tsanas12}{\sc A. Tsanas and A. Xifara,} ``Accurate quantitative estimation of energy performance of residential building using statistical machine learning tools,", {\em Energy Buildings}, vol. 49, pp. 560-567, 2012. 

\bibitem{ucidata}{\sc M. Kelly, R. Longjohn, and K. Nottingham,} {\em The UCI Machine Learning Repository}, \url{https://archive.ics.uci.edu}

\bibitem{vardi23}{\sc G. Vardi,}  ``On the implicit bias in deep-learning algorithms", {\em Communications of the ACM}, vol. 66, pp. 86-93, 2023.

\bibitem{vaswani19x}{\sc S. Vaswani, F. Bach, M. Schmidt,} "Fast and faster convergence of SGD for over-parameterized models (and an accelerated Perceptron)", {\em Proceedings of the 22nd International Conference on Artificial Intelligence and Statistics}, 2019.

\bibitem{wainberg16}{\sc M. Wainberg, B. Alipanahi, and B. J. Frey,}  ``Are random forests truly the best classifiers?", {\em Journal of Machine Learning Research}, vol. 17, pp. 1-5, 2016.

\bibitem{wedderburn74}{\sc R. Wedderburn,} ``Quasi-likelihood functions, generalized linear models, and the Gauss-Newton method'', {\em Biometrika}, vol. 61, pp. 439-447, 1974.

\bibitem{woo17}{\sc H. Woo,} ``A characterization of the domain of Beta-divergence and its connection to Bregman variational model", {\em Entropy}, vol. 19, 482, 2017.

\bibitem{woo19a}{\sc H. Woo,} ``Logitron: Perceptron-augmented classification model based on an extended logistic loss function", {\em arXiv:1904.02958v1}, 2019.

\bibitem{woo19b}{\sc H. Woo,} ``The Bregman-Tweedie classification model", {\em arXiv: 1907.06923v1}, 2019.

\bibitem{woo19c}{\sc H. Woo,} ``Bregman-divergence-guided Legendre exponential dispersion model with finite cumulants ($k$-LED)", {\em arXiv: 1910.03025v1}, 2019.

\end{thebibliography}
\end{document}